\documentclass[11pt]{amsart}

\usepackage[letterpaper,top=3.5cm,bottom=3.5cm,left=3cm,right=3cm,marginparwidth=1.75cm]{geometry}
\usepackage{amsmath,amsfonts,amsthm,amssymb,amscd}
\usepackage{lipsum}
\usepackage{amsfonts}
\usepackage{graphicx}
\usepackage{epstopdf}
\usepackage{algorithmic}
\usepackage{appendix}
\usepackage{subfig}
\usepackage{graphicx}
\usepackage{color}
\usepackage{xcolor}
\usepackage[hidelinks]{hyperref}
\usepackage{cleveref}
\usepackage{bm}
\usepackage{soul}
\usepackage[normalem]{ulem}
\usepackage[foot]{amsaddr}

\usepackage{etoolbox}
\appto\appendix{\addtocontents{toc}{\protect\setcounter{tocdepth}{1}}}
\appto\listoffigures{\addtocontents{lof}{\protect\setcounter{tocdepth}{1}}}
\appto\listoftables{\addtocontents{lot}{\protect\setcounter{tocdepth}{1}}}

\newcommand{\N}{\mathcal{N}}

\newcommand{\G}{\mathcal{G}}

\newcommand{\E}{\mathbb{E}}

\newcommand{\V}{\Phi_R}

\newcommand{\PP}{\mathcal{P}}
\newcommand{\PPG}{\mathcal{P}^G}
\newcommand{\EE}{\mathcal{E}}
\newcommand{\Cov}{\mathrm{Cov}}
\newcommand{\M}{M}

\newcommand{\I}{I}
\newcommand{\dd}{\mathrm{d} }
\newcommand{\R}{\mathbb{R}}
\newcommand{\T}{\mathcal{T}}

\newcommand{\rhoa}{a}

\newcommand{\fM}{\mathfrak{M}}
\newcommand{\la}{\lambda_{\star,\max}}

\newcommand{\Prec}{P}

\newcommand{\bigO}{\mathcal{O}}
\newcommand{\Hess}{\nabla_{\theta}\nabla_{\theta}}
\newcommand{\ranglen}{\rangle_{\R^{N_a}}}

\DeclareMathOperator*{\argmin}{arg\,min}

\newcommand{\bigTheta}{\Theta}

\newenvironment{newremark}[1]{%
    \begin{remark}#1}{%
    \Endofdef\end{remark}%
}
\newcommand{\xqed}[1]{%
    \leavevmode\unskip\penalty9999 \hbox{}\nobreak\hfill
    \quad\hbox{\ensuremath{#1}}}
\newcommand{\Endofdef}{\xqed{\lozenge}}

\newtheorem{theorem}{Theorem}[section]
\newtheorem{lemma}[theorem]{Lemma}

\newtheorem{proposition}[theorem]{Proposition}
\newtheorem{definition}[theorem]{Definition}
\newtheorem{example}[theorem]{Example}
\newtheorem{condition}[theorem]{Condition}
\theoremstyle{remark}
\newtheorem{remark}[theorem]{Remark}
\numberwithin{equation}{section}

\definecolor{darkred}{rgb}{.6,0,0}
\definecolor{darkblue}{rgb}{0,0,.7}
\definecolor{darkgreen}{rgb}{0,.7,0}
\definecolor{darkbrown}{rgb}{0.8,0.4,0.4}

\begin{document}
\title[Sampling via Gradient Flows]{Sampling via Gradient Flows in the\\ Space of Probability Measures}

\author{Yifan~Chen\textsuperscript{2,1}}
\address{\textsuperscript{1}California Institute of Technology, Pasadena, CA}
\address{\textsuperscript{2}Courant Institute, New York University, NY}
\email{yifan.chen@nyu.edu}

\author{Daniel~Zhengyu~Huang\textsuperscript{3,1}}
\email{huangdz@bicmr.pku.edu.cn}
\address{\textsuperscript{3}Beijing International Center for Mathematical Research, Peking University, Beijing, China}
\author{Jiaoyang Huang\textsuperscript{4}}
\address{\textsuperscript{4}University of Pennsylvania, Philadelphia, PA}
\email{huangjy@wharton.upenn.edu}
\vspace{0.1in}
\author{Sebastian Reich\textsuperscript{5}}
\address{\textsuperscript{5}Universit\"{a}t Potsdam, Potsdam, Germany}
\email{sebastian.reich@uni-potsdam.de}
\author{Andrew M. Stuart\textsuperscript{1}}
\email{astuart@caltech.edu}
\maketitle

\vspace{-2.5em}
\begin{abstract}
Sampling a target probability distribution with an unknown normalization constant is a fundamental challenge in computational science and engineering. Recent work shows that algorithms derived
by considering gradient flows in the space of probability measures open up new avenues for
algorithm developments. This paper makes three contributions to this approach to sampling,
by scrutinizing the design components of such gradient flows.
Any instantiation of a gradient flow for sampling needs an energy functional and a metric to determine the flow, as well as numerical approximations of the flow to derive algorithms. Our
first contribution is to show that the Kullback-Leibler (KL) divergence, as an energy functional, has the \textit{unique} property (among all $f$-divergences) that gradient flows resulting from it do not depend on the normalization constant of the target distribution; {this justifies the widespread use of the KL divergence in sampling}. Our second contribution is to study the choice of metric from the perspective of invariance. The Fisher-Rao metric is known as the \textit{unique} choice (up to scaling) that is diffeomorphism invariant. 
As a computationally tractable alternative, we introduce a relaxed, affine invariance property for the metrics and gradient flows. In particular, we construct various affine invariant Wasserstein and Stein gradient flows. 
Affine invariant gradient flows are shown to behave more favorably than their non-affine-invariant counterparts
when sampling highly anisotropic distributions, in theory and by using particle methods.
Our third contribution is to study, and develop efficient algorithms based on Gaussian approximations of the gradient flows; this leads to an alternative to particle methods. 
We establish connections between various Gaussian approximate gradient flows, discuss their relation to gradient methods arising from parametric variational inference, and study their convergence properties. Our theory and numerical experiments demonstrate the strengths and potential limitations of the Gaussian approximate Fisher-Rao gradient flow, which is affine invariant, by considering a wide range of target distributions.

\end{abstract}
\tableofcontents

\section{Introduction}
\subsection{The Sampling Problem}
In this paper, we are concerned with the problem of sampling a 
probability distribution that is known up to normalization. This problem
is fundamental in many applications arising in computational science and engineering and is
widely studied in applied mathematics, machine learning 
and statistics communities. 
A particular application is Bayesian inference for large-scale inverse problems; such
problems are ubiquitous, for example in climate science~\cite{isaac2015scalable,schneider2017earth,huang2022iterated,lopez2022training}, engineering~\cite{yuen2010bayesian,cui2016dimension,cao2022bayesian}, and machine learning~\cite{rasmussen2003gaussian,murphy2012machine, chen2021solving, chen2021consistency}. These applications have fueled the need for efficient and scalable sampling algorithms.

Mathematically, the objective is to sample a target probability
distribution with density $\rho_{\rm post}(\cdot)$, for the parameter $\theta \in \R^{N_{\theta}}$, given by
\begin{align}
\label{eq:posterior}
    \rho_{\rm post}(\theta) \propto \exp(-\V(\theta)),
\end{align}
where {$\V: \R^{N_{\theta}} \to \R$} is a known function.
We use the notation $\rho_{\rm post}$ because of the potential application
to Bayesian inference, where $\V$ represents the regularized {negative
log} likelihood function; however, we do not explicitly use the Bayesian
structure in this paper, and our discussion applies to general target distributions {known only up to normalization.}

\subsection{Gradient Flow Methodology}
Numerous approaches to the sampling problem
have been proposed in the literature. Most are based on construction of a dynamical system for densities that converges to the target distribution, or its approximation, after a specified finite time or at infinite time.
The most common examples that are widely used in Bayesian inference are sequential Monte Carlo (SMC, 
specified finite time) \cite{doucet2009tutorial} and Markov chain Monte Carlo (MCMC, infinite time)
\cite{brooks2011handbook}.

Among them, the Langevin diffusion \cite{pavliotis2014stochastic} and its discretization constitute an important class of MCMC algorithms for sampling. It has been shown in the seminal work~\cite{jordan1998variational} that
 the Fokker-Planck equation describing the evolution of densities of the Langevin diffusion is the Wasserstein gradient flow of the Kullback–Leibler (KL) divergence. That is, the dynamical system of densities corresponding to the Langevin diffusion has a gradient flow structure. Such structure has also been identified in other popular sampling algorithms, for example, the Stein variational gradient descent \cite{liu2016stein}, with the Stein variational gradient flow~\cite{liu2017stein} as the continuous limit.
 
 Indeed, extending beyond the {above-mentioned} examples, gradient flows have profoundly influenced our understanding and development of sampling algorithms.
 Various gradient flows have been adopted by researchers to address the sampling problems. For instance, the Wasserstein-Fisher-Rao gradient flow was used to sample multi-modal distributions in \cite{lu2019accelerating,lu2022birth}. {In~\cite{garbuno2020interacting,garbuno2020affine}, the Kalman-Wasserstein metric 
 (introduced, but not named, in \cite{reich2015probabilistic}) was used to induce the gradient flow
 leading to advantageous algorithms for sampling anisotropic distributions.} Interpolation between the Wasserstein and Stein metrics was studied in \cite{he2022regularized}. Accelerated gradient flows in the probability space have been studied in \cite{wang2022accelerated}. A recent overview of the use
of gradient flows in optimization and sampling 
can be found in \cite{TrillosNoticeAMS}.  In addition to the continuous-time picture, the optimization perspective on gradient flows also leads to new developments in discrete-time algorithms for sampling \cite{wibisono2018sampling}.

Given the many choices of gradient flows in the literature, the aim of this paper is to study their design ingredients and identify {key properties} that make the gradient flows favorable in the sampling context. Our main focus is the continuous-time formulation of the flow.

\subsection{Design Ingredients of Gradient Flows}
Given an energy functional $\EE$ on the probability space, {a class of gradient flows} of $
\EE$ can be formally written as
\begin{equation}
\label{eq:GF}
    \frac{\partial \rho_t}{\partial t}  =-M(\rho_t)^{-1}\frac{\delta \mathcal{E}}{\delta \rho}\Bigr|_{\rho = \rho_t}. 
\end{equation}
Here, $\frac{\delta \mathcal{E}}{\delta \rho}$ represents the first variation of $\EE$, which is
then evaluated at $\rho=\rho_t$. Positive definite operator $M(\rho)$ leads to nonlinear
preconditioning by its inverse; such operators $M(\rho)$ arise naturally from a Riemannian metric defined in the probability density space.

A key property of the gradient flow is 
\begin{equation}
\label{eq:GF-E}
   \frac{\dd}{\dd t}\mathcal{E}(\rho_t)=
\Bigl\langle \frac{\delta \mathcal{E}}{\delta \mathcal{\rho}}\Bigr|_{\rho = \rho_t}, \frac{\partial \rho_t}{\partial t}\Bigr
\rangle= 
-\Bigl\langle  M(\rho_t) \frac{\partial \rho_t}{\partial t}, \frac{\partial \rho_t}{\partial t} \Bigr
\rangle \le 0 .
\end{equation}
{This demonstrates that the gradient flow~\cref{eq:GF} will keep decreasing the energy functional and
indeed can be used as the basis of proofs to establish converge to $\rho_{\rm post}$ when it is the unique stationary point of $\EE$; for more details of the above notations we refer to \Cref{sec-Gradient-Flows}.} 

{By choosing different $\EE(\cdot)$ and $M(\cdot)$, one obtains various gradient flows with varying convergence properties and levels of numerical implementation difficulties. In the sampling context, it is natural to ask which choices lead to the most favorable algorithms, for example, in terms of numerical implementation and convergence rates on specific problem classes. Such questions are the focus of this paper.}

\subsection{Our Contributions and Paper Organization}
The primary contributions of this work are as follows.
\begin{itemize}

    \item {We show in \Cref{thm: KL unique f divergence} that the KL divergence stands out as a unique energy functional $\EE(\cdot)$ among all $f$-divergences with respect to the target. Specifically, it emerges  as the sole choice (up to scaling)} that yields gradient flows independent of the normalization constant of the target distribution. Since handling the unknown normalization constant poses a significant challenge in the sampling problem, our finding establishes KL divergence as the desired energy functional for designing gradient flows, ensuring that concerns about normalization constants are eliminated during numerical implementation. {We then focus on this choice of 
    energy functional $\EE(\cdot)$ in the remainder of the paper.}
    
    \item {We highlight invariance properties stemming from the choice of metric (and hence
    {preconditioner in \eqref{eq:GF}}) 
    and elucidate implications for the convergence of the gradient flow. We prove in
    \Cref{thm:FR-convergence} that
    the Fisher-Rao metric, which is the unique (up to scaling) metric exhibiting diffeomorphism invariance,  achieves a uniform exponential rate of convergence to the target distribution, under quite general conditions. Moreover, we introduce a relaxed (weaker), affine invariance property for gradient flows. As a concrete manifestation, we construct various affine invariant Wasserstein and Stein gradient flows. These affine invariant gradient flows are more convenient to approximate numerically via particle methods than the diffeomorphism invariant Fisher-Rao gradient flow; pre-existing theory, {and numerical experiments in this paper} demonstrate that these affine invariant methods behave favorably compared to their non-affine-invariant versions when sampling highly anisotropic distributions.}
    \item {We study efficient implementable sampling algorithms
    found by restricting the gradient flows to Gaussian measures, as alternatives to particle methods. In \Cref{prop:1} we} demonstrate that, under mild assumptions, the Gaussian approximation achieved through metric based projection is equivalent to the approximation derived through moment closure; the latter is more convenient for calculations. Furthermore, we establish connections between various Gaussian approximate gradient flows, {discuss their relation to gradient methods to solve the
    Gaussian variational inference problem}, and study their convergence properties. Our theory and numerical experiments demonstrate the strengths and limitations of the Fisher-Rao gradient flow under Gaussian approximations, for Gaussian, logconcave, and general target distributions. 
\end{itemize}

{The paper is structured as follows. Section \ref{sec-Gradient-Flows} introduces gradient flows
in the space of probability densities.
In \Cref{sec:Energy}, {we discuss the selection of energy functionals,
defined over the probability density space and parameterized by the target
distribution,}
showing the primacy of the KL divergence within all $f$-divergences.
In \Cref{sec:GF and AI}, we discuss the choice of metrics in the density space and examine the induced gradient flow, with a particular focus on invariance properties such as diffeomorphism invariance and affine invariance.
In \Cref{sec:GGF}, we study Gaussian approximations to efficiently approximate the gradient flows and showcase the advantage and limitation of affine invariance in this context. 
In \Cref{sec:darcy}, we demonstrate the effectiveness of the aforementioned approaches on a PDE-constrained Bayesian inverse problem, complementing the simple illustrative numerical examples presented in \Cref{sec:GF and AI} and \Cref{sec:GGF}.
Concluding remarks are presented in \Cref{sec:conclusion}. The appendix contains details of all the proofs.}

\subsection{Literature Survey} We survey the relevant literature in gradient flows, affine invariance, and Gaussian approximations below.
\subsubsection{Gradient Flows}
There {is a vast literature concerning 
the use of gradient flows in probability density space,}
employing a variety of different metric tensors, to minimize an energy defined as the Kullback–Leibler (KL) divergence between the current density and the target distribution. 
The most relevant to this paper are the Wasserstein, Fisher-Rao, and Stein gradient flows.

The Wasserstein gradient flow was identified in the seminal work~\cite{jordan1998variational}. The
authors showed that the Fokker-Planck equation is the Wasserstein gradient flow of the KL divergence
of the current density estimate from the target. Since then, Wasserstein gradient flow has played a significant role in optimal transport~\cite{santambrogio2017euclidean}, sampling~\cite{chen2018natural,lambert2022variational}, machine learning~\cite{chizat2018global,salimans2018improving}, partial differential equations~\cite{otto2001geometry,carrillo2018analytical} and many other areas. 
The Fisher-Rao metric was introduced by {C.R. Rao~\cite{rao1945information}} via the Fisher information matrix. The original definition is in parametric density spaces, and the corresponding Fisher-Rao gradient flow in the parameter space leads to natural gradient descent~\cite{amari1998natural}. 
The Fisher-Rao metric in infinite dimensional probability spaces was discussed in~\cite{FRmetricInfDim1991, srivastava2007riemannian}. The concept underpins information geometry \cite{amari2016information,ay2017information}. The gradient flow of the KL divergence under the Fisher-Rao metric is induced by a mean-field model of birth-death type. The birth-death process has been used in sequential Monte Carlo samplers to reduce the variance of particle weights~\cite{del2006sequential} and to accelerate Langevin sampling~\cite{lu2019accelerating,lu2022birth}. 
The discovery of the Stein metric~\cite{liu2017stein} follows the introduction of the Stein variational gradient descent algorithm~\cite{liu2016stein}. The study of the Stein gradient flow~\cite{liu2017stein,lu2019scaling,duncan2019geometry} sheds light on the analysis and improvements of the algorithm~\cite{detommaso2018stein, wang2020information, wang2022accelerated, liu2023towards}.
\subsubsection{Affine Invariance}
\label{sec: survey affine invariance}
The idea of affine invariance was introduced
for MCMC methods in \cite{goodman2010ensemble,foreman2013emcee},  motivated by the empirical success
of the Nelder-Mead simplex algorithm \cite{nelder1965simplex} in optimization. 
Sampling methods with the affine invariance property can be effective for highly anisotropic distributions; this is because they behave identically in all coordinate systems related through an affine transformation; 
in particular, they can be understood by studying the best possible coordinate system, which
reduces anisotropy to the maximum extent possible within the class of affine transformations. The numerical studies presented in~\cite{goodman2010ensemble} demonstrate that affine-invariant MCMC methods offer significant performance improvements over standard MCMC methods. 
This idea has been further developed to enhance sampling algorithms in more general contexts. Preconditioning strategies for the Langevin dynamics to achieve affine-invariance were discussed in \cite{leimkuhler2018ensemble}. 
{In~\cite{reich2015probabilistic}, the Kalman-Wasserstein 
metric was introduced, and gradient flows with respect to this metric were advocated and shown to achieve affine invariance ~\cite{garbuno2020interacting, garbuno2020affine}. 
Moreover, the authors in \cite{garbuno2020interacting,garbuno2020affine,pidstrigach2021affine} used the empirical covariance as preconditioners in ensemble Kalman based particle approximations of the Kalman-Wasserstein
gradient flow,} leading to a family of derivative-free affine invariant sampling approaches.
Similarly, the work \cite{liu2022second} employed the empirical covariance  to precondition second order Langevin dynamics.  
Affine invariant samplers can also be combined with the pCN (preconditioned Crank–Nicolson) MCMC method~\cite{cotter2013mcmc}, to boost the performance of MCMC in function space~\cite{coullon2021ensemble,dunlop2022gradient}.
Another family of affine-invariant sampling algorithms is based on Newton or Gauss-Newton methods, since 
the use of the Hessian matrix as the preconditioner naturally induces the affine invariance property. 
Such methods include the stochastic Newton MCMC~\cite{martin2012stochastic}, the Newton flow with different metrics~\cite{detommaso2018stein, wang2020information}, and mirror Langevin diffusion~\cite{hsieh2018mirrored,zhang2020wasserstein,chewi2020exponential}.

\subsubsection{Gaussian Approximation}
Parametric approximations such as those made by Gaussians have been widely used in variational inference \cite{jordan1999introduction,wainwright2008graphical,blei2017variational}.
These methods, in the Gaussian setting, aim to solve the problem 
\begin{align}
\label{eq:GVI}
    (m^\star, C^\star) = \argmin_{m, C}~\mathrm{KL}[\N(m, C) \Vert  \rho_{\rm post}].
\end{align} 
Gradient flows in the space of $(m, C)$ can be used to identify $(m^\star, C^\star)$. These gradient flows can also be understood as a Gaussian approximation of the gradient flow in the density space. The Riemannian structure in the density space can be projected to the space of Gaussians, leading to the concept of natural gradient \cite{amari1998natural,martens2020new,zhang2019fast} that has the advantage of reparametrization invariance.
Other work on the use of Gaussian variational inference methods includes the papers \cite{quiroz2018gaussian,khan2017conjugate,lin2019fast,galy2021particle, yumei2022variational}. 

In addition to their role in parametric variational inference, Gaussian approximations have been widely deployed in various generalizations of Kalman filtering~\cite{kalman1960new,sorenson1985kalman,julier1995new,wan2000unscented,evensen1994sequential}.  For Bayesian inverse problems, iterative ensemble Kalman samplers have been proposed \cite{emerick2013investigation,chen2012ensemble,wan2000unscented}. 
The paper \cite{huang2022efficient} introduced an ensemble Kalman methodology based on a novel mean-field dynamical system  that depends on its own filtering distribution.
For all these algorithms based on a Gaussian ansatz, the accuracy depends on some measure
of being close to Gaussian.
Regarding the use of Gaussian approximations in Kalman inversion we highlight, in addition to
the approximate Bayesian methods already cited, the use of ensemble Kalman methods for optimization: see \cite{iglesias2013ensemble,chada2022convergence,kim2022hierarchical,huang2022iterated,weissmann2021adaptive}. Kalman filtering has also been used in combination with variational inference~\cite{lambert2022continuous}. 
The relation between iterative Kalman filtering and Gauss-Newton or Levenberg Marquardt algorithms was studied in~\cite{bell1993iterated,bell1994iterated,huang2022iterated,chada2020iterative},
and leads to ensemble Kalman based optimization methods which are affine invariant.
\subsection{Notation}
We use $\#$ to denote the pushforward operation for general measures. It is defined via duality. More precisely consider probability measures $\mu, \nu$ in $\R^{N_{\theta}}$. Then $\nu = \varphi \# \mu$ if and only if
\[ \int f(\theta) {\rm d}\nu = \int f(\varphi(\theta)) {\rm d}\mu, \]
for any integrable $f$ under measure $\nu$. If $\mu, \nu$ admit smooth densities $\rho$ and $\tilde{\rho}$ respectively, we can use the change-of-variable formula {to derive the identity} $\tilde{\rho}(\tilde\theta) = \rho(\varphi^{-1}(\tilde{\theta}))|\nabla_{\tilde{\theta}} \varphi^{-1}(\tilde{\theta})|$.

\section{Gradient Flows}
\label{sec-Gradient-Flows}
In this section, we introduce the general formulation of gradient flows in probability density space. Methodologically, to define gradient flows, one needs a differential structure in the density space, which then leads to the definition of tangent spaces and metric tensors that determine a gradient flow. Whilst they may be rigorously defined in specific contexts\footnote{The rigorous theory of gradient flows in suitable infinite-dimensional functional spaces and its link with evolutionary PDEs is a long-standing subject; see \cite{ambrosio2005gradient, ambrosio2006gradient} for discussions and a rigorous treatment of gradient flows in metric space.}, the technical hurdles must be addressed on a case-by-case basis; we seek to keep such technicalities to a minimum and focus on a formal methodology for deriving the flow equation. {Once the flow is formally identified, we can rigorously study its large-time convergence properties and numerical approximations of the flow, which are the key factors in the understanding
and design of efficient sampling algorithms.}

For the formal methodology, we consider the probability space with smooth positive densities:
\begin{equation}
\label{eqn-smooth-positive-densities}
\mathcal{P} = \Bigl\{\rho \in C^{\infty}(\R^{N_\theta}) : \int \rho \mathrm{d}\theta = 1 ,\, \rho > 0\Bigr\}\, ,
\end{equation}
with the tangent space
\begin{equation}
\label{eqn: tangent space def}
    T_\rho \PP \subseteq \{\sigma \in C^{\infty}(\R^{N_\theta}): \int \sigma \mathrm{d}\theta = 0\}.
\end{equation}
{We assume that both the target $\rho_{\rm post}$, and density $\rho$ given by
the gradient flow, are in $\PP$: $\rho, \rho_{\rm post} \in \PP$. This allows} us to use differential structures under the smooth topology to calculate\footnote{The formal Riemannian geometric calculations in the density space were first proposed by Otto in \cite{otto2001geometry}.  The calculations in the smooth setting are rigorous if $\R^{N_{\theta}}$ is replaced by a compact manifold, as noted in \cite{Lott08}; in such case in the definition of tangent spaces, \eqref{eqn: tangent space def} also becomes identity. For rigorous results in general probability space, we refer to \cite{ambrosio2005gradient}.} gradient flow equations.

Denote an energy functional by $\EE(\cdot;\rho_{\rm post}): \PP \to \R$, {where we explicitly highlight
its dependence on $\rho_{\rm post}$ and assume that
$\rho_{\rm post}$ is the minimizer of $\EE(\cdot;\rho_{\rm post})$ over $\PP$.}
Then, a gradient flow of $\EE$ in the probability density space has the form given in \eqref{eq:GF}.
In the following we use the notation that $\langle\cdot,\cdot\rangle$ is the duality pairing between $T^*_{\rho}\PP$ and $T_{\rho}\PP$; it can be identified as the $L^2$ inner product when both arguments are classical functions. {Note that functions in $T^*_{\rho}\PP$ are not uniquely defined under the $L^2$ inner product, since 
$\langle \psi,\sigma\rangle = \langle \psi+c,\sigma \rangle$ for all $\sigma \in \T_\rho \PP$ and any constant $c$.
A unique representation can be identified by requiring, for example, that $\psi \in T_\rho^*\PP$ implies $\E_\rho[\psi]=0$.\footnote{{This choice is also naturally motivated by the definition of Fisher-Rao gradient flows; see \eqref{eq:gm}. That is, for $\psi \in T^*_{\rho}\PP$ implying $\E_{\rho}[\psi] = 0$, the Fisher-Rao metric tensor reduces to a multiplication by the density $\rho$.}} 
In the context of this paper, the most important elements in $T^*_\rho \PP$ are the first variation $\frac{\delta \mathcal{E}}{\delta \rho}$ of energy functionals $\EE(\cdot,\rho_{\rm post})$, defined by
\begin{align}
\label{eqn-first-variation-def}
\Bigl\langle \frac{\delta \mathcal{E}}{\delta \rho}, \sigma \Bigr\rangle
=  \lim_{\epsilon \rightarrow 0} \frac{\mathcal{E}(\rho + \epsilon \sigma; \rho_{\rm post}) - \mathcal{E}(\rho, \rho_{\rm post})}{\epsilon},
\end{align}
for any $\sigma \in T_\rho \PP$, assuming the limit exists.} 

The metric tensor $M(\rho): T_{\rho}\PP \to T_{\rho}^*\PP$ can be derived from a Riemannian metric $g$ that one imposes in the probability density space; we have $g_{\rho}(\sigma_1,\sigma_2) = \langle M(\rho)\sigma_1, \sigma_2\rangle$ for $\sigma_1, \sigma_2 \in T_{\rho}\PP$, and $M(\rho)^{-1}:T_{\rho}^*\PP \to T_{\rho}\PP$ is sometimes referred to as the
Onsager operator~\cite{onsager1931reciprocal,onsager1931reciprocal-II,mielke2016generalization}. For simplicity of understanding, we view $M(\rho_t)^{-1}$ in the gradient flow equation \eqref{eq:GF} as a preconditioning operator for the density dynamics. For more details of the Riemannian perspective we refer to \cite{Lott08, chen2023gradient}. We note that the gradient flow \cref{eq:GF} will converge to $\rho_{\rm post}$ if it is the unique stationary point of $\EE$. Thus numerical simulations of \cref{eq:GF} lead to sampling algorithms.

{A standard use of particle} methods to simulate dynamics of densities is to identify a mean-field stochastic
dynamical system, with state space $\R^{N_{\theta}}$, 
defined so that its law is given by  \cref{eq:GF}. 
For example, we may introduce the It\^o SDE
\begin{equation}
\begin{aligned}
\label{eq:MFD}
\dd \theta_t = f(\theta_t; \rho_t, \rho_{\rm post})\dd t + h(\theta_t; \rho_t,\rho_{\rm post}){\rm d}W_t, 
\end{aligned}
\end{equation}
where $W_t\in \R^{N_{\theta}}$ is a standard Brownian
motion. Because the drift $f: \R^{N_{\theta}} \times \PP \times \PP\rightarrow \R^{N_{\theta}}$ and diffusion coefficient $h : \R^{N_{\theta}}\times\PP \times \PP \rightarrow \R^{N_{\theta}\times N_{\theta}}$ are evaluated at $\rho_t$, the density of $\theta_t$ itself, by definition this is a mean-field 
model. Other types of mean-field dynamics, such as birth-death dynamics, and ODEs, also exist.

When utilizing \cref{eq:GF} for sampling purposes, two primary properties of the flow become our focal points. Firstly, rapid convergence of the flow to $\rho_{\rm post}$ is desired. Secondly, the numerical implementation of the flow must be tractable, and preferably, as convenient as possible, {for example through a mean-field
model such as \eqref{eq:MFD}.} Below we describe an example.

\begin{example} \label{ex:first}
Choose $\EE$ to be the KL divergence:
\begin{align}
\label{eq:energy}
\mathcal{E}(\rho; \rho_{\rm post}) = \mathrm{KL}[\rho \Vert  \rho_{\rm post}]  =\int\rho \log\Bigl(\frac{ \rho}{ \rho_{\rm post}}\Bigr)\,\dd \theta,
\end{align}
and use the Wasserstein metric tensor \cite{otto2001geometry} which satisfies $M(\rho)^{-1} \psi =  -\nabla_\theta \cdot (\rho \nabla_\theta \psi)$, then the gradient flow~\cref{eq:GF} takes the form: 
\begin{align}
\label{eq:mean-field-Wasserstein}
   \frac{\partial \rho_t}{\partial t} =  -\nabla_{\theta} \cdot (\rho_t \nabla_\theta \log \rho_{\rm post}) + \nabla_{\theta} \cdot (\nabla_{\theta} \rho_t).
\end{align}
This is the Fokker-Planck equation of the Langevin dynamics~\cite{jordan1998variational}:
\begin{equation}
\label{eq:lan}
\mathrm{d}\theta_t=\nabla_{\theta}
\log\rho_{\mathrm{post}}(\theta_t)\mathrm{d}t+\sqrt{2}\mathrm{d}W_t.
\end{equation}
{This is a trivial (in the sense that it is not actually a mean-field model) example of
\eqref{eq:MFD}. Discretizations} of this Langevin dynamics can be used to numerically implement the flow. Regarding the convergence property of \cref{eq:mean-field-Wasserstein}, we have the following classical result \cite{markowich2000trend}:
\begin{proposition}
\label{prop-convergence-wasserstein}
     {Assume that $\rho_{\rm post}$ is $\alpha$-strongly logconcave:  $\log \rho_{\rm post} \in C^2(\R^{N_{\theta}})$ and 
     \begin{align}
     \label{a:logconcave}
     -\Hess \log \rho_{\rm post}(\theta) \succeq \alpha I.
     \end{align}
Then,} for all $t\geq 0$, it holds that
\begin{align}
{\rm KL}[\rho_{t} \Vert  \rho_{\rm  post}]\leq {\rm KL}[\rho_{0} \Vert  \rho_{\rm  post}] e^{-2\alpha t}.
\end{align}
\end{proposition}
The convergence rate depends crucially on $\rho_{\rm post}$. If $\rho_{\rm post} = \N(0,I)$, the unit Gaussian distribution, the convergence rate is $e^{-2t}$. However, if the Gaussian distribution is highly anisotropic, then $\alpha$ depends on the lower bound of the eigenvalues of the {precision} matrix, which can be very small, leading to a slow convergence rate. Related convergence results with relaxed assumptions on the target distribution can be found in \cite{poincare1890equations, gross1975logarithmic, bakry1985diffusions, bakry2014analysis, villani2021topics}, based on functional inequalities, {and in \cite{mattingly2002ergodicity} using
coupling methods.} There are also many results  \cite{durmus2019analysis, shen2019randomized, vempala2019rapid, dalalyan2020sampling, chewi2021analysis, wu2022minimax} regarding the non-asymptotic convergence guarantee for discrete algorithms based on the Langevin dynamics.
\end{example}

{Note that for the choice of energy functionals in the preceding example, it is not necessary to know
the normalization constant for $\rho_{\rm post}$; this is highly desirable. On the other hand, for
the Wasserstein metric the preceding example shows that the convergence rate depends sensitively on
the target; this is undesirable. In Sections \ref{sec:Energy} and \ref{sec:GF and AI} we discuss, 
within a broader context, choices of the energy functionals and metric tensors. We identify certain 
choices that can facilitate ease of numerical implementations and fast convergence rates
{across large classes of targets}.}

\section{Choice of Energy Functionals} 
\label{sec:Energy}


In the literature, the KL divergence is the most commonly used energy functional for deriving gradient flows in the probability space. Using \cref{eqn-first-variation-def} and \cref{eq:energy} we can formally calculate 
its first variation as
\begin{align}
\label{eq:vare}
{\frac{\delta \mathcal{E}}{\delta \rho} = \log \rho- \log \rho_{\rm post} - \E_{\rho}[\log \rho- \log \rho_{\rm post}] \in T_\rho^* \PP} .
\end{align}
From \cref{eq:vare} we observe that, for the KL divergence, $\frac{\delta \mathcal{E}}{\delta \rho}$ remains unchanged if we scale $\rho_{\rm post}$ by any positive constant $c >0$, i.e. if we change $\rho_{\rm post}$ to $c\rho_{\rm post}$. This  property eliminates the need to know the normalization constant of $\rho_{\rm post}$ in order to calculate
the first variation. As a consequence, the gradient flow~\cref{eq:GF} derived from this energy functional is also independent of the normalization constant. This has significant advantages in terms of numerical implementations, as handling the unknown normalization constant is a well-known challenge in sampling.

We can also formulate the property in terms of the energy functional itself, without delving into the technicality of dealing with the mathematical well-definedness of first variations\footnote{For the $f$-divergences considered in this paper, we can show the limit in \cref{eqn-first-variation-def} exists so the first variation is well-defined. For details see the proof of \Cref{thm: KL unique f divergence} in \Cref{sec: Proofs of prop: KL unique f divergence} using a localization argument.}. An equivalent formulation is 
that $\mathcal{E}(\rho;c\rho_{\rm post})-\mathcal{E}(\rho;\rho_{\rm post})$ is independent of $\rho$, for any $c \in (0,\infty)$. {
In fact, this implies that  $$
\Bigl\langle \frac{\delta \mathcal{E}}{\delta \rho}(\rho;c\rho_{\rm post}) - \frac{\delta \mathcal{E}}{\delta \rho}(\rho; \rho_{\rm post}), \sigma \Bigr\rangle  = 0
$$
for $\sigma \in T_{\rho}\PP$ when the first variations exist; thus in such case, first variations do not depend on the scaling of $\rho_{\rm post}$.
}

Can this property be satisfied for other energy functionals? In \Cref{thm: KL unique f divergence}, we show that the answer is no among all $f$-divergences with continuously differentiable  $f$.
Here the $f$-divergence between two continuous
density functions $\rho$ and $\rho_{\rm post}$, positive everywhere, is defined as 
\[D_f[\rho \Vert \rho_{\rm post}] = \int \rho_{\rm post} f\Bigl(\frac{\rho}{\rho_{\rm post}}\Bigr) {\rm d}\theta.\]
For convex $f$ with $f(1) = 0$, Jensen's inequality implies that $D_f[\rho \Vert \rho_{\rm post}] \geq 0$. The KL divergence used in \eqref{eq:energy} corresponds to the choice $f(x)=x\log x.$ 


\begin{theorem}
\label{thm: KL unique f divergence} Assume that $f:(0,\infty) \to \mathbb{R}$ is continuously differentiable and $f(1) = 0$.
    Then the KL divergence is the only $f$-divergence (up to scalar factors) such that $D_f[\rho \Vert c\rho_{\rm post}]-D_f[\rho \Vert \rho_{\rm post}]$ is independent of $\rho \in \mathcal{P}$
    for any $c \in (0,\infty) $ and 
    for any $\rho_{\rm post} \in \mathcal{P}$.
\end{theorem}
The proof can be found in \Cref{sec: Proofs of prop: KL unique f divergence}.
As a consequence of Theorem \ref{thm: KL unique f divergence},
gradient flows defined by the energy \eqref{eq:energy} 
do not depend on the normalization 
constant of $\rho_{\rm post}$. {Hence numerical approximations of gradient
flows stemming from this energy $\mathcal{E}$ are more straightforward to implement
in comparison with the use of other divergences or {distances} as energy.}
This justifies the choice of the KL divergence as an energy functional for sampling,
and our developments in most of this paper are hence specific to the energy 
\cref{eq:energy}. 

\begin{newremark} 
Other energy functionals can be, and are, used for 
constructing gradient flows; for example, the chi-squared divergence~\cite{chewi2020svgd,lindsey2022ensemble}:
\begin{equation}
\label{eq:chi-squared}
    \chi^2(\rho \Vert \rho_{\rm post}) = \int \rho_{\rm post}\Bigl(\frac{\rho}{\rho_{\rm post}} - 1\Bigr)^2 \dd \theta
    = \int \frac{\rho^2}{\rho_{\rm post}}\dd \theta - 1.
\end{equation}

The normalization constant can appear explicitly in the gradient flow equation when general energy functionals are used. Additional techniques need to be explored to simulate such flows. For example, when the energy functional is the chi-squared divergence, in \cite{chewi2020svgd}, kernelization is used to avoid the normalization constant in the Wasserstein gradient flow. Moreover, in \cite{lindsey2022ensemble} where a modification of the Fisher-Rao metric is used, ensemble methods with birth-death type dynamics are adopted to derive numerical methods; the normalization constant can be absorbed into the birth-death rate. These techniques may change the flow equation completely or its time scales, thus affecting its convergence behavior.
\end{newremark}

\section{Choice of Metrics}
\label{sec:GF and AI}
We now fix the energy functional to be the KL divergence and discuss the choices of $M(\cdot)$. The metric tensor can significantly influence the rate at which the gradient flow converges. Within this section, we delve into a range of metrics, {focusing on \textit{invariance} properties; as we will show in this and the
next section, these play an important role in determining convergence rates of the flow
and algorithms based on it}.
In \Cref{ssec:Fisher-Rao}, we discuss the Fisher-Rao gradient flow, highlight its diffeomorphism invariance property, and showcase the remarkable uniform convergence rate to the target distribution which
stems from this invariance; it is also known that the Fisher-Rao metric is the unique diffeomorphism invariant metric (up to scaling) in the probability density space. However, implementing the Fisher-Rao gradient flow via particle methods {is resource-demanding and does not currently lead to efficient algorithms (see discussion
below).}
In \Cref{sec:gradient-flow}, we introduce a relaxed, affine invariance property for gradient flows.
It is worth noting that certain widely used gradient flows, namely the Wasserstein and Stein gradient flows, are not affine invariant. Some modifications of the Wasserstein and Stein gradient flows to achieve affine invariance are introduced in \Cref{ssec:Wasserstein} and \Cref{ssec:Stein}. These affine invariant gradient flows are more convenient to approximate numerically via particle methods compared to the diffeomorphism invariant Fisher-Rao gradient flow. 
In \Cref{sec-affine-invariance-experiments}, we conduct illustrative numerical experiments showing that these affine invariant gradient flows behave more favorably compared to their non-affine-invariant versions when sampling highly anisotropic distributions.

\subsection{Fisher-Rao Gradient Flow and Diffeomorphism Invariance}
\label{ssec:Fisher-Rao}

The Fisher-Rao Riemannian metric\footnote{For some intuitions for the definition of the Fisher-Rao Riemannian metric, see \Cref{appendix: intuition Fisher-Rao Riemannian Metric}.} \cite{amari2016information,ay2017information} is
\[g_{\rho}^{\mathrm{FR}}(\sigma_1,\sigma_2)=\int \frac{\sigma_1\sigma_2}{\rho} \mathrm{d}\theta, \text{ for } \sigma_1,\sigma_2 \in T_{\rho}\PP.\]
The metric tensor then admits the following form: {
\begin{align}
\label{eq:gm}
\M^{\mathrm{FR}}(\rho)^{-1} \psi = \rho \psi \in T_{\rho}\PP
\end{align}
for any $\psi \in T_\rho^* \PP$; recall that $\psi \in T_\rho^* \PP$ implies $\E_\rho[\psi] =0$.}
{By direct calculations, we obtain}
the Fisher-Rao gradient flow of the KL divergence as 
\begin{equation} 
\begin{aligned}    
\label{eq:mean-field-Fisher-Rao}
\frac{\partial \rho_t}{\partial t} =& -\M^{\mathrm{FR}}(\rho_t)^{-1}\frac{\delta \mathcal{E}}{\delta \rho}\Bigr|_{\rho=\rho_t}, \\ 
=&  -\rho_t \Bigl( 
\bigl( \log \rho_t - \log \rho_{\rm post}\bigr) - \E_{\rho_t}[ \log \rho_t -
\log \rho_{\rm post}] \Bigr).
\end{aligned}
\end{equation}
An important property of the Fisher-Rao gradient flow is its \textit{diffemorphism invariance}. More precisely, consider a diffemorphism in the parameter space $\varphi: \R^{N_{\theta}} \to \R^{N_{\theta}}$ and correspondingly $\tilde{\rho}_t = \varphi \# \rho_t$,$\tilde{\rho}_{\rm post} = \varphi \# \rho_{\rm post}$. Then, it holds that 
\begin{equation}
\label{eq-transformed-Fisher-Rao}
   \frac{\partial \tilde{\rho}_t}{\partial t} =  -\tilde{\rho}_t \Bigl( \bigl( \log \tilde{\rho}_t- \log \tilde{\rho}_{\rm post}\bigr) - \E_{\tilde{\rho}_t}[  \log \tilde{\rho}_t - \log \tilde{\rho}_{\rm post}] \Bigr).
\end{equation}
That is, the form of the equation {remains invariant under diffeomorphic transformations.} As a consequence, suppose there is a diffemorphism satisfying $\tilde{\rho}_{\rm post} = \varphi \# \rho_{\rm post} = \N(0,I)$, then we can study the convergence rate of \cref{eq-transformed-Fisher-Rao} for a Gaussian target distribution and the rate will apply directly to \cref{eq:mean-field-Fisher-Rao}, due to the following property of the KL divergence:
\begin{equation}
    \mathrm{KL}[\rho_t \Vert \rho_{\rm post}] =\mathrm{KL}[\varphi \# \rho_t \Vert \varphi \# \rho_{\rm post}] = \mathrm{KL}[\tilde{\rho}_t \Vert \tilde{\rho}_{\rm post}]. 
\end{equation}
The above fact implies that the Fisher-Rao gradient flow may converge at the same rate for general distributions and Gaussian distributions, indicating its favorable convergence behavior. Indeed, in 
\Cref{thm:FR-convergence} we prove its uniform convergence rate {across a wide
class of target} distributions. The proof can be found in \Cref{proof:FR-convergence}; we note that the diffemorphism invariance property is not explicitly used in the current proof.
\begin{theorem}
\label{thm:FR-convergence}  
Assume that there exist constants $K, B>0$ 
such that the initial density $\rho_0$ satisfies
\begin{align}\label{e:asup1}
    e^{-K(1+|\theta|^2)}\leq \frac{\rho_0(\theta)}{\rho_{\rm post}(\theta)}\leq e^{K(1+|\theta|^2)},
\end{align}
and both $\rho_0, \rho_{\rm post}$ have bounded second moments:
\begin{align}\label{e:asup2}
    \int |\theta|^2 \rho_0(\theta)\dd  \theta\leq B, \quad \int |\theta|^2 \rho_{\rm post}(\theta)\dd  \theta\leq B.
\end{align}
Let $\rho_t$ solve the Fisher-Rao gradient flow \eqref{eq:mean-field-Fisher-Rao}. Then, for any $t\geq \log\bigl((1+B)K\bigr)$,
\begin{align}\label{e:KLconverge}
    {\rm KL}[\rho_{t} \Vert  \rho_{\rm post}]\leq (2+B+eB)Ke^{-t}.
\end{align}
\end{theorem}
Similar propositions are presented in concurrent works~\cite[Theorem 3.3]{lu2019accelerating} and~\cite[Theorem 2.3]{lu2022birth}. Our results do not require $ {\rho_0(\theta)}/{\rho_{\rm post}(\theta)}$ to be bounded, thus applying to a broader class of distributions including Gaussians. An explicit expansion of the KL divergence along the Fisher-Rao gradient flow is also given in  
\cite{domingo2023explicit}, which indicates the optimal asymptotic convergence rate is $\bigO(e^{-2t})$.


As mentioned earlier, the diffeomorphism invariance of the Fisher-Rao gradient flow provides an intuitive explanation {as to why} it achieves such an exceptional uniform convergence. This is in sharp contrast to the Wasserstein gradient flow (\Cref{prop-convergence-wasserstein}). {Nevertheless, effective numerical approximation} the Fisher-Rao gradient flow is not straightforward, in particular because na\"ive particle methods lead to algorithms which
do not evolve the support of the measure, requiring enormous resources therefore. In \cite{lu2019accelerating, lu2022birth}, birth-death dynamics are used to simulate the flow, using kernel density estimators to approximate the current density and adding transport steps to move the support of particles; the latter in fact changes the flow equation. The diffeomorphism invariance property is also not preserved at the particle level.

On the other hand, the invariance of the gradient flow is due to the invariance of the Riemannian metric itself. The Fisher-Rao metric is known to be the only metric, up to constants, that satisfies the diffeomorphism invariance property \cite{cencov2000statistical, ay2015information, bauer2016uniqueness}. Therefore there 
{are no} alternatives if we ask for such strong diffeomorphism invariance. In the following subsection, we explore a relaxed, affine invariance property for the gradient flows. In contrast to the diffeomorphism invariance property, we are able to construct various affine invariant gradient flows, in particular, the affine invariant Wasserstein and Stein gradient flows. 
Moreover, when using particle methods to simulate the flow, we can show that the affine invariance property is also preserved at the particle level.

\subsection{Affine Invariance}
\label{sec:gradient-flow}
Roughly speaking, affine invariant gradient flows are invariant under any invertible \textit{affine transformations} of the density variables;
as a consequence, the convergence rate is independent of the affine transformation. It is
thus natural to expect that algorithms with this property have an advantage for sampling highly anisotropic target distributions. We note that the affine invariance property has been studied in MCMC algorithms \cite{goodman2010ensemble,foreman2013emcee} and certain gradient flows \cite{garbuno2020interacting,garbuno2020affine}; see the literature survey in \Cref{sec: survey affine invariance}. Our goal here is to bring this notion to general gradient flows in the density space.
\subsubsection{Affine Invariant Gradient Flow}
In the following, we define affine invariant gradient flows.

\begin{definition}[Affine Invariant Gradient Flow] 
\label{def: Affine Invariant Gradient Flow} 
Consider the gradient flow
\begin{equation}
    \frac{\partial \rho_t}{\partial t}  =-M(\rho_t)^{-1}\frac{\delta \mathcal{E}}{\delta \rho}(\rho; \rho_{\rm post})\Bigr|_{\rho = \rho_t}, 
\end{equation}
and the affine transformation $\tilde\theta = \varphi(\theta) = A\theta + b$. Let $\tilde{\rho}_t :=\varphi \# \rho_t$ and $\tilde{\rho}_{\rm post} = \varphi \# \rho_{\rm post}$.
The {\em gradient flow is affine invariant} if 
\begin{equation}
    \frac{\partial \tilde{\rho}_t}{\partial t}  =-M(\tilde{\rho}_t)^{-1}\frac{\delta \mathcal{E}}{\delta \rho}(\tilde{\rho}; \tilde{\rho}_{\rm post})\Bigr|_{\tilde\rho = \tilde\rho_t}, 
\end{equation}
for any invertible affine transformation $\varphi$.
\end{definition}

The key idea in the preceding definition is that, after the affine transformation, the dynamics of $\tilde{\rho}_t$ is itself a gradient flow, in the same metric as the gradient flow in the original variables. {Thus affine
invariance is similar to diffemorphism invariance, but with the diffemorphism constrained to invertible affine mappings.} In \Cref{ssec:Wasserstein,ssec:Stein}, we will construct concrete examples such as affine invariant Wasserstein and Stein gradient flows that satisfy the affine invariance.

\subsubsection{Affine Invariant Mean-Field Dynamics}
{As mentioned in the Fisher-Rao gradient flow, another important factor when using the flow for sampling is its numerical implementation, for example via particle methods. In the following we introduce the affine invariance property at the particle level. More specifically we consider the use of mean-field It\^o SDEs of the form \eqref{eq:MFD}
and study its affine invariance property.}
The density of solution to \eqref{eq:MFD} is governed by a nonlinear Fokker-Planck equation
\begin{equation} 
\label{eq:MFD-2}
 \frac{\partial \rho_t}{\partial t} = -\nabla_\theta \cdot(\rho_t f) + \frac12\nabla_{\theta}\cdot\bigl(\nabla_\theta \cdot(hh^T \rho_t)\bigr).
\end{equation}
By choice of $f, h$ it may be possible to ensure that \cref{eq:MFD-2} coincides with
\cref{eq:GF}. Then an interacting particle system can be used to approximate
\cref{eq:MFD}, generating an empirical measure which approximates $\rho_t.$


\begin{definition}[Affine Invariant Mean-Field Dynamics]
\label{def: affine invariant mean-field equation}
 Consider the mean-field dynamics~\cref{eq:MFD}
and the affine transformation $\tilde\theta = \varphi(\theta) = A\theta + b$, {with $A$ invertible.}
The mean-field dynamics is called affine invariant, if
\begin{subequations}
\label{eq:aMFD-AI}
\begin{align}
\begin{split}
\label{eq:aMFD-AI-f}
&Af(\theta; \rho, \rho_{\rm post}) = f(\varphi(\theta); \varphi \# \rho, \varphi \# \rho_{\rm post}), 
\end{split}\\
\begin{split}
\label{eq:aMFD-AI-sigma}
&Ah(\theta; \rho, \rho_{\rm post}) = h(\varphi(\theta); \varphi \# \rho, \varphi \# \rho_{\rm post}),
\end{split}
\end{align}
\end{subequations}
for any {such invertible} affine transformation $\varphi$. Equivalently, this implies that $\tilde{\theta}_t = \varphi(\theta_t)$ satisfies a SDE of the same form as \cref{eq:MFD}:
\begin{equation}
\begin{aligned}
\dd \tilde{\theta}_t = f(\tilde{\theta}_t; \tilde{\rho}_t, \tilde{\rho}_{\rm post})\dd t + h(\tilde{\theta}_t; \tilde{\rho}_t,\tilde{\rho}_{\rm post}){\rm d}W_t, 
\end{aligned}
\end{equation}
where $\tilde{\rho}_t = \varphi \# \rho_t$ and $\tilde{\rho}_{\rm post} = \varphi \# \rho_{\rm post}$.
\end{definition}

Again, the key idea in the preceding definition is that, after the affine transformation, the mean field dynamics \cref{eq:MFD} remains the same. We will present affine invariant mean-field dynamics for the affine invariant gradient flows we constructed.

\subsection{Affine Invariant Wasserstein Gradient Flow} 
\label{ssec:Wasserstein}
In this subsection, we discuss the Wasserstein metric tensor\footnote{For completeness, we include some intuitions regarding the definition of the Wasserstein Riemannian metric in \Cref{appendix: intuition Wasserstein Riemannian Metric}.}  \cite{otto2001geometry, jordan1998variational, Lott08} that has been widely used to construct gradient flows. Define 
$\psi_\sigma$ to be the solution of the PDE 
\begin{equation}
\label{eq:liouville}
-\nabla_{\theta} \cdot (\rho \nabla_{\theta} \psi_\sigma)=\sigma.
\end{equation}
This definition requires specification of function spaces to ensure 
unique invertibility
of the divergence form elliptic operator; here we use it formally to derive flow equations.
The Wasserstein Riemannian metric has the form
\begin{equation}
\label{eqn-Wasserstein-Riemannian-metric}
    g_{\rho}^{\mathrm{W}}(\sigma_1,\sigma_2) = \int \rho(\theta) \nabla_{\theta} \psi_{\sigma_1}(\theta)^T  \nabla_{\theta} 
\psi_{\sigma_2}(\theta) \mathrm{d}\theta.
\end{equation}
The corresponding metric tensor satisfies {(see Example \ref{ex:first})}
\begin{align}
M^{\mathrm{W}}(\rho)^{-1} \psi = -\nabla_\theta \cdot (\rho \nabla_\theta \psi) \in T_{\rho}\PP.
\end{align}


{Recall the form of the Wasserstein gradient flow of the KL divergence given in Example \ref{ex:first},
along with the (trivial) mean-field Langevin dynamics.
The Wasserstein gradient flow is not affine invariant. Indeed, from \Cref{prop-convergence-wasserstein}, we observe that its convergence rate is not invariant upon affine transformations of the target distribution. To make it affine invariant we consider the following generalized Wasserstein metric tensor}
\begin{equation}
    \M^{\mathrm{AIW}}(\rho)^{-1} \psi = -\nabla_\theta \cdot (\rho P(\theta,\rho) \nabla_\theta \psi) \in T_{\rho}\PP,
\end{equation}
where $\Prec: \R^{N_\theta}\times \PP \to \R^{N_\theta\times N_\theta}_{\succ 0}$ is a preconditioner satisfying the following condition:

\begin{condition}
\label{condition:Preconditioner}
Consider any invertible affine transformation $\tilde\theta = \varphi(\theta) = A \theta + b$ and correspondingly $\tilde{\rho} = \varphi \# \rho$.  The preconditioning matrix satisfies
\begin{align} 
\label{eq:P-affine}
\Prec(\tilde\theta, \tilde\rho) = A \Prec(\theta, \rho) A^T.
\end{align}
\end{condition}

\begin{newremark}
    Examples of preconditioning matrices that satisfy \cref{eq:P-affine} include the covariance matrix $P(\theta,\rho) = C(\rho)$ and some local preconditioners,  such as \[P(\theta,\rho) = \bigl(\theta - m(\rho)\bigr)\bigl(\theta - m(\rho)\bigr)^T,\] or more generally, \[P(\theta,\rho) = \int (\theta' - m(\rho))(\theta' - m(\rho))^T\kappa(\theta,\theta', \rho) \rho(\theta') {\rm d}\theta'.\]
    Here $\kappa: \R^{N_{\theta}}\times\R^{N_{\theta}}\times\PP \rightarrow \R$ is a positive definite kernel for any fixed $\rho$, and it is affine invariant, namely $\kappa(\tilde\theta, \tilde\theta', \tilde\rho) = \kappa(\theta, \theta', \rho)$ 
    under any invertible affine transformation $\tilde\theta = \varphi(\theta) = A \theta + b$ and correspondingly $\tilde{\rho} = \varphi \# \rho$. A potential choice is $$ \kappa(\theta, \theta', \rho) = \exp\left(-\frac{1}{2}(\theta - \theta')^TC(\rho)^{-1}(\theta - \theta')\right).$$
\end{newremark}
When such condition is satisfied, the gradient flow is affine invariant:
\begin{theorem}
\label{thm:Wasserstein-affine-invariant}
Under the assumption on $P$ given in 
\Cref{condition:Preconditioner},
the associated gradient flow of the KL divergence with respect to the metric tensor $M^{\mathrm{AIW}}$, namely
\begin{equation}
\begin{aligned}    
\label{eq:AI-WGF}
\frac{\partial \rho_t(\theta)}{\partial t} 
&= \nabla_{\theta} \cdot \Bigl(\rho_t \Prec(\theta, \rho_t) (\nabla_{\theta} \log \rho_t  - \nabla_{\theta} \log \rho_{\rm post})\Bigr),
\end{aligned}
\end{equation}
is affine invariant.
\end{theorem}
The proof of this theorem is provided in \Cref{proof:Wasserstein-affine-invariant}. Henceforth we refer to $M^{\mathrm{AIW}}$ satisfying the condition of the
preceding proposition as an affine invariant Wasserstein metric tensor. For the specific choice of the preconditioner $P(\theta,\rho) = C(\rho)$, the covariance, the flow is called {the Kalman-Wasserstein gradient flow \cite{garbuno2020interacting,garbuno2020affine}; the underlying metric structure was first identified in \cite{reich2015probabilistic}.}
When $\rho_{\rm post}$ belongs to the Gaussian family, the Kalman-Wasserstein gradient flow provably achieves convergence rate $\bigO(e^{-t})$ \cite{garbuno2020interacting, garbuno2020affine,carrillo2021wasserstein}. It is worth mentioning that there will be an additional computational cost associated with evaluating $P(\theta,\rho)$. However, 
for many practical problems, the dominant computational expense lies in evaluating $\Phi_R$ and its derivatives, and the additional computational cost is negligible in comparison.

We may use the following mean field dynamics to simulate the affine invariant Wasserstein gradient flow~\eqref{eq:AI-WGF}:
\begin{equation} 
\label{eq:AI-Wasserstein-MD-Ct}
\begin{aligned}
\dd \theta_t &= 
\Prec(\theta_t, \rho_t)\nabla_{\theta} \log \rho_{\rm post}(\theta_t)\dd t\\
&\quad\quad\quad+ \Bigl(\bigl(D(\theta_t,\rho_t) - \Prec(\theta_t, \rho_t) \bigr)\nabla_{ \theta} \log \rho_t(\theta_t) + d(\theta_t, \rho_t)  
\Bigr)\dd t +  h(\theta_t, \rho_t){\rm d}W_t.
\end{aligned}
\end{equation}
Here $h : \R^{N_{\theta}} \times \PP   \rightarrow \R^{N_{\theta}\times N_{\theta}}$,
 $D(\theta,\rho) = \frac{1}{2}h(\theta,\rho)h(\theta,\rho)^T$ and 
$d(\theta,\rho) = \nabla_{\theta} \cdot D(\theta,\rho)$.    
 When $D(\theta,\rho_t)  -P(\theta_t,\rho_t) \neq 0$, the equation requires
knowledge of the score function $\nabla_{ \theta} \log \rho_t(\theta_t)$ of the current density; for this purpose, various approaches have been adopted in the literature \cite{maoutsa2020interacting, wang2022optimal, shen2022self, boffi2022probability}; see also \cite{song2020sliced} and references therein
for discussions of score estimation. 


At the particle level, we can prove the above dynamics is affine invariant; see \Cref{lem:AI-Wasserstein-MD} and its proof in \Cref{proof:AI-Wasserstein-MD}.
\begin{theorem}
\label{lem:AI-Wasserstein-MD}
    The mean-field dynamics~\cref{eq:AI-Wasserstein-MD-Ct} is affine invariant under the assumption on the preconditioner $P$ given in \Cref{condition:Preconditioner} 
    and the assumptions on $h$ given in \cref{eq:aMFD-AI-sigma}.
\end{theorem}
In particular, let $C(\rho)$ denote the covariance matrix of $\rho.$
If we take $\Prec(\theta, \rho) = C(\rho)$ 
 and $h(\theta,\rho) = \sqrt{2C(\rho)}$, then we get the following affine invariant overdamped Langevin equation, introduced in \cite{garbuno2020interacting,garbuno2020affine}:
\begin{equation}
\begin{aligned}    
\label{eq:AI-Wasserstein-MD-Ct2}
\mathrm{d}\theta_t = 
C(\rho_t)\nabla_{\theta} \log \rho_{\rm post}(\theta_t)\mathrm{d}t 
+ \sqrt{2C(\rho_t)}{\rm d}W_t.
\end{aligned}
\end{equation}
Comparison with \cref{eq:lan} demonstrates that it is a preconditioned
{version of the standard overdamped Langevin dynamics \eqref{eq:mean-field-Wasserstein}.}

\subsection{Affine Invariant Stein Gradient Flow} 
\label{ssec:Stein}
In this subsection we discuss the widely used Stein metric tensor\footnote{For completeness, we include some intuitions in the definition of the Stein Riemannian metric in \Cref{appendix: intuition Stein Riemannian Metric}.} for constructing gradient flows. 
The Stein gradient flow \cite{liu2017stein, lu2019scaling, duncan2019geometry} is the continuous limit of the Stein variational gradient descent~\cite{liu2016stein}.


Let $\psi_\sigma$ solve the integro-partial
differential equation
\begin{equation}
\label{eq:liouville2}
 -\nabla_{\theta} \cdot \Bigl(\rho(\theta)  \int \kappa(\theta,\theta',\rho) \rho(\theta')\nabla_{\theta'} \psi_\sigma(\theta') \mathrm{d}\theta'\Bigr)=\sigma(\theta).
\end{equation}
Here $\kappa: \R^{N_{\theta}}\times\R^{N_{\theta}}\times\PP \rightarrow \R$ is a positive definite kernel for any fixed $\rho$. As before suitable function spaces need to be specified to ensure that
this equation is uniquely solvable; here we will treat it formally to derive flow equations. The Stein Riemannian metric is 
\begin{equation}
    g_{\rho}^{\mathrm{S}}(\sigma_1,\sigma_2) = \int\int \kappa(\theta,\theta',\rho)\rho(\theta)\nabla_\theta \psi_{\sigma_1}(\theta)^T \nabla_{\theta'}\psi_{\sigma_2}(\theta') \rho(\theta')\mathrm{d}\theta\mathrm{d}\theta'.
\end{equation}
The corresponding metric tensor satisfies
\begin{align}
\M^{\mathrm{S}}(\rho)^{-1} \psi =   -\nabla_{\theta} \cdot \Bigl(\rho(\theta)  \int \kappa(\theta,\theta',\rho) \rho(\theta')\nabla_{\theta'} \psi(\theta') \mathrm{d}\theta'\Bigr) \in T_{\rho}\PP.
\end{align}

The Stein gradient flow has the form
\begin{equation}
\begin{aligned}    
\label{eq:Stein-GF}
\frac{\partial \rho_t(\theta)}{\partial t} &=-\Bigl(\M^{\mathrm{S}}(\rho_t)^{-1}\frac{\delta \mathcal{E}}{\delta \rho}\Bigr|_{\rho=\rho_t}\Bigr)(\theta) \\
&= \nabla_{\theta}\cdot\Bigl(\rho_t(\theta)\int \kappa(\theta,\theta',\rho_t)\rho_t(\theta')\nabla_{\theta'} \bigl(\log\rho_t(\theta') - \log \rho_{\rm post}(\theta') \bigr)\mathrm{d}\theta' \Bigr),
\end{aligned}
\end{equation}
and it has the mean-field counterpart~\cite{liu2016stein,liu2017stein} in $\theta_t$ with the
law $\rho_t$:
\begin{equation}
\begin{aligned}    
\label{eq:particle-Stein}
\frac{\mathrm{d}\theta_t}{\mathrm{d}t} 
&= 
\int \kappa(\theta_t,\theta',\rho_t)\rho_t(\theta') \nabla_{\theta'}\bigl(\log \rho_{\rm post}(\theta') - \log\rho_t(\theta') \bigr)\mathrm{d}\theta' \\
&= 
\int \kappa(\theta_t,\theta',\rho_t)\rho_t(\theta') \nabla_{\theta'}\log \rho_{\rm post}(\theta') +  \rho_t(\theta')\nabla_{\theta'}\kappa(\theta_t,\theta', \rho_t)\mathrm{d}\theta'.
\end{aligned}
\end{equation}
Here, the second equality is obtained using integration by parts; it facilitates
an expression that avoids the score (gradient of the log density function of $\rho_t$). 
This is useful because, when implementing particle methods, the resulting integral can then be approximated directly by Monte Carlo methods.

The Stein gradient flow is in general not affine invariant. To make it affine invariant, we introduce the following metric tensor
\begin{align}
\M^{\mathrm{AIS}}(\rho)^{-1} \psi = - \nabla_{\theta}\cdot\Bigl(\rho(\theta)\int \kappa(\theta,\theta',\rho)\rho(\theta')\Prec( \theta, \theta',\rho)\nabla_{\theta'}\psi(\theta')\mathrm{d}\theta' \Bigr) \in T_{\rho}\PP.
\end{align}
Here, the preconditioner $\Prec: \R^{N_{\theta}}\times\R^{N_{\theta}}
\times \PP \to \R^{N_{\theta}\times N_{\theta}}$
is positive definite, which may take the form
of $\Prec(\theta, \theta',\rho) = L(\theta,\rho)L(\theta',\rho)^T$.
The kernel function and the preconditioner satisfy the following condition:
\begin{condition}
\label{Condition:Preconditioner-kernel}
Consider any invertible affine transformation $\tilde\theta = \varphi(\theta) = A \theta + b$ and correspondingly $\tilde{\rho} = \varphi \# \rho$; moreover $\tilde{\theta}' = \varphi(\theta')$. The preconditioning matrix and the kernel satisfy
$$ \kappa(\tilde\theta, \tilde\theta',\tilde\rho)\Prec(\tilde \theta, \tilde \theta',  \tilde\rho)
= \kappa(\theta, \theta', \rho) A \Prec(\theta, \theta', \rho) A^T.$$
\end{condition}

\begin{newremark}
    Examples satisfying \Cref{Condition:Preconditioner-kernel} can be obtained using an affine invariant kernel that fulfills the condition $\kappa(\tilde\theta, \tilde\theta', \tilde\rho) = \kappa(\theta, \theta', \rho)$, along with preconditioning matrices that satisfy \cref{eq:P-affine}.
\end{newremark}
Then, the corresponding gradient flow becomes affine invariant:
\begin{theorem}
\label{proposition:Stein-affine-invariant}
Under the assumption on $P$ and $\kappa$ given in \Cref{condition:Preconditioner}.
The associate gradient flow of the KL divergence with respect to the metric tensor $M^{\rm AIS}$, namely
\begin{equation}
\begin{aligned}    
\label{eq:AI-Stein}
\frac{\partial \rho_t(\theta)}{\partial t} &= \nabla_{\theta}\cdot(\mathsf{f})\\
\mathsf{f}&= \Bigl(\rho_t(\theta)\int \kappa(\theta,\theta',\rho_t)\rho_t(\theta')\Prec( \theta, \theta',\rho_t)\nabla_{\theta'} \bigl(\log\rho_t(\theta') - \log \rho_{\rm post}(\theta') \bigr)\dd \theta' \Bigr),
\end{aligned}
\end{equation}
is affine invariant.
\end{theorem}
The proof of this theorem can be found in~\Cref{proof:Stein-affine-invariant}. Furthermore, we can construct the following mean-field dynamics to simulate the flow:
\begin{equation}
\begin{aligned}    
\label{eq:AI-Stein-MD}
\frac{\dd \theta_t}{\dd t} 
&= 
\int \Bigl(\kappa(\theta_t,\theta', \rho_t)\rho_t(\theta') \Prec( \theta_t, \theta', \rho_t ) \nabla_{\theta'}\log \rho_{\rm post}(\theta')\\
&\qquad\qquad\qquad\qquad\qquad\qquad +  \nabla_{\theta'}\cdot (\kappa(\theta_t,\theta',\rho_t) \Prec(\theta_t, \theta', \rho_t) ) \rho_t(\theta')\Bigr) \dd \theta'.
\end{aligned}
\end{equation}
Similarly, we can prove the affine invariance of the above mean field dynamics; see \Cref{lem:AI-Stein-MD} and its proof in \Cref{proof:AI-Stein-MD}.
\begin{theorem}
\label{lem:AI-Stein-MD}
The mean-field dynamics \cref{eq:AI-Stein-MD} is affine invariant under \Cref{Condition:Preconditioner-kernel}.
\end{theorem}
For a survey of convergence analysis of these gradient flows we refer to \cite[Section 3.5]{chen2023gradient}.
\subsection{Illustrative Numerical Experiments}
\label{sec-affine-invariance-experiments}
In this subsection, we conduct numerical experiments comparing the affine invariant Wasserstein and Stein gradient flows with their non-affine invariant versions. We focus our experiments on three two-dimensional  posterior distributions. In defining
them we use the notation $\theta = [\theta^{(1)},\theta^{(2)}]^T \in \R^2$.
\begin{enumerate}
    \item \emph{Gaussian Posterior.} 
        $$
    \V(\theta) = \frac{1}{2}\theta^{T}
    \begin{bmatrix}
      1 & 0\\
      0 & \lambda
    \end{bmatrix}
    \theta
    \quad \textrm{with} \quad \lambda = 0.01,\,0.1,\,1.
    $$
    We initialize the gradient flows from 
    $$
    \theta_0 \sim \N\Bigl(
    \begin{bmatrix}
  10\\
  10
    \end{bmatrix}
, 
\begin{bmatrix}
  \frac{1}{2} & 0\\
  0 & 2
\end{bmatrix}\Bigr).
$$
    \item \emph{Logconcave Posterior}. 
    $$
\V(\theta) = \frac{(\sqrt{\lambda}\theta^{(1)} - \theta^{(2)})^2}{20} +\frac{(\theta^{(2)})^4}{20} \quad \textrm{with} \quad \lambda = 0.01,\,0.1,\,1.
$$
We initialize the gradient flows from 
$
\theta_0 \sim \N\Bigl(
\begin{bmatrix}
  10\\
  10
\end{bmatrix}
, 
\begin{bmatrix}
  4 & 0\\
  0 & 4
\end{bmatrix}\Bigr)
$.

    \item \emph{General Posterior.} 
$$
\V(\theta) = \frac{\lambda( \theta^{(2)} - (\theta^{(1)})^2)^2}{20} + \frac{(1 - \theta^{(1)})^2}{20}  \quad \textrm{with} \quad \lambda = 0.01,\,0.1,\,1.
$$
This example is known as the Rosenbrock function~\cite{goodman2010ensemble}. We initialize the gradient flows from 
$$
\theta_0 \sim \N\Bigl(
\begin{bmatrix}
  0\\
  0
\end{bmatrix}
, 
\begin{bmatrix}
  4 & 0\\
  0 & 4
\end{bmatrix}\Bigr).
$$
\end{enumerate}
In all the examples, the parameter $\lambda$ controls the anisotropy of the distribution.
We will use several summary statistics to compare the resulting samples with the ground truth. They are $\E[\theta]$, the covariance $\Cov[\theta]$, and $\E[\cos(\omega^T \theta + b)]$; for the last summary statistics we randomly draw $\omega\sim\N(0, I)$ and $b\sim \textrm{Uniform}(0, 2\pi)$ and report the average over 20 random draws of $\omega$ and $b$. The ground truths of these summary statistics are evaluated by integrating $\rho_{\rm post}$ numerically (See \Cref{sec:integration} for detailed formula of numerical integration). We will compare the performance of the following gradient flows (GFs) in our experiments. 
\begin{itemize}
    \item Wasserstein GF: The Wasserstein gradient flow, which is implemented by the Langevin dynamics 
    \eqref{eq:lan}.
    \item Affine invariant Wasserstein GF: The affine invariant Wasserstein gradient flow with $P(\theta, \rho) = C(\rho)$ and $h(\theta,\rho) = \sqrt{2C(\rho)}$; this is implemented by the stochastic interacting particle approximation of the
    {preconditioned Langevin dynamics} \eqref{eq:AI-Wasserstein-MD-Ct2}.
    \item Stein GF: The Stein gradient flow with 
\begin{equation}
\label{eq:Stein-kernel}
    \kappa(\theta, \theta', \rho) = (1 + 4\log(J+1)/N_{\theta})^{N_{\theta}/2}\exp(-\frac{1}{h}\lVert\theta - \theta'\rVert^2),
\end{equation}
which is implemented as a deterministic interacting particle dynamics~\eqref{eq:particle-Stein}.
Here $h = \textrm{med}^2/\log(J+1)$; $\textrm{med}^2$ is the squared median of the pairwise Euclidean distance between the current particles, following~\cite{liu2016stein}.
    \item Affine invariant Stein GF: The affine invariant Stein gradient flow with 
\begin{equation}
\label{eq:affine-invariant-Stein-kernel}
P = C(\rho), \quad \kappa(\theta, \theta',\rho) =(1 + 2/N_\theta)^{N_{\theta}/2}\exp(-\frac{1}{2N_\theta}(\theta - \theta')^TC(\rho)^{-1}(\theta - \theta')),
\end{equation}
which is implemented as a deterministic interacting particle dynamics~\eqref{eq:AI-Stein-MD}.
\end{itemize}
Here the scaling constants\footnote{In the case of the Stein GF, there is no analytical formula for the integral~\cref{eq:kernel-const}. To deal with this, we estimate the scaling constant by replacing 
$\textrm{med}^2\I$ with $N_\theta C(\rho)$, for which we can analytically compute~\cref{eq:kernel-const}.} in the above definition of kernel functions are chosen such that
\begin{align}
\label{eq:kernel-const}
    \int \int \kappa(\theta, \theta', \rho) \N(\theta, m, C(\rho)) \N(\theta', m, C(\rho)) \dd \theta \dd \theta' = 1.
\end{align} 
This choice makes the Stein gradient flows comparable with the Wasserstein 
gradient flows in terms of time scales, for a fair comparison. In practical implementation, it is recommended to drop these scaling constants.
 All the gradient flows are implemented by interacting particle systems with $J =100$ particles (except for the Wasserstein GF, {using \eqref{eq:lan}, where particles do not interact}).

For the Gaussian posterior, the convergence of different gradient flows, according to the three summary statistics, are presented in~\Cref{fig:Gaussian_gd_particle_converge}.  The imposition of the affine invariance property makes the convergence rate independent of the anisotropy $\lambda$ and accelerates the sampling for badly scaled Gaussian ($\lambda = 0.01$).
We note that all these gradient flows 
do not converge within machine precision because of the limited number of particles.

\begin{figure}[ht]
\centering
    \includegraphics[width=0.8\textwidth]{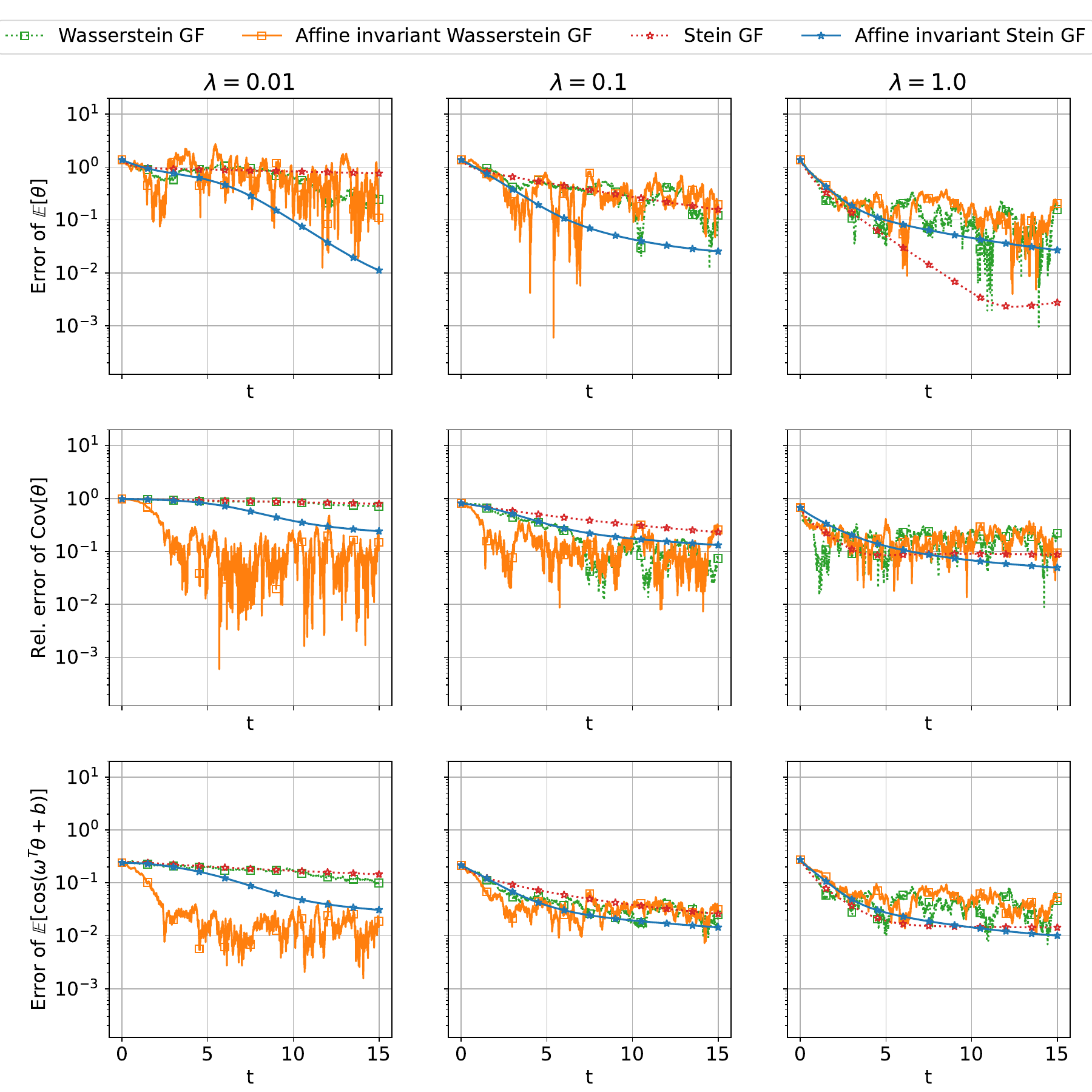}
    \caption{Gaussian posterior case: convergence of different gradient flows in terms of the $L^2$ error of $\E[\theta_t]$, the relative Frobenius norm error of the covariance $ \frac{\lVert \Cov[\theta_t] - \Cov[\theta_{{\rm true}}]\rVert_F}{\lVert \Cov[\theta_{\rm true}]\rVert_F}$, and the error of $\E[\cos(\omega^T \theta_t + b)]$. }
    \label{fig:Gaussian_gd_particle_converge}
\end{figure}

For the logconcave posterior, the results are shown in \Cref{fig:Logconcave_gd_particle_converge}. Again, the imposition of the affine invariance property makes the convergence rate independent of the anisotropy $\lambda$ and accelerates the sampling in the highly anisotropic case ($\lambda = 0.01$).
\begin{figure}[ht]
\centering
    \includegraphics[width=0.8\textwidth]{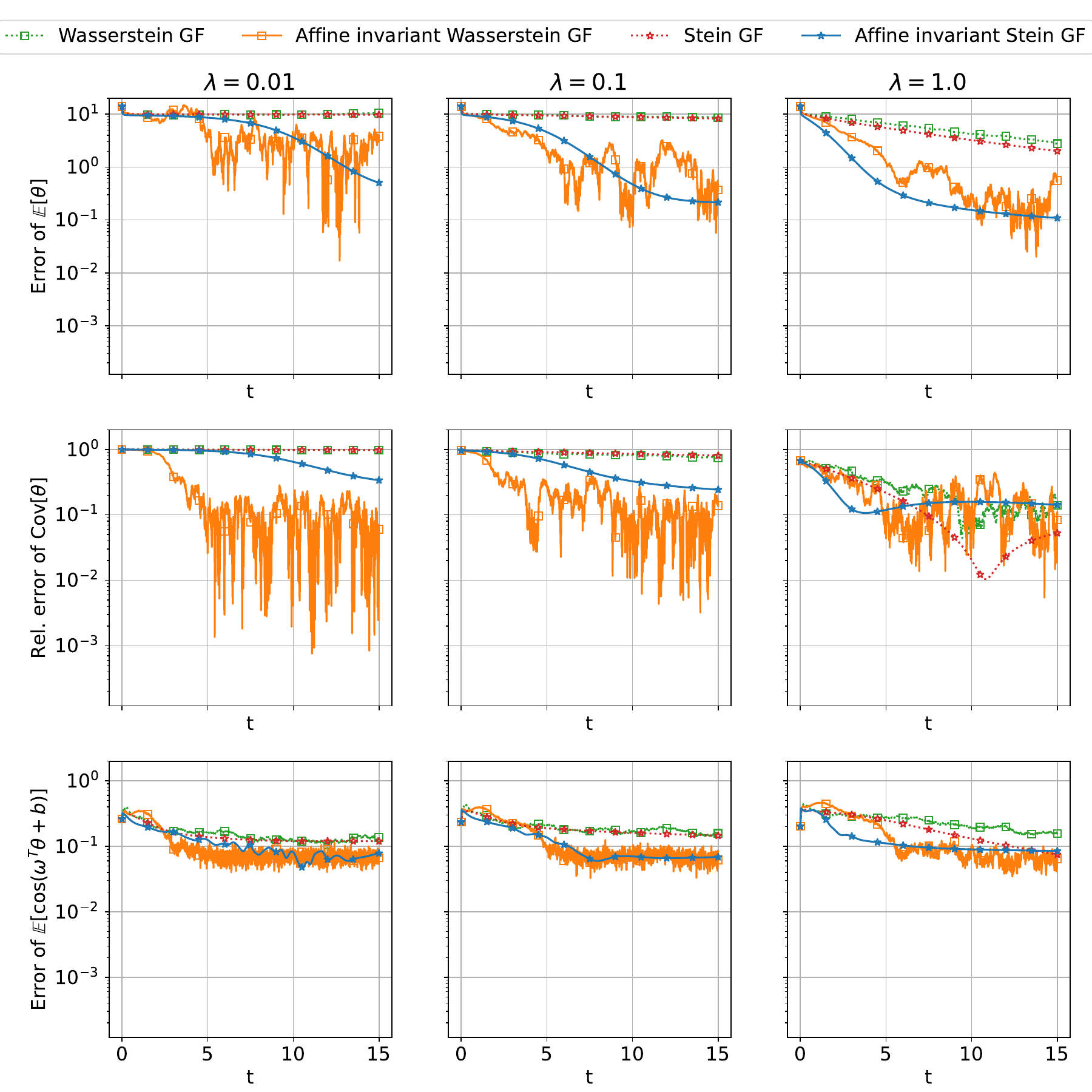}
    \caption{Logconcave posterior case: convergence of different gradient flows in terms of the $L^2$ error of $\E[\theta_t]$, the relative Frobenius norm error of the covariance $ \frac{\lVert \Cov[\theta_t] - \Cov[\theta_{{\rm true}}]\rVert_F}{\lVert \Cov[\theta_{\rm true}]\rVert_F}$, and the error of $\E[\cos(\omega^T \theta_t + b)]$.}
    \label{fig:Logconcave_gd_particle_converge}
\end{figure}

For the general posterior, we note that the Rosenbrock function is a non-convex function. Its minimum is at $[1, 1]$. The expectation and covariance of the posterior density function is (See \Cref{sec:integration}) 
$$
\E[\theta] = \begin{bmatrix}
  1\\
  11
\end{bmatrix} \qquad 
\Cov[\theta] = \begin{bmatrix}
  10& 20\\
  20& \frac{10}{\lambda} + 240
\end{bmatrix}.
$$ The particles obtained by different gradient flows at $t=15$ are depicted in~\Cref{fig:Rosenbrock_gd_particle_density}, and their convergence behaviors according to the three summary statistics are depicted in~\Cref{fig:Rosenbrock_gd_particle_converge}.
For small $\lambda$ (e.g., $\lambda = 0.01$), ${\theta^{(2)}}$ is the stretch direction, and therefore the imposition of the affine invariance property makes the convergence faster.  
However, when $\lambda$ increases, the posterior density concentrates on
a manifold with significant curvature (See \Cref{fig:Rosenbrock_gd_particle_density}). 
Although the particle positions match well with the density contours, the convergence of summary statistics significantly deteriorates; the imposition of 
affine invariance does not {significantly relieve the issue.} This indicates that additional structures need to be explored for sampling {densities such as this, which
concentrate} on a curved manifold.
\begin{figure}[ht]
\centering
    \includegraphics[width=0.8\textwidth]{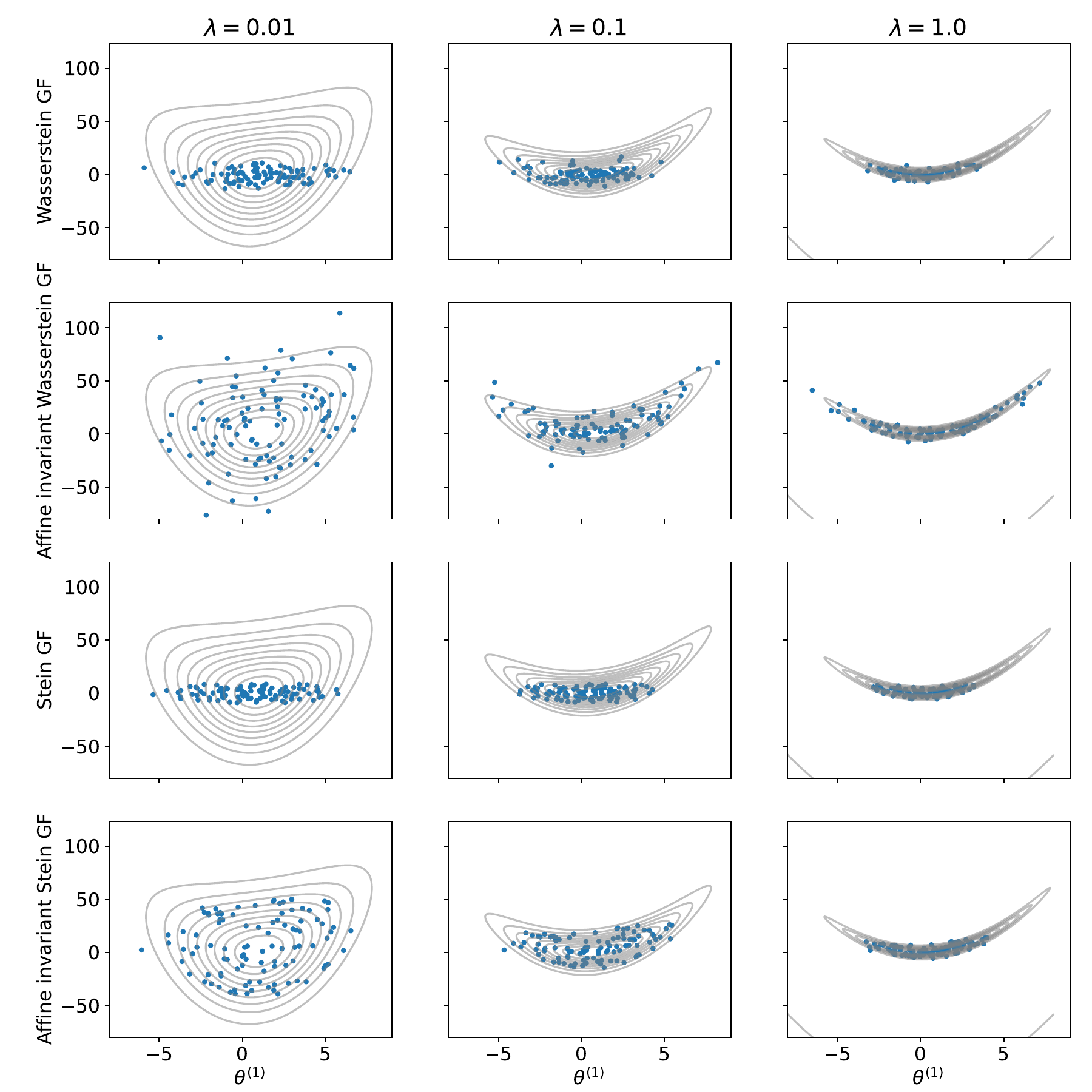}
    \caption{General posterior case: particles obtained by different gradient flows at $t=15$. Grey lines represent the contour of the true posterior.}
    \label{fig:Rosenbrock_gd_particle_density}
\end{figure}

\begin{figure}[ht]
\centering
    \includegraphics[width=0.8\textwidth]{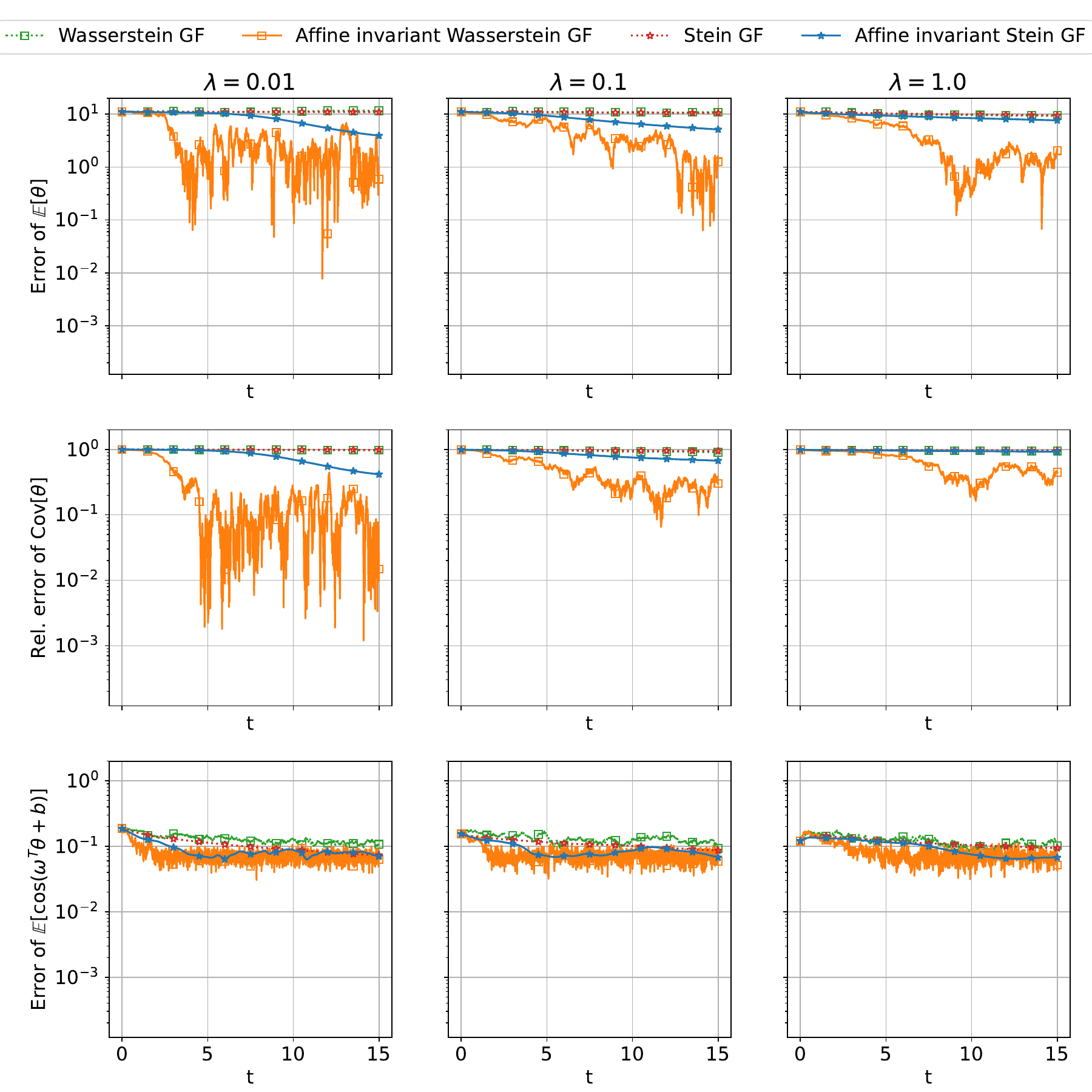}
    \caption{General posterior case: convergence of different gradient flows in terms of the $L^2$ error of $\E[\theta_t]$, the relative Frobenius norm error of the covariance $ \frac{\lVert \Cov[\theta_t] - \Cov[\theta_{{\rm true}}]\rVert_F}{\lVert \Cov[\theta_{\rm true}]\rVert_F}$, and the error of $\E[\cos(\omega^T \theta_t + b)]$.}
    \label{fig:Rosenbrock_gd_particle_converge}
\end{figure}

Moreover, the convergence curves for the affine invariant Wasserstein gradient flow, implemented with the Langevin dynamics, oscillate slightly due to the added noise; those obtained from the affine invariant Stein gradient flow, implemented by Stein variational gradient descent, are smooth~(See \Cref{fig:Gaussian_gd_particle_converge,fig:Logconcave_gd_particle_converge,fig:Rosenbrock_gd_particle_converge}). The added noise helps for sampling non-Gaussian, highly anisotropic posteriors in comparison with the affine invariant Stein gradient flow~(See \Cref{fig:Rosenbrock_gd_particle_density,fig:Rosenbrock_gd_particle_converge}).

As a summary, affine invariant gradient flows, when implemented via particle methods, typically outperform their non-affine invariant versions, especially when the target distribution is highly anisotropic. 
\section{Gaussian Approximate Gradient Flow}
\label{sec:GGF}
An alternative to using particle methods to approximate a density is through parametric approximations. In this section, we focus on Gaussian approximations of the gradient flows in the density space. {The resulting Gaussian approximate gradient flows can be understood as gradient flows in the parameter space of Gaussian distributions, 
within the framework of variational inference} \cite{jordan1999introduction,wainwright2008graphical,blei2017variational}.

In \Cref{ssec:GGF} and \Cref{sec-gaussian-moment-closure} we introduce Gaussian approximations based on metric projection and moment closure respectively. We show their equivalence under certain conditions in \Cref{sec-equivalence-gaussian-approx}. Based on this equivalence, we use the moment closure approach to derive different Gaussian approximate gradient flows. In \Cref{ssec:Gaussian-Fisher-Rao}
we introduces several explicit Gaussian approximate gradient flows, establish their connections and discuss their relation to natural gradient methods in variational inference. In \Cref{ssec:GGF-converge} we provide convergence analysis for these Gaussian approximate gradient flows, showing that the affine invariance property is beneficial for the convergence rates of the flows, for Gaussian and logconcave target distributions. We conduct illustrative numerical experiments in   \Cref{sec-gaussian-numerics} to compare these gradient flows further; the results validate our theoretical analysis.

\subsection{Gaussian Approximation via Metric Projection}

\label{ssec:GGF}
Consider the Gaussian density space
\begin{subequations}
\label{eq:PPG}
\begin{align}
    \PPG &:= \Bigl\{\rho_{\rhoa}: \rho_{\rhoa}(\theta) = \frac{\exp\bigl({-\frac{1}{2}(\theta - m)^T C^{-1} (\theta - m)}\bigr)}{\sqrt{|2\pi C|}} \textrm{ with } \rhoa =(m,C) \in \mathcal{A}\Bigr\},\label{eq:PG-a}\\ 
    \mathcal{A}&=
    \Bigl\{(m,C) : m\in\R^{N_\theta},\,C\succ 0\in \R^{N_\theta \times N_\theta}
    \Bigr\}.
\end{align}
\end{subequations}
Instead of minimizing $\rho$ over $\PP$ as in the gradient flow in the whole probability density space, one can minimize over 
Gaussian density space $\PPG \subset \PP$. Given an energy functional $\EE$, this consideration leads to the following constrained minimization problem:
\begin{equation}
\label{eq:Gaussian-KL0}
\min_{a \in \mathcal{A}} \EE(\rho_a;\rho_{\rm post}) .
\end{equation}
This forms the basis of Gaussian variational inference. For the reasons explained in \Cref{sec:Energy} we choose $\EE$
to be the KL divergence.
Given a metric tensor $M(\rho)$ in the tangent space of $\PP$, we can define the projected metric tensor in the tangent space of $\mathcal{A}$ as
\begin{equation}
\label{eq:mra}
\fM(\rhoa):=\int \nabla_\rhoa \rho_\rhoa(\theta) \bigl(M(\rho_a)\nabla_\rhoa \rho_\rhoa^T(\theta)\bigr)\dd \theta,
\end{equation}
{defined by}
\begin{equation}
\label{eq:mra2}
\langle  \fM(a)\sigma_1, \sigma_2 \ranglen  = \langle M(\rho_{a})\nabla_\rhoa \rho_{\rhoa} \cdot \sigma_1, \nabla_\rhoa \rho_{\rhoa} \cdot \sigma_2 \rangle, \quad \forall \sigma_1,\sigma_2\in\R^{N_{\rhoa}},
\end{equation}
where $N_{\rhoa}$ is the dimension of the parameter $a$.
{Given the metric tensor,} we can write down the resulting gradient flow equation in $\mathcal{A}$:
\begin{align}
\label{eq:G-GF}
   \frac{\partial \rhoa_t}{\partial t} = -\fM(\rhoa_t)^{-1}\left.\frac{\partial  \EE(\rho_{\rhoa};\rho_{\rm post})}{\partial \rhoa } \right|_{\rhoa=\rhoa_t}.
\end{align}
This leads to a flow of the density $\rho_{a_t}$ that stays in the Gaussian space $\PP^G$. We remark that the above metric projection approach is standard in Riemannian geometry to obtain constrained metrics and gradient flows \cite{do1992riemannian}.

\subsection{Gaussian Approximation via Moment Closures}
\label{sec-gaussian-moment-closure}
For any gradient flow in the density space \cref{eq:GF}, we can also consider the following moment closure approach to get a Gaussian approximation. First, we write down
evolution equations for the mean and covariance under \cref{eq:GF} noting that they
satisfy the following identities:
\begin{equation}
\begin{aligned}
\label{eq:mC-Momentum}
&\frac{\dd  m_t}{\dd t} = \frac{\dd }{\dd t}\int  \rho_t(\theta) \theta \dd \theta = -\int 
M(\rho_t)^{-1}\frac{\delta \mathcal{E}}{\delta \rho}\Bigr|_{\rho = \rho_t} \theta \dd \theta, \\
&\frac{\dd  C_t}{\dd t} = \frac{\dd }{\dd t}\int  \rho_t(\theta) (\theta - m_t)(\theta - m_t)^T \dd \theta = -\int M(\rho_t)^{-1}\frac{\delta \mathcal{E}}{\delta \rho}\Bigr|_{\rho = \rho_t}  (\theta - m_t)(\theta - m_t)^T \dd \theta.
\end{aligned}
\end{equation}
This is not, in general, a closed system for the
mean and covariance; this is because $\rho_t$ is not, in general,
determined by only the first and second moments. To close the system, we replace $\rho_t$ by $\rho_{\rhoa_t} = \N(m_t, C_t)$ in \cref{eq:mC-Momentum}. Then we obtain the following closed system for the evolution of
$(m_t,C_t)$:
\begin{equation}
\begin{aligned}
\label{eq:mC-Momentum2}
&\frac{\dd  m_t}{\dd t} = -\int 
M(\rho_{a_t})^{-1}\frac{\delta \mathcal{E}}{\delta \rho}\Bigr|_{\rho = \rho_{a_t}}  \theta \dd \theta, \\
&\frac{\dd  C_t}{\dd t} = -\int  M(\rho_{a_t})^{-1}\frac{\delta \mathcal{E}}{\delta \rho}\Bigr|_{\rho = \rho_{a_t}}   (\theta - m_t)(\theta - m_t)^T \dd \theta. 
\end{aligned}
\end{equation}
The moment closure approach has been used in Kalman filtering \cite{sarkka2007unscented} and Wasserstein gradient flows \cite{lambert2022variational} to obtain a reduced system of equations.
\subsection{Equivalence Between Two Approaches}
\label{sec-equivalence-gaussian-approx}
{The metric projection approach encapsulated in identity \eqref{eq:mra2} is rooted in Riemannian geometry, 
whilst the moment closure approach leading to \eqref{eq:mC-Momentum2} is based on probability. However
these two seemingly distinct approaches} lead to equivalent Gaussian approximations of the gradient flow; see the following \Cref{prop:1} and its proof in \Cref{proof-ssec:GGF}. In fact, such equivalence has been pointed out for the Wasserstein gradient flow in \cite{lambert2022variational}. Our result provides a general condition for such equivalence.
\begin{theorem}
\label{prop:1}
Suppose the following condition holds:

\begin{equation}
    \label{eq:mcp-cond}
    \M(\rho_\rhoa) T_{\rho_\rhoa}\PPG = {\rm span}\{\theta_i, \theta_i\theta_j, 1\leq i,j\leq N_{\theta}\}.
\end{equation}
Here $\M(\rho_\rhoa)$ is the metric tensor and $T_{\rho_\rhoa}\PPG$ is the tangent space of the Gaussian density space $\PPG$. Moreover, $\theta_i, \theta_i\theta_j$ are understood as functions of $\theta$. Then the mean and covariance evolution equations \eqref{eq:mC-Momentum2} are equivalent to \eqref{eq:G-GF}.
\end{theorem}

We have the following \Cref{prop:2}  showing that many metric tensors we have considered in \Cref{sec:GF and AI} indeed satisfy the condition~\cref{eq:mcp-cond}. Its proof can be found in \Cref{proof-ssec:GGF}.

\begin{theorem}
\label{prop:2}
Assumption~\eqref{eq:mcp-cond} holds for the Fisher-Rao metric tensor, the affine invariant Wasserstein metric tensor with a preconditioner $\Prec$ independent of $\theta$, and the affine invariant Stein metric tensor with a preconditioner $\Prec$ independent of $\theta$ and with a bilinear kernel $\kappa(\theta, \theta',\rho) = (\theta - m)^TA(\rho)(\theta' - m) + b(\rho)$; here we require $b\neq0$, and $A$ nonsingular.
\end{theorem}
This equivalence is useful as it allows us to use the moment closure approach to calculate the projected gradient flows. The moment closure approach is more convenient than the metric projection approach, since the former involves $M(\rho_{a_t})^{-1}$ while the latter involves $M(\rho_{a_t})$.  We know from the examples of the Wasserstein and Stein metric tensors that $M(\rho_{a_t})^{-1}$ is a differential operator which is simpler in calculations compared to its inverse. {On the other hand the projected gradient
flow is an intrinsic object and so it is desirable that the easy-to-implement
moment closure recovers it.}


%
\subsection{Equations of Gaussian Approximate Gradient Flows} 
\label{ssec:Gaussian-Fisher-Rao}
In this subsection, we use the moment closure approach to derive Gaussian approximate gradient flows for different metric tensors, discuss their connections, and draw relations to natural gradient methods in variational inference.
\subsubsection{Gaussian Approximate Fisher-Rao Gradient Flow}
Applying the moment closure approach to Fisher-Rao gradient flow \eqref{eq:mean-field-Fisher-Rao} leads to
the following equation:
\begin{equation*}
\begin{aligned}
\frac{\dd  m_t}{\dd  t} & = \int \theta \left(\rho_{a_t} \bigl(\log \rho_{\rm post} - \log \rho_{a_t}\bigr) - \rho_{a_t}\E_{\rho_{a_t}}[\log \rho_{\rm post} - \log \rho_{a_t}] \right) {\rm d}\theta \\
&= \Cov_{\rho_{\rhoa_t}}[\theta ,\, \log \rho_{\rm post} - \log \rho_{\rhoa_t} ]\\ & = C_t\E_{\rho_{\rhoa_t}}[\nabla_\theta (\log \rho_{\rm post} - \log \rho_{\rhoa_t}) ],\\
\frac{\dd  C_t}{\dd  t} &= \E_{\rho_{\rhoa_t}}[(\theta -m_t)(\theta -m_t)^T(\log \rho_{\rm post} - \log \rho_{\rhoa_t} - \E_{\rho_{\rhoa_t}}[\log \rho_{\rm post} - \log \rho_{\rhoa_t}])]\\
&= C_t \E_{\rho_{\rhoa_t}}[\Hess(\log \rho_{\rm post} - \log \rho_{\rhoa_t})] C_t,
\end{aligned}
\end{equation*}
where $\rho_{\rhoa_t} \sim \N(m_t, C_t)$, and we have used the Stein's lemma (see \Cref{lemma: stein equality}) in the last steps of the derivation. Furthermore noting that $\E_{\rho_{a_t}}[\nabla \log \rho_{a_t}] = 0$, we obtain
\begin{equation}
\begin{aligned}
\label{eq:Gaussian Fisher-Rao}
\frac{\dd  m_t}{\dd  t}  & =  C_t\E_{\rho_{\rhoa_t}}[\nabla_\theta \log \rho_{\rm post} ],\\
\frac{\dd  C_t}{\dd  t}
&= C_t + C_t \E_{\rho_{\rhoa_t}}[\Hess \log \rho_{\rm post}]C_t.
\end{aligned}
\end{equation}

\begin{newremark}
\Cref{eq:Gaussian Fisher-Rao} is the same as the continuous limit of \textit{natural gradient descent} \cite{amari1998natural,martens2020new,zhang2019fast} for Gaussian variational inference, where the Fisher information matrix is used to precondition the gradient descent dynamics for the optimization problem in \cref{eq:Gaussian-KL0}. In fact, the Fisher information matrix is equal to $\fM(\cdot)$ in \cref{eq:mra} when $M(\cdot)$ is chosen as the Fisher-Rao metric tensor \cite{malago2015information}.
\end{newremark}

{The Gaussian approximate Fisher-Rao gradient flow~\eqref{eq:Gaussian Fisher-Rao} is affine invariant.
Under any invertible affine transformation $\tilde{\theta} =\varphi(\theta) = A\theta + b$, the Gaussian density $\rho_{a_t}$ with $a_t = (m_t, C_t)$ in \cref{eq:Gaussian Fisher-Rao} will be transformed to another Gaussian density 
$\tilde{\rho}_a = \varphi\#\rho_a = \rho_{\tilde{a}_t}$ with $\tilde{a}_t = (\tilde{m}_t, \tilde{C}_t) = (Am, ACA^T)$. Let $\tilde{\rho}_{\rm post} = \varphi\#\rho_{\rm post}$.
Then, the dynamics of $(\tilde{m}_t, \tilde{C}_t)$ remains the same form as~\eqref{eq:Gaussian Fisher-Rao} satisfying
\begin{equation}
\begin{aligned}
\frac{\dd  \tilde{m}_t}{\dd  t}  & =  \tilde{C}_t\E_{\rho_{\tilde{\rhoa}_t}}[\nabla_{\tilde\theta} \log \tilde{\rho}_{\rm post} ],\\
\frac{\dd  \tilde{C}_t}{\dd  t}
&= \tilde{C}_t + \tilde{C}_t \E_{\rho_{\tilde{\rhoa}_t}}[\nabla_{\tilde\theta}\nabla_{\tilde\theta} \log \tilde{\rho}_{\rm post}]\tilde{C}_t.
\end{aligned}
\end{equation}
The above equation can be derived using \Cref{lemma:change_variable}.
}



\subsubsection{Gaussian Approximate Wasserstein Gradient Flow}
By applying the moment closure approach to \eqref{eq:mean-field-Wasserstein}, we get the mean and covariance evolution equations 
for the Gaussian approximate Wasserstein gradient flow with 
$\theta-$independent $\Prec$ as follows:
\begin{equation}
\label{eq:Wasserstein-mC}
\begin{split}
      \frac{\mathrm{d} m_t}{\mathrm{d}t} &= \int\Bigl[\rho_{\rhoa_t} \Prec(\rho_{\rhoa_t}) \nabla_{\theta}(\log \rho_{\rm post} - \log \rho_{\rhoa_t}) \Bigr] \mathrm{d}\theta =  \Prec(\rho_{\rhoa_t}) \E_{\rho_{\rhoa_t}}[\nabla_\theta \log \rho_{\rm post} ], \\
    \frac{\mathrm{d} C_t}{\mathrm{d}t} &= \int\nabla_{\theta} \cdot \Bigl[\rho_{\rhoa_t} \Prec(\rho_{\rhoa_t}) \nabla_{\theta} (\log \rho_{\rhoa_t}  - \log \rho_{\rm post})\Bigr] (\theta - m_t)(\theta - m_t)^T\mathrm{d}\theta \\
    &=  2\Prec(\rho_{\rhoa_t}) +  \Prec(\rho_{\rhoa_t}) \E_{\rho_{\rhoa_t}}[\Hess\log\rho_{\rm post}]C_t +  C_t\E_{\rho_{\rhoa_t}}[\Hess\log\rho_{\rm post}] \Prec(\rho_{\rhoa_t}), 
\end{split}
\end{equation}
where $\rho_{\rhoa_t} \sim \N(m_t, C_t)$, and we have used integration by parts and Stein's lemma (see \Cref{lemma: stein equality}), and the fact that $\rho \nabla_{\theta} \log \rho = \nabla_{\theta}\rho$ and $\E_{\rho_{a_t}}[\nabla_{\theta} \log \rho_{a_t}] = 0$ in the above derivation.

When we set $\Prec(\rho) \equiv \I$, the 
evolution equation becomes
\begin{equation}
\begin{aligned}   
\label{eq:Gaussian-Wasserstein}
&\frac{\mathrm{d}m_t}{\mathrm{d}t} &&= \E_{\rho_{\rhoa_t}} \bigl[ \nabla_{\theta}  \log \rho_{\rm post}  \bigr],\\
&\frac{\mathrm{d}C_t}{\mathrm{d}t} &&= 2\I + \E_{\rho_{\rhoa_t}}\bigl[\Hess\log \rho_{\rm post} \bigr]C_t + C_t\E_{\rho_{\rhoa_t}}\bigl[\Hess\log \rho_{\rm post} \bigr].
\end{aligned}
\end{equation}
This corresponds to the gradient flow under the constrained Wasserstein metric in the Gaussian density space~\cite{sarkka2007unscented,julier1995new,lambert2022variational}.
\begin{newremark}
    When $P$ is the identity operator, the metric tensor $\fM(\rhoa)$ define in \cref{eq:mra} has an explicit formula \cite{chen2018natural, takatsu2011wasserstein, Luigi2018, bhatia2019bures, li2023}, and the corresponding Gaussian density space is called the Bures–Wasserstein space \cite{villani2021topics}.
\end{newremark}

If we set $\Prec(\theta, \rho) = C(\rho)$, a choice that leads to affine invariant Wasserstein gradient flows in the density space, then the
resulting evolution equations for the corresponding mean and covariance are
\begin{equation}
\label{eq:Gaussian-AI-Wasserstein}
\begin{aligned}    
&\frac{\dd m_t}{\dd t} &&= C_t\E_{\rho_{\rhoa_t}} \bigl[ \nabla_{\theta}  \log \rho_{\rm post} \bigr],\\
&\frac{\dd C_t}{\dd t} &&= 2C_t + 2C_t \E_{\rho_{\rhoa_t}}\bigl[\Hess \log \rho_{\rm post}\bigr]C_t.
\end{aligned}
\end{equation}
\Cref{eq:Gaussian-AI-Wasserstein} is similar to the Gaussian approximate Fisher-Rao gradient flow \eqref{eq:Gaussian Fisher-Rao}, but with scaling factor $2$ in the covariance evolution. {This equation is also affine invariant, akin to the Gaussian approximate Fisher-Rao gradient flow~\eqref{eq:Gaussian Fisher-Rao}.}


\subsubsection{Gaussian Approximate Stein Gradient Flow}
We apply the moment closure approach to \eqref{eq:Stein-GF}. The mean and covariance evolution equations of the preconditioned Stein gradient flow \eqref{eq:AI-Stein} with a $\theta-$independent $\Prec$
are
\begin{equation}
\begin{aligned}   
\label{eq:Gaussian-Stein-CTL0}
&\frac{\mathrm{d}m_t}{\mathrm{d}t} = -\int\Bigl(\rho_{\rhoa_t}(\theta)\int \kappa(\theta,\theta',\rho_{\rhoa_t})\rho_{\rhoa_t}(\theta')\Prec_t \nabla_{\theta'} \bigl(\log\rho_{\rhoa_t}(\theta') - \log \rho_{\rm post}(\theta') \bigr)\mathrm{d}\theta' \Bigr) \mathrm{d}\theta, \\
&\frac{\mathrm{d}C_t}{\mathrm{d}t} =\\
&-\int \Bigl(\rho_{\rhoa_t}(\theta)\int \kappa(\theta,\theta',\rho_{\rhoa_t})\rho_{\rhoa_t}(\theta')\Prec_t \nabla_{\theta'} \bigl(\log\rho_{\rhoa_t}(\theta') - \log \rho_{\rm post}(\theta') \bigr)\mathrm{d}\theta' \Bigr) (\theta - m_t)^T \\
&+ (\theta - m_t)\Bigl(\rho_{\rhoa_t}(\theta)\int \kappa(\theta,\theta',\rho_{\rhoa_t})\rho_{\rhoa_t}(\theta')\Prec_t \nabla_{\theta'} \bigl(\log\rho_{\rhoa_t}(\theta') - \log \rho_{\rm post}(\theta') \bigr)\mathrm{d}\theta' \Bigr)^T \mathrm{d}\theta. \\
\end{aligned}
\end{equation}
where $\rho_{\rhoa_t} \sim \N(m_t, C_t)$, $P_t:=P(\rho_{\rhoa_t})$, and we have used integration by parts in the above derivation. 

Imposing the form of the bilinear
kernel mentioned in \Cref{prop:2} and using the Stein's lemma (See \Cref{lemma: stein equality}) and the fact that $\E_{\rho_{a_t}}[\nabla \log \rho_{a_t}] = 0$, we obtain
\begin{equation}
\begin{aligned}   
\label{eq:Gaussian-Stein-CTL}
\frac{\mathrm{d}m_t}{\mathrm{d}t} 
=&b_tP_t \E_{\rho_{\rhoa_t}}[\nabla_{\theta}\log\rho_{\rm post}]
,\\
\frac{\mathrm{d}C_t}{\mathrm{d}t} 
=& P_tA_tC_t + C_tA_tP_t\\
&\quad\quad\quad+ P_t\E_{\rho_{\rhoa_t}}[\Hess\log\rho_{\rm post}]C_tA_tC_t +
C_tA_tC_t\E_{\rho_{\rhoa_t}}[\Hess\log\rho_{\rm post}]P_t.
\end{aligned}
\end{equation}
Here, we used the notations that $A_t= A(\rho_{a_t})$ and $b_t = b(\rho_{a_t})$.

Different choices of the preconditioner $\Prec$ and the bilinear kernel $\kappa$ in \Cref{prop:2}
allow us to construct various Gaussian approximate gradient flows. Setting $\Prec(\theta, \rho) = C(\rho)$
and choosing the bilinear kernel with $A(\rho) = \frac{1}{2}C(\rho)^{-1}$ and $b(\rho) = 1$
lead to the Gaussian approximate Fisher-Rao gradient flow in \cref{eq:Gaussian Fisher-Rao}.
Choosing the preconditioner $\Prec(\rho) = \I$ and the bilinear kernel 
with $A(\rho)=C(\rho)^{-1}$ and $b(\rho)=1$ leads to the Gaussian approximate Wasserstein gradient flow in \cref{eq:Gaussian-Wasserstein}.
Moreover, setting the preconditioner $\Prec_t = \I$ and the bilinear kernel 
with $A(\rho)=I$ and $b(\rho)=1$
recovers the Gaussian sampling approach introduced in \cite{galy2021flexible}:
\begin{equation}
\label{eq:galy2021flexible}
\begin{aligned}
\frac{\mathrm{d}m_t}{\mathrm{d}t}
=&\E_{\rho_{\rhoa_t}}[\nabla_{\theta}\log\rho_{\rm post}]
,\\
\frac{\mathrm{d}C_t}{\mathrm{d}t}
=& 2C_t + \E_{\rho_{\rhoa_t}}[\Hess\log\rho_{\rm post}]C_t^2 +
C_t^2\E_{\rho_{\rhoa_t}}[\Hess\log\rho_{\rm post}].
\end{aligned}
\end{equation}
{We note that this dynamics is not affine invariant.}

{The above discussion implies that the Gaussian approximate Stein gradient flow is quite general:  with various preconditioners and bilinear kernels, it can recover many Gaussian dynamics used in sampling.}
{
\begin{remark}
In many sampling problems, we may evaluate $\nabla_{\theta}\log\rho_{\rm post}$ but not  the Hessian matrix $\nabla_{\theta}\nabla_{\theta}\log\rho_{\rm post}$.
Nevertheless we note that Stein's identities, as presented in \Cref{lemma: stein equality}, allow one to eliminate one derivative in the aforementioned dynamics, through
\begin{equation}
    \E_{\rho_{a_t}}[ \nabla_{\theta} \nabla_{\theta}\log \rho_{\rm post}] = \E_{\rho_{a_t}}[\nabla_{\theta} \log \rho_{\rm post} (\theta - m_t)^T] C_t^{-1}.
\end{equation}   
\end{remark}
}

\subsection{Convergence Analysis}
\label{ssec:GGF-converge}
In the last subsection, we obtained various Gaussian approximate gradient flows, {all of which which can be used to obtain} a Gaussian approximation of the target distribution.  Some of the Gaussian dynamics are affine invariant and some are not. The goal of the subsection is to provide convergence analysis of these dynamics and understand the effect of affine invariance for different target distributions.
We will consider three classes of target distribution: the Gaussian posterior case, 
logconcave posterior case, and general posterior case.

\subsubsection{Gaussian Posterior Case} 
\label{ssec:gaussian}

We start with the Gaussian posterior case. For such case, we are able to compute the explicit formula of the dynamics and then establish convergence rates. We consider the Gaussian approxiamte Fisher-Rao and Wasserstein gradient flows here. Additionally, for the purpose of comparison, we also consider the standard gradient flow for Gaussian variational inference, employing the Euclidean inner-product metric in $\R^{N_a}$ (i.e., setting $\fM(\rhoa_t)=I$ in \eqref{eq:G-GF}). This choice leads to the following dynamical system
\begin{equation}
\label{eq:g-gf}
\begin{aligned}
\frac{\dd m_t}{\dd t} &= \E_{\rho_{\rhoa_t}}\bigl[\nabla_\theta  \log \rho_{\rm post}   \bigr],  \\
\frac{\dd C_t}{\dd t} &= \frac{1}{2}C_t^{-1} + \frac{1}{2}\E_{\rho_{\rhoa_t}}\bigl[\Hess\log \rho_{\rm post} \bigr].
\end{aligned}
\end{equation}
In this article, we call it the vanilla Gaussian approximate gradient flow.
In \Cref{p:rate}, we provide convergence rates for the three different dynamics.
\begin{theorem}\label{p:rate}
Assume the posterior distribution is Gaussian
$\rho_{\rm post}(\theta)\sim \N(m_{\star}, C_{\star})$.
Denote the largest eigenvalue of $C_{\star}$ by $\la$. 
For gradient flows with initialization $C_0=\lambda_0\I$, the following hold:
\begin{enumerate}
\item for the Gaussian approximate Fisher-Rao gradient flow \eqref{eq:Gaussian Fisher-Rao}: 
\[\|m_t-m_{\star}\|_2=\bigO(e^{-t}), \|C_t-C_{\star}\|_2=\bigO(e^{-t});\]
\item for the Gaussian approximate Wasserstein gradient flow \eqref{eq:Gaussian-Wasserstein}: 
\[\|m_t-m_{\star}\|_2=\bigO(e^{-t/\la}), \|C_t-C_{\star}\|_2=\bigO(e^{-2t/\la});\]
\item for the vanilla Gaussian approximate gradient flow~\eqref{eq:g-gf}: \[\|m_t-m_{\star}\|_2=\bigO(e^{-t/\la}), \|C_t-C_{\star}\|_2=\bigO(e^{-t/(2\la^2)}).\]
\end{enumerate}
where the implicit constants depend on $m_{\star}$, $C_{\star}$ and $\lambda_0$.
\end{theorem}
The proof can be found in~\Cref{proof:p:rate}. For the Gaussian approximate gradient flow~\eqref{eq:g-gf} and the Gaussian approximate Wasserstein gradient flow~\eqref{eq:Gaussian-Wasserstein}, if $\la$ is large, their convergence rate is much slower than the Gaussian approximate Fisher-Rao gradient flow. The affine invariance property of the Gaussian approximate Fisher-Rao gradient flow {allows it to} achieve a uniform exponential rate of convergence $\bigO(e^{-t})$ for Gaussian posteriors.

The uniform convergence rate $\bigO(e^{-t})$ has also been studied for the mean and covariance dynamics of the Kalman-Wasserstein gradient flow~\eqref{eq:Gaussian-AI-Wasserstein} in~\cite[Lemma 3.2]{garbuno2020interacting}\cite{carrillo2021wasserstein,burger2023covariance}, which is affine invariant. As a consequence, the affine invariance property accelerates the convergence of the flow when the target distribution is Gaussian.

Can affine invariance also accelerate the dynamics beyond Gaussian posteriors? To investigate this question, in the following we study the convergence property of the Gaussian approxiamate Fisher-Rao gradient flow, which is affine invariant, for logconcave posteriors and general posteriors.

\subsubsection{Logconcave Posterior Case} 
While the convergence rate of the Gaussian approximate Fisher-Rao gradient flow remains at $\bigO(e^{-t})$ and independent of $\rho_{\rm post}$ for the Gaussian posterior case (as shown in \Cref{p:rate}), the situation changes in the logconcave posterior case. In this scenario, the convergence of the Gaussian approximate Fisher-Rao gradient flow is no longer universally independent of $\rho_{\rm post}$. We can establish the following theorem regarding {lower and upper bounds on} the local convergence rate of the Gaussian approximate Fisher-Rao gradient flow \eqref{eq:Gaussian Fisher-Rao}.

\begin{theorem}\label{p:logconcave-local}
Assume the posterior distribution $\rho_{\rm post}$ is $\alpha$-strongly logconcave and $\beta$-smooth such that $\log \rho_{\rm post} \in C^2(\R^{N_{\theta}})$, $-\Hess \log \rho_{\rm post}(\theta) \succeq \alpha I,$ and $-\Hess \log \rho_{\rm post} \preceq \beta \I$.
Denote the unique minimizer of the Gaussian variational inference problem~\cref{eq:Gaussian-KL0} with the KL divergence as the energy functional by $\rho_{\rhoa_\star} :=  \N(m_{\star}, C_{\star})$.
For $N_\theta = 1$, let $\lambda_{\star,\max} < 0$ denote the largest eigenvalue of the linearized Jacobian matrix of the Gaussian approximate Fisher-Rao gradient flow \eqref{eq:Gaussian Fisher-Rao} around $m_\star$ and $C_\star$; this number
determines the local convergence rate of the Gaussian approximate Fisher-Rao 
gradient flow. 
Then we have
\begin{align}
\label{eq:lambda_max}
    - \lambda_{\star,\max} \geq \frac{1}{(7+\frac{4}{\sqrt{\pi}})\bigl(1 + \log(\frac{\beta}{\alpha})\bigr)}.
\end{align}
Moreover, the bound is sharp: it is possible to construct a sequence of triplets $\rho_{{\rm post},n}$, $\alpha_n$ and $\beta_n$, where $\lim_{n \to \infty} \frac{\beta_n}{\alpha_n} = \infty$, such that, if we let $\lambda_{\star,\max,n}$ denote the corresponding largest eigenvalues of the linearized Jacobian matrix for the $n$-th triple, then, it holds that
\[-\lambda_{\star,\max,n} = \bigO\left(1/\log \frac{\beta_n}{\alpha_n}\right). \]
\end{theorem}
The proof can be found in~\Cref{proof:proof-counter-example-Gaussian-Fisher-Rao}. For the counterexample, we construct $\rho_{{\rm post},n}$ such that $-\nabla_{\theta}\nabla_{\theta}\rho_{{\rm post},n}$ is a bump function between $\alpha_n$ and $\beta_n$ with the width of the bump gradually approaching 0. Our analysis reveals that this construction can drive $-\lambda_{\star,\max}$ as small as possible.

\Cref{p:logconcave-local} implies that the local convergence rate of the Gaussian approximate Fisher-Rao gradient flow is $\log(\frac{\beta}{\alpha})$ for logconcave posteriors. Therefore, such affine invariant flow is still advantageous in sampling highly anisotropic logconcave posteriors, since the dependence of the local convergence rate on the anisotropic ratio or condition number $\beta/\alpha$ is only logarithmic. 

Recently, the work \cite{lambert2022variational} proved the global convergence of the Gaussian approximate Wasserstein gradient flow (which is not affine invariant) when the posterior is logconcave. We leave it as a future work to prove the global convergence rate for the Gaussian approximate Fisher-Rao gradient flows.

\subsubsection{General Posterior Case} 
In general, we can construct posteriors such that the convergence of the  Gaussian approximate Fisher Rao gradient flow to a stationary point can be arbitrarily slow. 

\begin{theorem}
\label{proposition:counter-example-slow}
For any $K > 0$ there exist a target $\rho_{\rm post}$ such that, for the three Gaussian approximate gradient flows~\cref{eq:Gaussian Fisher-Rao}, 
\cref{eq:Gaussian-Wasserstein} and \cref{eq:g-gf}, the convergence rate to 
their stationary points can be as slow as $\bigTheta(t^{-\frac{1}{2K}})$. 
\end{theorem}
The proof can be found in \Cref{proof:counter-example-slow}. This demonstrates that the use of affine invariant properties cannot achieve a universal acceleration for all posterior distributions.

\subsection{Illustrative Numerical Examples}
\label{sec-gaussian-numerics}
In this subsection, we perform numerical experiments comparing different Gaussian approximate gradient flows. The test target distributions from \Cref{sec-affine-invariance-experiments} are used.
We consider the three dynamical models for mean and covariance given in \cref{eq:Gaussian Fisher-Rao}, \cref{eq:Gaussian-Wasserstein} and \cref{eq:g-gf}.
Letting $\rho_{a_n} = \N(m_n, C_n)$ denote the approximated Gaussian density at time $n\Delta t$, the Gaussian approximate Fisher-Rao gradient flow \eqref{eq:Gaussian Fisher-Rao} is discretized as
\begin{equation}
\begin{aligned}
m_{n+1} & =  m_{n} + \Delta t C_n\E_{\rho_{\rhoa_n}}[\nabla_\theta \log \rho_{\rm post} ],\\
C_{n+1}^{-1}
&= C_n^{-1} - \Delta t \bigl( C_n^{-1} + \E_{\rho_{\rhoa_n}}[\Hess \log \rho_{\rm post}] \bigl) .
\end{aligned}
\end{equation}
Here, we use the forward Euler scheme to discretize the dynamics of $C^{-1}$ rather than $C$, for which we observe better numerical stability.
Following~\cite{lambert2022variational}, the Gaussian approximate Wasserstein gradient flow \eqref{eq:Gaussian-Wasserstein} is discretized as
\begin{equation}
\begin{aligned}   
m_{n+1}& = m_{n} + \Delta t \E_{\rho_{\rhoa_n}} \bigl[ \nabla_{\theta}  \log \rho_{\rm post}  \bigr],\\
C_{n+1}& = 
\Bigl( I + \Delta t \bigl(\E_{\rho_{\rhoa_n}}[\Hess \log \rho_{\rm post}] + C_{n}^{-1}\bigr) \Bigr) C_{n}  
 \Bigl( I + \Delta t \bigl(\E_{\rho_{\rhoa_n}}[\Hess \log \rho_{\rm post}] + C_{n}^{-1}\bigr) \Bigr).
\end{aligned}
\end{equation}
The Gaussian approximate gradient flow \eqref{eq:g-gf} is discretized by the forward Euler scheme directly.
The expectations in the evolution equations are calculated by the unscented transform~\cite{julier1997new,huang2022iterated,huang2022efficient} with $J =2N_\theta + 1 = 5$ quadrature points. In this implementation 
the Gaussian approximation has considerable speedup in comparison with the 
previously mentioned particle-based sampling approaches, where $J=100$. We use the same summary statistics as in \Cref{sec-affine-invariance-experiments} to compare the obtained Gaussian with the true target distribution.

For the Gaussian posterior, the convergence of different gradient flows, according to the three summary statistics, are presented in~\Cref{fig:Gaussian_gd_converge}. The imposition of the affine invariance property makes the convergence rate independent of the anisotropy $\lambda$ and accelerates the sampling for badly scaled Gaussian ($\lambda = 0.01$).
The convergence rates of Gaussian approximate gradient flows match well with the
predictions of~\Cref{p:rate}. 

\begin{figure}[ht]
\centering
    \includegraphics[width=0.8\textwidth]{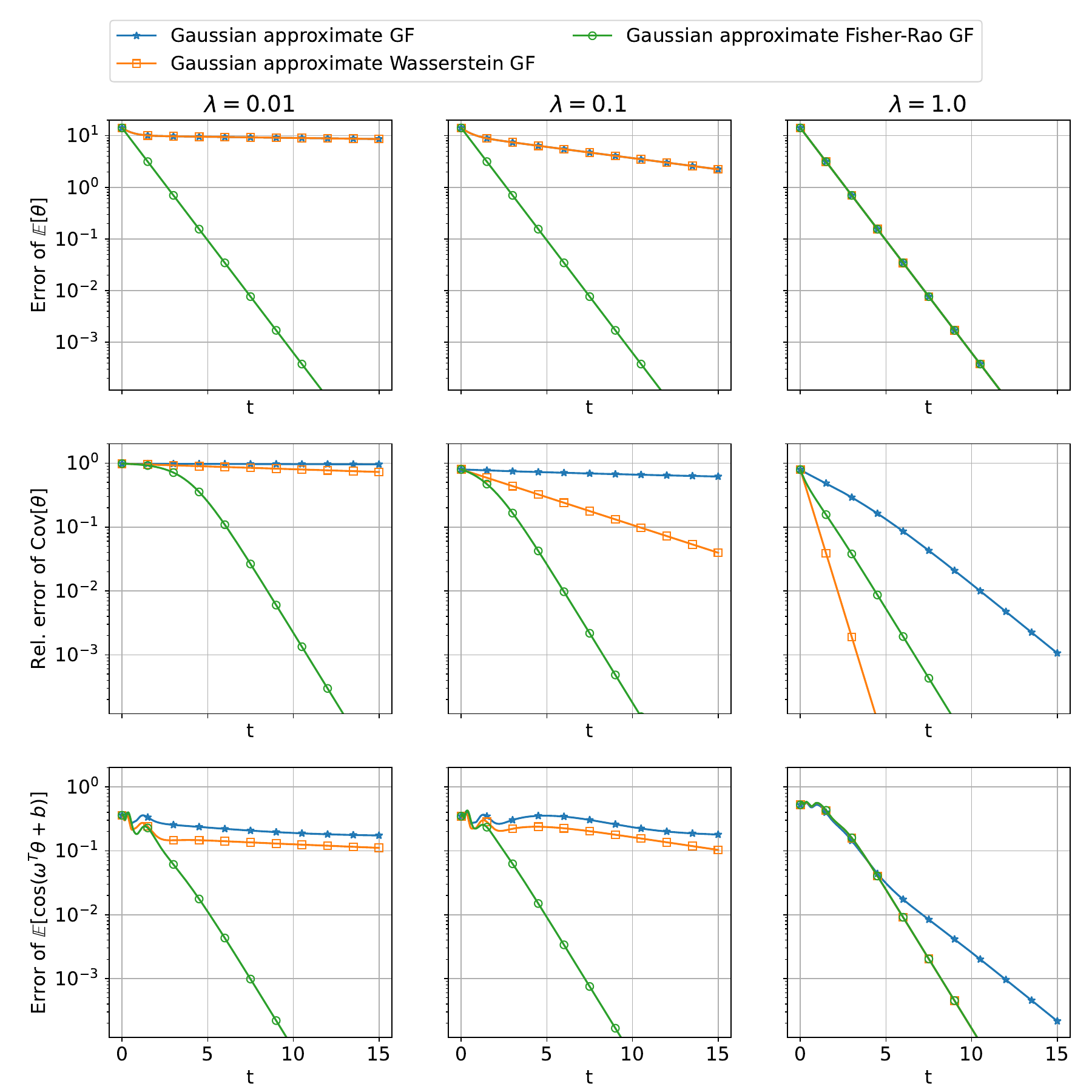}
    \caption{Gaussian posterior case: convergence of different dynamics in terms of $L^2$ error of $\E[\theta_t]$, the relative Frobenius norm error of the covariance $ \frac{\lVert \Cov[\theta_t] - \Cov[\theta_{{\rm true}}]\rVert_F}{\lVert \Cov[\theta_{\rm true}]\rVert_F}$, and the error of $\E[\cos(\omega^T \theta_t + b)]$.}
    \label{fig:Gaussian_gd_converge}
\end{figure}

For the logconcave posterior, 
the convergence behaviors of different gradient flows, according to the three summary statistics, are presented in \Cref{fig:Logconcave_gd_converge}. The imposition of the affine invariance property makes the convergence rate independent of the anisotropy $\lambda$ and accelerates the sampling in the highly anisotropic case ($\lambda = 0.01$). 
We observe that the convergence rate of the Gaussian approximate Fisher-Rao gradient flow does not deteriorate with increased anisotropy constant $\lambda$; this 
indicates that the local convergence rate in~\Cref{p:logconcave-local}, for this gradient
flow, may be extended to arbitrary dimensionalities, suggesting the possibility of achieving a global convergence rate as well.

\begin{figure}[ht]
\centering
    \includegraphics[width=0.8\textwidth]{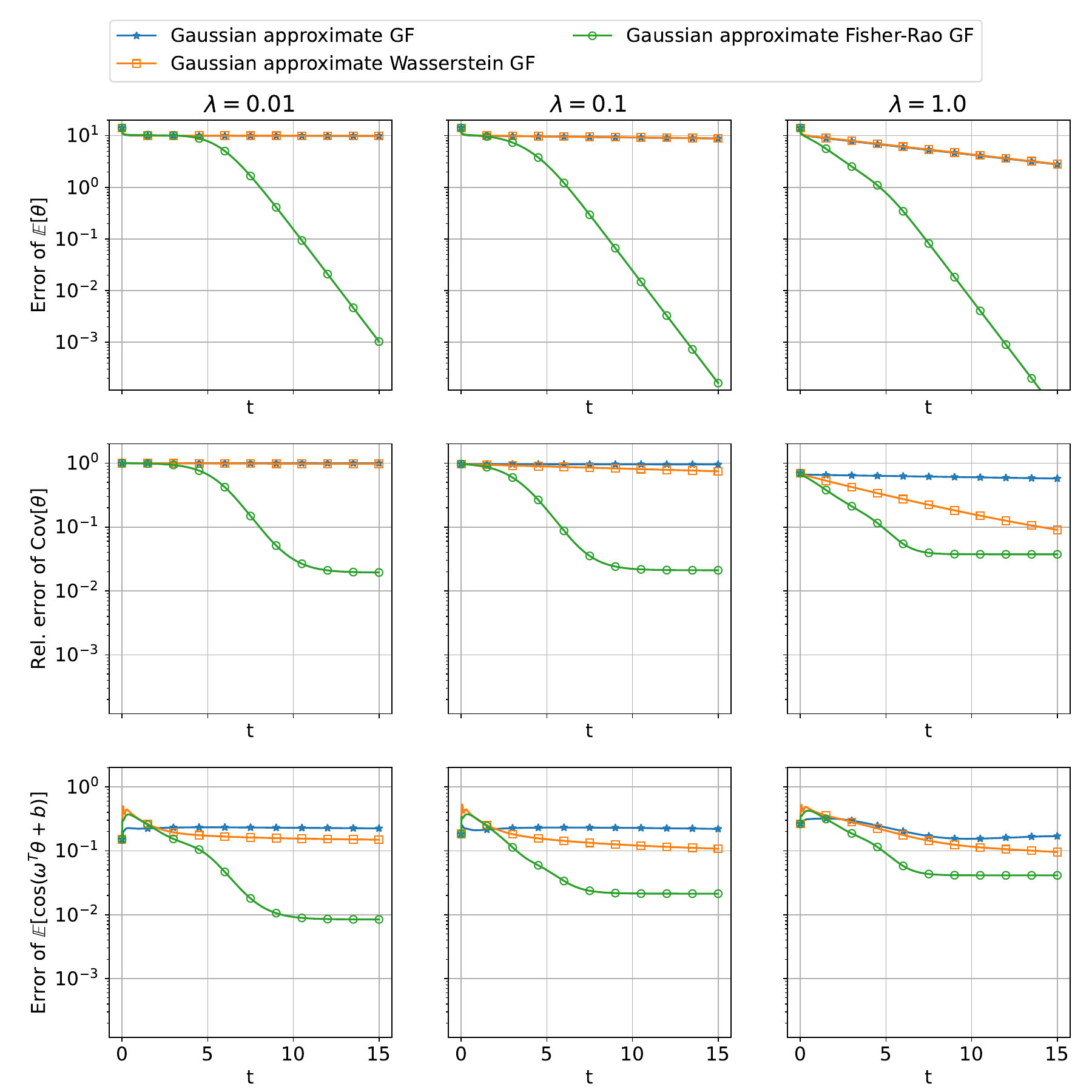}
    \caption{ Logconcave posterior case: convergence of different dynamics in terms of $L^2$ error of $\E[\theta_t]$, the relative Frobenius norm error of the covariance $ \frac{\lVert \Cov[\theta_t] - \Cov[\theta_{{\rm true}}]\rVert_F}{\lVert \Cov[\theta_{\rm true}]\rVert_F}$, and the error of $\E[\cos(\omega^T \theta_t + b)]$}
    \label{fig:Logconcave_gd_converge}
\end{figure}

For the general posterior, the estimated posterior densities (3 standard deviations) obtained by different Gaussian approximate gradient flows are presented in~\Cref{fig:Rosenbrock_gd_density}, and their convergence behaviors according to the three summary statistics are depicted in~\Cref{fig:Rosenbrock_gd_converge}.
Here, the use of Gaussian approximations cannot represent the posterior 
distribution well because the posterior is far from Gaussians. 

\begin{figure}[ht]
\centering
    \includegraphics[width=0.8\textwidth]{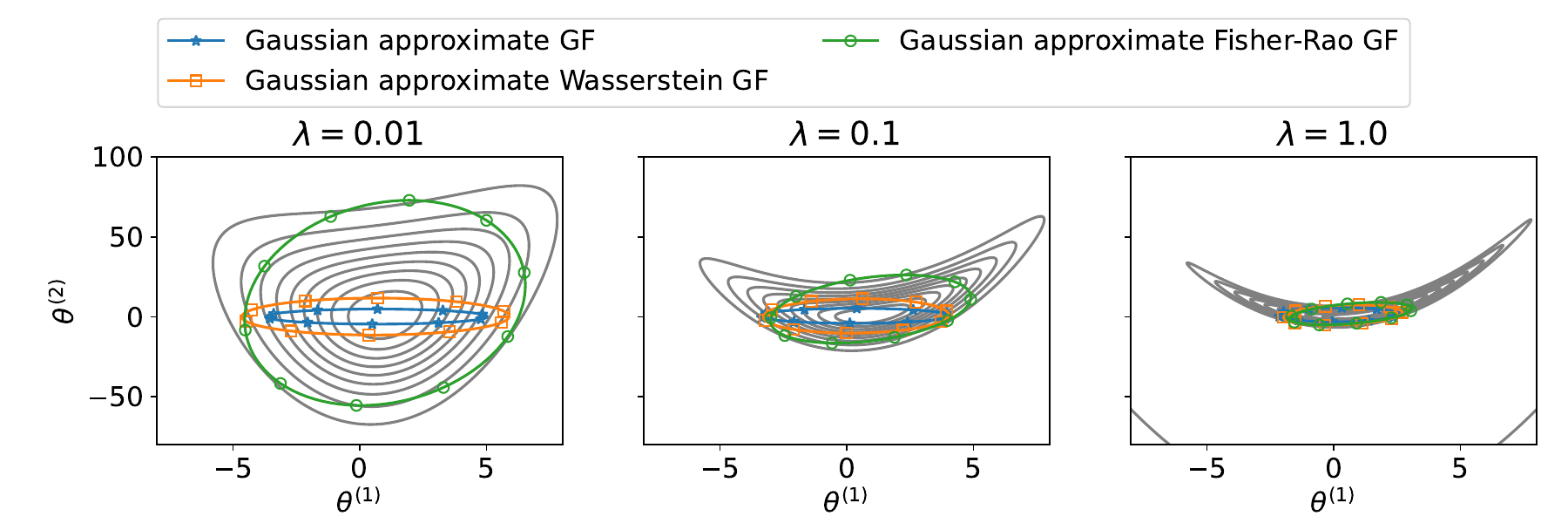}
    \caption{General posterior case: density functions (3 standard deviations) obtained by different dynamics at $t=15$. Grey lines represent the contour of the true posterior.
    }
    \label{fig:Rosenbrock_gd_density}
\end{figure}

\begin{figure}[ht]
\centering
    \includegraphics[width=0.8\textwidth]{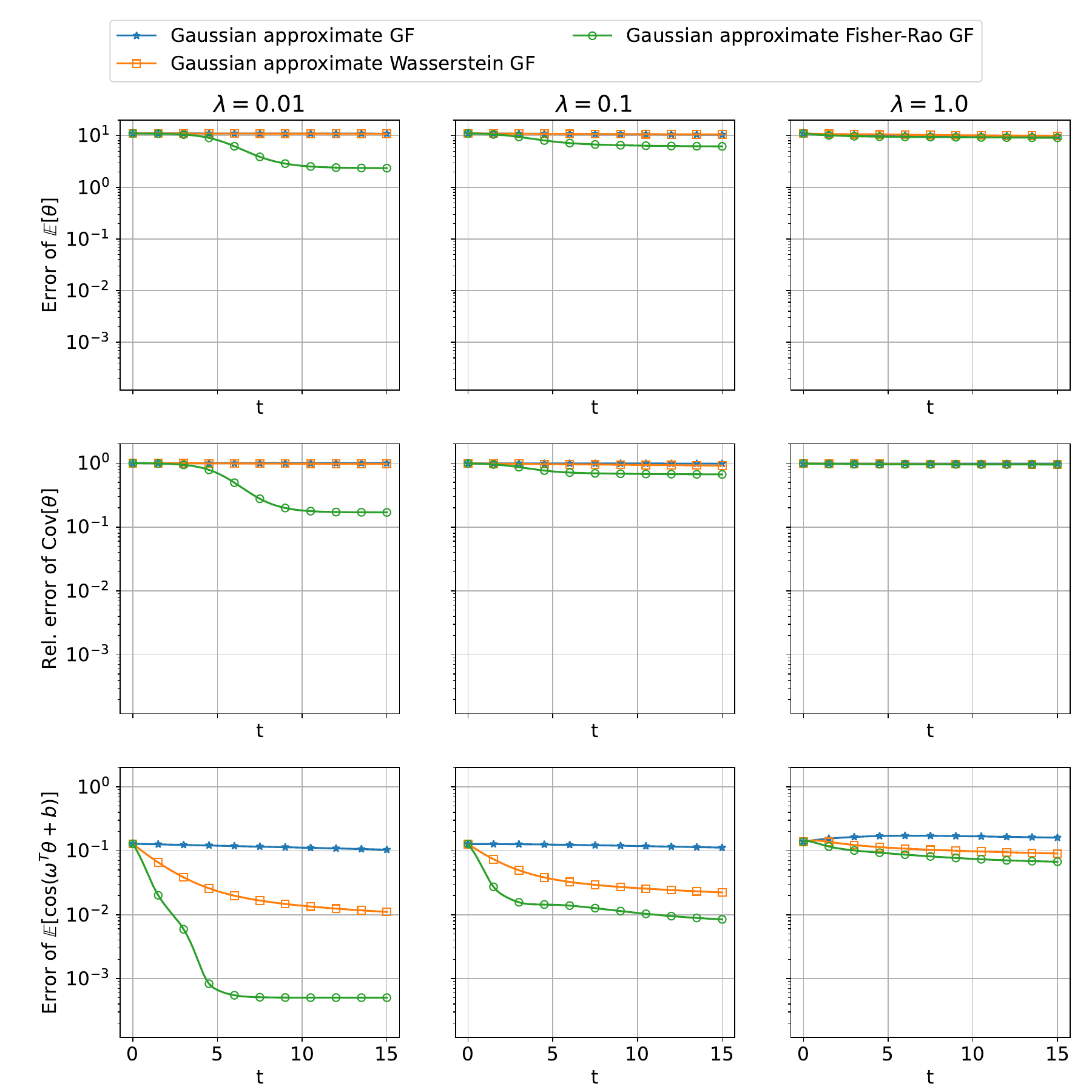}
    \caption{General posterior case: convergence of different dynamics in terms of $L_2$ error of $\E[\theta_t]$, the relative Frobenius norm error of the covariance $ \frac{\lVert \Cov[\theta_t] - \Cov[\theta_{{\rm true}}]\rVert_F}{\lVert \Cov[\theta_{\rm true}]\rVert_F}$, and the error of $\E[\cos(\omega^T \theta_t + b)]$.}
    \label{fig:Rosenbrock_gd_converge}
\end{figure}

In summary, along with the numerical results obtained by particle methods in \Cref{sec-affine-invariance-experiments}, we observe that affine-invariant modifications improve convergence for both Gaussian and particle approximations. Generally, Gaussian approximation methods are more efficient for sampling Gaussian and logconcave posterior distributions, as well as general posteriors with small $\lambda$, in terms of the number of function evaluations. However, this efficiency comes at the expense of computing the Hessian matrix. For general posterior distributions with $\lambda=1$ that deviate significantly from Gaussian distributions, Gaussian approximation may not accurately represent the distribution (See \Cref{fig:Rosenbrock_gd_density}). In such cases, particle methods, due to their nonparametric nature, might offer improved representations (See \Cref{fig:Rosenbrock_gd_particle_density}), but achieving convergence remains challenging and could require a large number of particles, with or without affine-invariant modifications.

\section{Application: Darcy Flow}
\label{sec:darcy}
In this section, we assess the effectiveness of the aforementioned approaches in addressing Bayesian inverse problems, specifically focusing on a 1D Darcy flow problem. The Darcy equation describes the pressure field $p(x)$ in a porous medium with a parameterized, positive permeability field $a(x,\theta)$, through the following PDE
\begin{equation}
\label{eq:Darcy}
    \begin{aligned}
       -\partial_x \bigl(a(x, \theta) \partial_x  p(x)\bigr) &= f(x), \quad &&x\in D,\\
    p(x) &= 0, \quad &&x\in \partial D.         
    \end{aligned}
\end{equation}
Here the computational domain is $D = [0,1]$. Homogeneous Dirichlet boundary conditions are applied on  $\partial D$, and $f$ is the source of the fluid (see the left of \Cref{fig:darcy-1d-ref}):
\begin{align}
    f(x) = \begin{cases}
               2000 & 0 \leq x \leq \frac{1}{3}\\
               1000 & \frac{1}{3} < x \leq \frac{2}{3}\\
               0 & \frac{2}{3} < x \leq 1\\
            \end{cases}. 
\end{align} 
The Bayesian inverse problem considered here is to determine parameter $\theta$ of
the field $a(\cdot;\theta)$ (see the right of \Cref{fig:darcy-1d-ref}) from observations $y_{\rm ref}$, which consist of pointwise measurements of the pressure $p(x)$ at $7$ equidistant points in the domain~(see the middle of \Cref{fig:darcy-1d-ref}), corrupted with observation noises $\eta \sim \N(0, \I)$. 
The field $a(x,\theta)$ is parametrized as
\begin{equation}
\label{eq:KL-1d}
    \log a(x,\theta) = \sum_{l=1} \theta_{(l)}\sqrt{\lambda_l} \psi_l(x),
\end{equation}
where the basis functions and eigenvalues are defined by
\begin{equation}
    \psi_l(x) = \sqrt{2}\cos(\pi l x),
                 \qquad \lambda_l = (\pi^2 |l|^2 + \tau^2)^{-2},
\end{equation}
and $\theta_{(l)} \sim \N(0,10^2)$ i.i.d. 
We note that these considerations amount to assuming
that $\log a(x, \theta)$ is a mean zero Gaussian random field
with covariance
\begin{equation}
    \mathsf{C} = (-\Delta + \tau^2 I )^{-2}.
\end{equation}
Here $-\Delta$ is the Laplacian operator on $D$ subject to homogeneous Neumann boundary conditions, acting on the space of spatial mean zero functions~\cite{dashti2013bayesian}. In fact, \eqref{eq:KL-1d} is the Karhunen-Lo\`eve expansion of the Gaussian random field.
Here, for numerical studies, we truncate the sum \eqref{eq:KL-1d} to the first $N_\theta = 16$ terms
and we take $\tau=3$. 

The forward problem, denoted by 
$y = \mathcal{G}(\theta)$, is defined as the map from $\theta$ to $p(x)$ at the $7$ equidistant points. Numerically, the map is implemented through solving the Darcy equation~\eqref{eq:Darcy} by a finite difference method on a size-$128$ grid to obtain $p(x)$ at the $7$ observation locations. A typical computational goal in Bayesian inverse problems is to sample the posterior distribution
\begin{equation}
\rho_{\rm post} \propto \exp(-\Phi_R(\theta)), 
\qquad 
\Phi_R(\theta) = \frac{1}{2}\bigl(y_{\rm ref} - \mathcal{G}(\theta)\bigr)^T\bigl(y_{\rm ref} - \mathcal{G}(\theta)\bigr) + \frac{1}{200}\theta^T\theta,
\end{equation}
where we specify a Gaussian prior $\N(0, 10^2\I)$ on $\theta$. Note that here $y_{\rm ref}, \G(\theta) \in \mathbb{R}^7$ and $\theta \in \mathbb{R}^{16}.$ The evaluation of $\Phi_R(\theta)$ requires solving the Darcy equation~\eqref{eq:Darcy}, and the corresponding gradients and Hessians are computed automatically through the Julia ForwardDiff library~\cite{RevelsLubinPapamarkou2016}. 
\begin{figure}[ht]
\centering
    \includegraphics[width=0.9\textwidth]{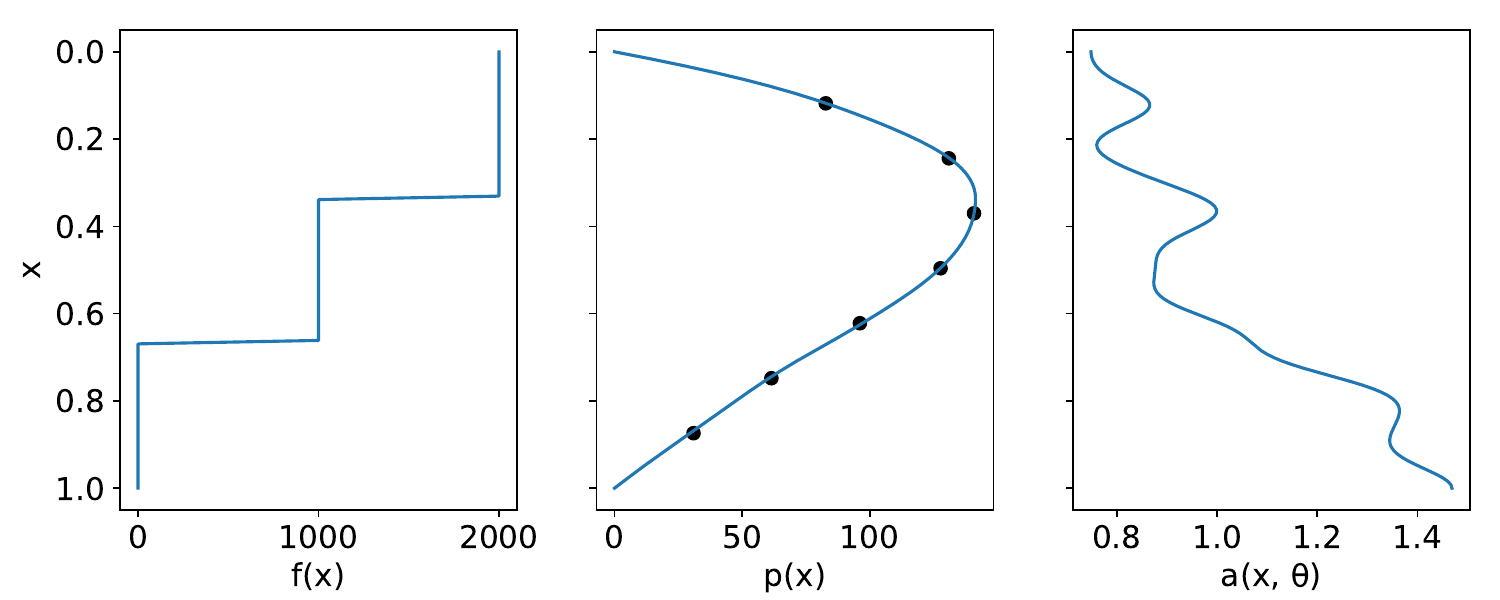}
    \caption{From left to right: source term $f(x)$, reference pressure field $p(x)$ with $7$ equidistant pointwise measurements, and reference permeability field $a(x,\theta)$ of the Darcy flow problem.}
    \label{fig:darcy-1d-ref}
\end{figure}

The reference, ground truth posterior distribution is obtained by using the preconditioned Crank–Nicolson algorithm (pCN)~\cite{cotter2013mcmc} (which is a specific function-space MCMC algorithm). We run pCN to obtain $2\times 10^8$ samples~(with $5\times10^7$ samples for the burn-in period) using step size $0.04$. The estimated covariance matrix is plotted in \Cref{fig:darcy-1d-cov}; this figure demonstrates that there is significant anisotropy. We run and compare all the aforementioned gradient flow approaches. We initialize them at $\N(0,\I)$ instead of the prior; this initialization allows for larger stable time step sizes. Details of particle based approaches, including Wasserstein GF, affine invariant Wasserstein GF, Stein GF, and affine invariant Stein GF have been presented in \Cref{sec-affine-invariance-experiments}. The scaling constants in the definition of kernels $\kappa$ in \cref{eq:Stein-kernel,eq:affine-invariant-Stein-kernel} are dropped, since they only affect the time step size. We use interacting particle systems with $J = 100$ particles.
Details of Gaussian approximated approaches, including Gaussian approximated Fisher-Rao GF, Gaussian approximated 
Waserstein GF, and Gaussian approximated GF have been presented in \Cref{sec-gaussian-numerics}. The expectations in the evolution equations are calculated by the unscented transform with $J = 2N_\theta + 1 = 33$ quadrature points.

\begin{figure}[ht]
\centering
    \includegraphics[width=0.48\textwidth]{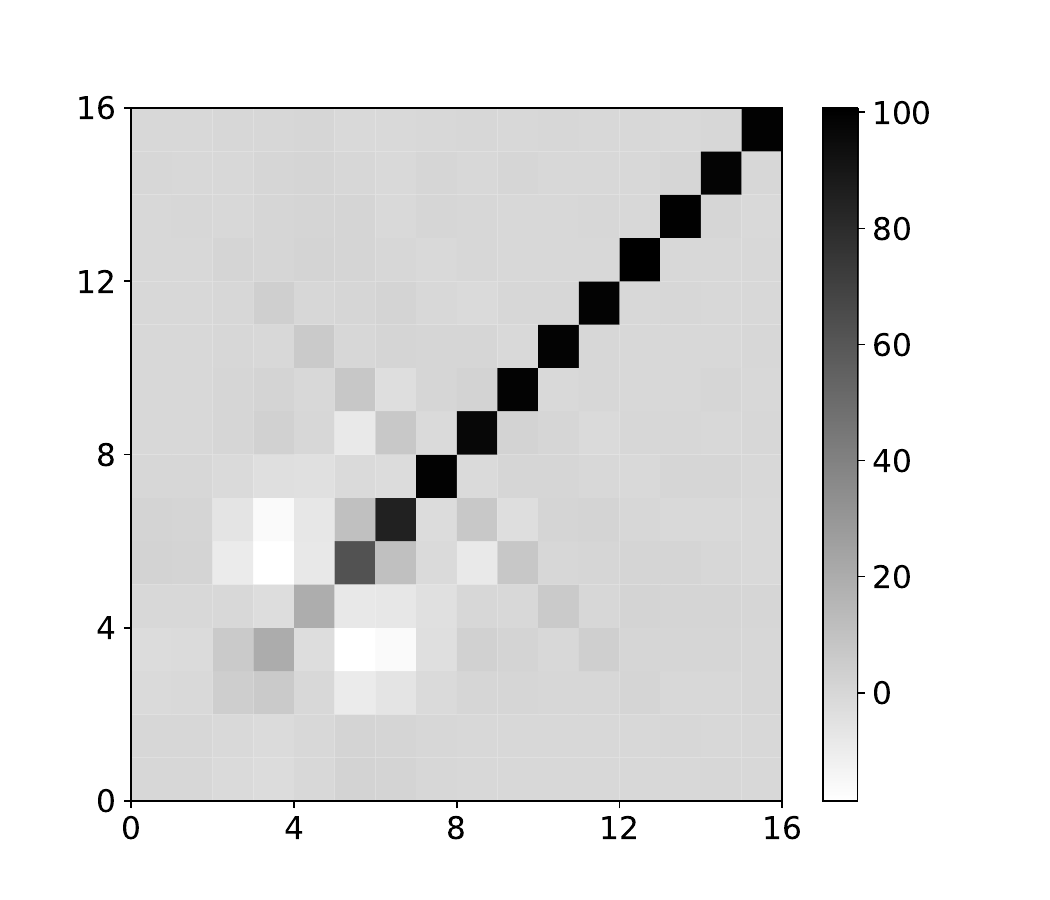}
    \caption{The estimated posterior covariance matrix through MCMC, in the Darcy flow problem.}
    \label{fig:darcy-1d-cov}
\end{figure}

The maximum stable time steps ($\Delta t$) for different approaches are outlined below: 0.030 (Wasserstein GF), 0.030 (Wasserstein GF), 0.990 (Stein GF), 0.090 (Affine Invariant Stein GF), 0.162 (Gaussian Approximated Fisher-Rao GF), 0.002 (Gaussian Approximated Fisher-Rao GF), and 0.018 (Gaussian Approximated Wasserstein GF). We determine these maximum stable time steps through iterative looping, incrementally stepping by $10^{-3}$.

The convergence curves of these approaches (in terms of relative mean errors and relative covariance errors) are presented in \Cref{fig:Darcy-1D-error}. The Gaussian approximate Fisher-Rao GF exhibits notably faster convergence, as a consequence of its relatively large time step size and its affine invariance property. Affine invariant modifications of the Wasserstein GF and Stein GF also show improved convergence rates compared to their non-affine invariant counterparts. Both the Wasserstein GF and affine invariant Wasserstein GF oscillate slightly due to the added noise.

\Cref{fig:darcy-1D-theta} shows properties of the
approximated posterior distributions by different methods, after the $5000^{\rm th}$ iteration. We compare them
with the result by MCMC and with the true parameters of the permeability field (referred to as ``Reference''). 
Notably, the true parameters fall within the 3-$\sigma$ confidence intervals determined by MCMC consistently. We observe that most approaches can reproduce the posterior mean and confidence intervals computed by MCMC accurately after $5000$ iterations, except for the Gaussian approximate GF and Stein GF which exhibit very slow convergence.
\begin{figure}[ht]
\centering
    \includegraphics[width=0.9\textwidth]{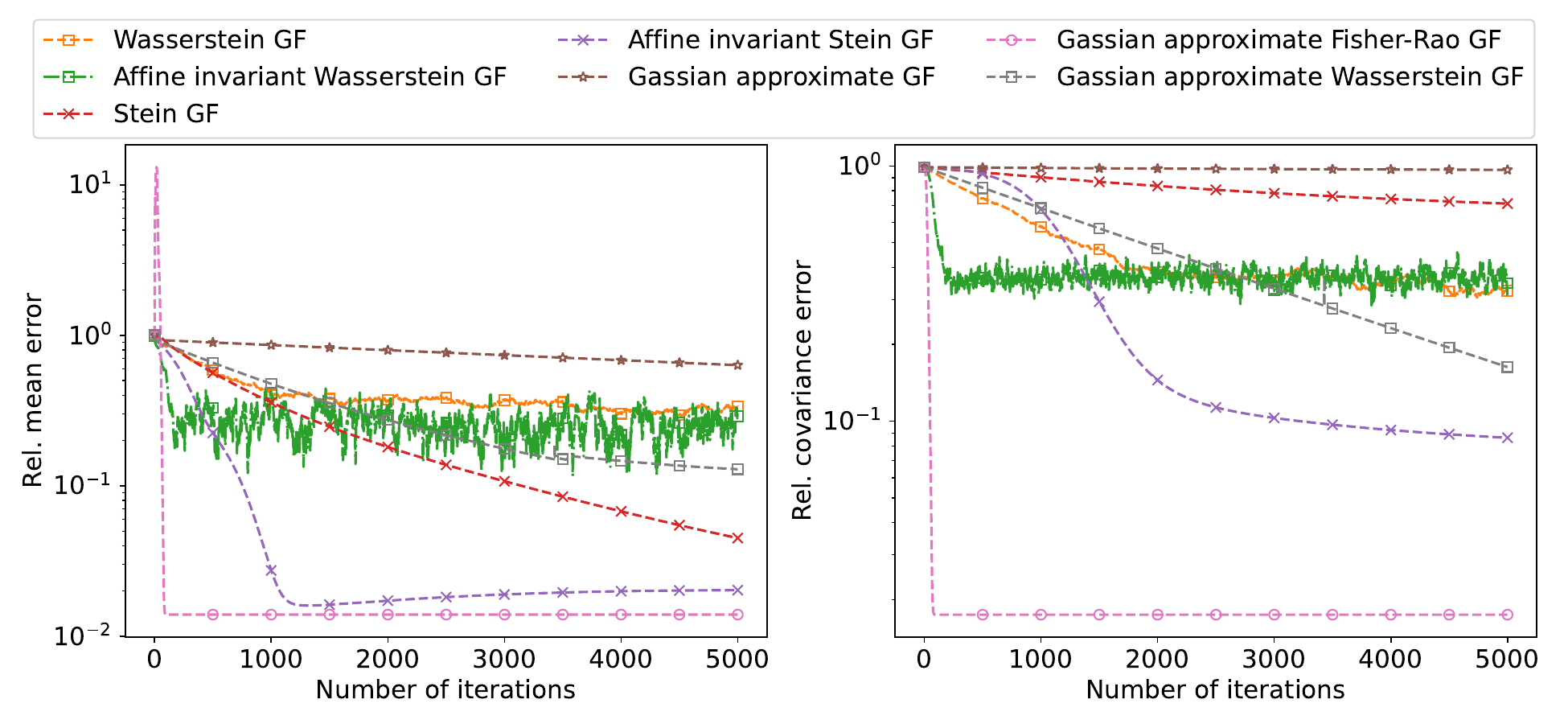}
    \caption{
     The relative mean errors and relative covariance errors of $\theta$ obtained by various gradient flow approaches, in the Darcy flow problem.}
    \label{fig:Darcy-1D-error}
\end{figure}

\begin{figure}[ht]
\centering
    \includegraphics[width=0.9\textwidth]{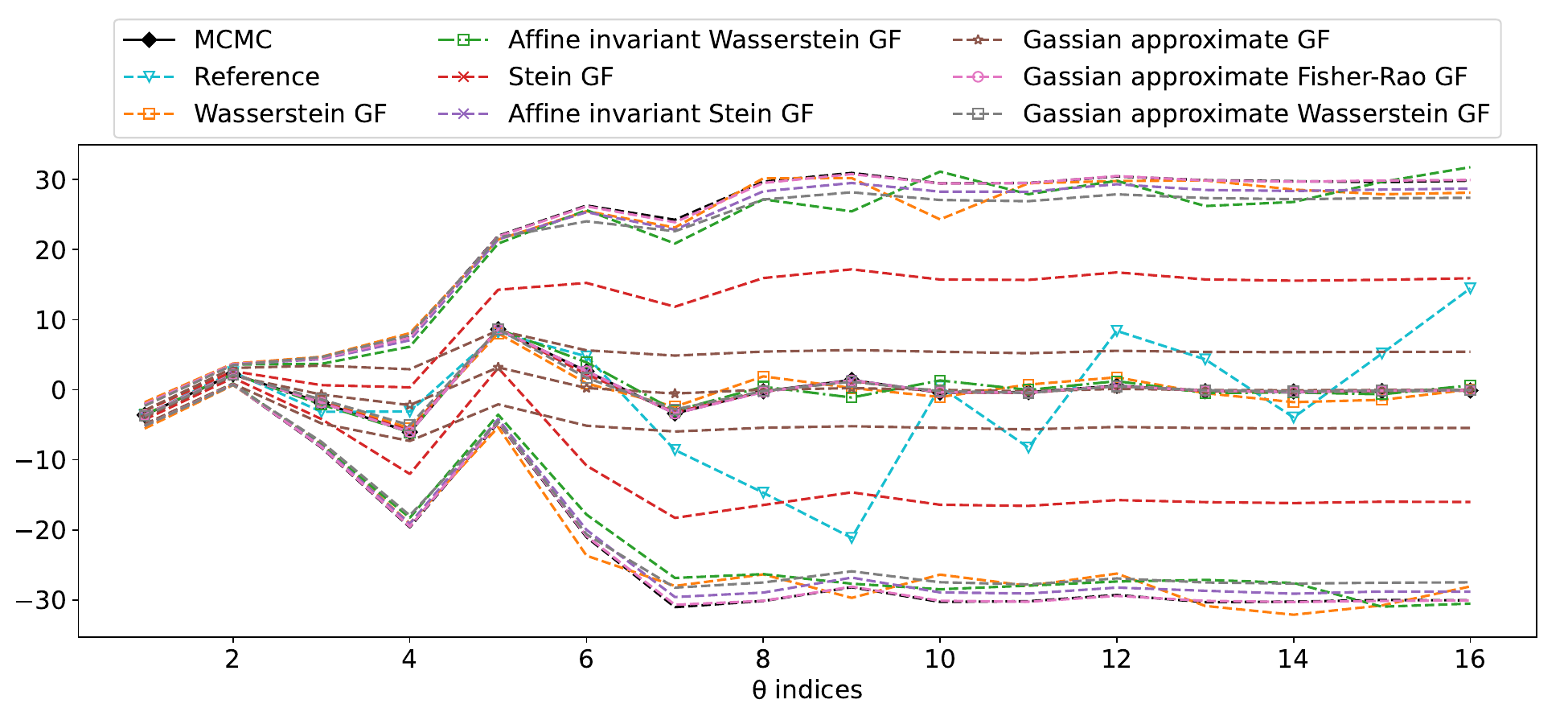}
    \caption{The estimated parameters $\theta_{(i)}$ (conditional mean) and the associated 3-$\sigma$ confidence intervals (dashed lines) obtained by various gradient flow approaches and MCMC, in the Darcy flow problem.}
    \label{fig:darcy-1D-theta}
\end{figure}
These experiments on Darcy flow inversion illustrate the successful applications of gradient flow approaches, designed using the KL energy functional with various metrics and numerically approximated by either particles or Gaussians, in tackling inverse problems in scientific computing. Particularly noteworthy is the efficient, robust performance of the Gaussian approximate Fisher-Rao gradient flow, an observation closely aligns with our theoretical expectations. 

For inverse problems in higher dimensions and in many natural science and engineering applications, computing derivatives of the forward map is costly and may not even be feasible \cite{huang2022iterated}. In such scenario, one can further apply derivative-free approximations to the Gaussian approximate and particle-approximate gradient flows. Specifically for Fisher-Rao gradient flows, the resulting algorithms connect to ensemble Kalman filters and unscented Kalman filters. In fact, the Kalman inversion algorithm in \cite{huang2022efficient} can be interpreted as a derivative-free approximation of Fisher-Rao gradient flows, which are shown to be very efficient in several large-scale applications. From this standpoint, the paper provides a theoretical foundation for designing various gradient flows to tackle these challenging inverse problems.

\section{Conclusions}
\label{sec:conclusion}

In this work, we have studied the design ingredients of continuous-time gradient flows for sampling distributions with unknown normalization constants, focusing on the energy functional and the metric tensor. Regarding the energy functional, we show  that
the KL divergence has the \textit{unique} property (among all $f$-divergences) that gradient flows resulting from this energy functional do not depend on the normalization constant of the target distribution. This makes the KL divergence advantageous in terms of numerical implementations. {Regarding the metric tensor, we highlight the importance of invariance properties and, in particular, their influence on the convergence rates of the gradient flow. The unique diffemorphism invariance property of the Fisher-Rao gradient flow allows it to achieve a uniform exponential convergence rate under general conditions;} however the particle implementation of the Fisher-Rao gradient flow is highly non-trivial. We introduce a relaxed, affine invariance property for the gradient flows; in particular, we construct various affine invariant Wasserstein and Stein gradient flows. These affine invariant gradient flows are more convenient to approximate numerically via particle methods than the diffeomorphism invariant Fisher-Rao gradient flow, and they behave more favorably compared to their non-affine-invariant versions when sampling highly anisotropic distributions.

In addition, we study Gaussian approximations of the flow that lead to efficient implementable algorithms as alternatives to particle methods. In particular, Gaussian approximations can be readily applied for the diffeomorphism invariant Fisher-Rao gradient flow. Our theory and numerical experiments demonstrate the strengths and potential limitations of the Gaussian approximate Fisher-Rao gradient flow, which is affine invariant, for a wide range of target distributions.

Our study shows that for general posterior 
distributions beyond the Gaussian and logconcave class, the consideration of affine invariance and Gaussian approximation may not be enough for designing an accurate and efficient sampler. Examples include posteriors that concentrate on manifolds with significant curvatures such as the
Rosenbrock function and also multimodal posteriors. In our future work, we will explore more sophisticated approximations such as Gaussian mixtures \cite{lin2019fast, daudel2021mixture,daudel2021infinite} and other invariance properties to design samplers for these challenging scenarios.

\vspace{0.5in}
\noindent {\bf Data and Code Availability.} All codes used to produce the
numerical results and figures in this paper are available at
\url{https://github.com/Zhengyu-Huang/InverseProblems.jl/tree/master/Gradient-Flow}. 

\vspace{0.5in}
\noindent {\bf Acknowledgments.} YC acknowledges the support from the Air Force Office of Scientific Research under MURI award number FA9550-20-1-0358 (Machine Learning and Physics-Based Modeling and Simulation). DZH and AMS are supported by NSF award AGS1835860 and by the generosity of Eric and Wendy Schmidt by recommendation of the Schmidt Futures program; 
DZH is also supported by High-performance Computing Platform of Peking University;
AMS is also
supported by the Office of Naval Research (ONR) through grant N00014-17-1-2079
and by a Department of Defense Vannevar Bush Faculty Fellowship. SR is supported by Deutsche Forschungsgemeinschaft (DFG) - Project-ID 318763901 - SFB1294.
JH is supported by NSF grant DMS-2054835.

\bibliographystyle{plain}

\newpage
\appendix

\section{Unique Property of the KL Divergence Energy}
\subsection{Proof of \Cref{thm: KL unique f divergence}}
\label{sec: Proofs of prop: KL unique f divergence}
\begin{proof}
We first note that the KL divergence satisfies the desired property:
for any $c \in (0,\infty)$ and $\rho_{\rm post} \in \PP$ it holds that
$$\mathrm{KL}[\rho \Vert  c\rho_{\rm post}]-\mathrm{KL}[\rho \Vert  \rho_{\rm post}] = -\log c.$$ 
Now we establish uniqueness. For any $f-$divergence with property that
$D_f[\rho \Vert c\rho_{\rm post}]-D_f[\rho \Vert \rho_{\rm post}]$ is independent of $\rho$, we have
\begin{equation}
    \lim_{t \to 0} \frac{(D_f[\rho + t\sigma \Vert c\rho_{\rm post}]-D_f[\rho + t\sigma \Vert \rho_{\rm post}]) - (D_f[\rho \Vert c\rho_{\rm post}]-D_f[\rho \Vert \rho_{\rm post}]) }{t} = 0,
\end{equation}
for any bounded, smooth function $\sigma$ supported in $B^d(0,R)$ that integrates to zero. Here $R > 0$ is a finite parameter that we will choose later, and $d=N_{\theta}$ is the dimension of $\theta$. In the above formula we have used the fact that for sufficiently small $t$, one has $\rho+t\sigma \in \PP$ since $R$ is finite, so the formula is well defined. By direct calculations, we get
\begin{equation}
    \int_{B^d(0,R)} \left(f'(\frac{\rho}{c\rho_{\rm post}}) -  f'(\frac{\rho}{\rho_{\rm post}})\right)\sigma {\rm d}\theta= 0.
\end{equation}
Since $\sigma$ is arbitrary, $f'(\frac{\rho}{c\rho_{\rm post}}) -  f'(\frac{\rho}{\rho_{\rm post}})$ must be a constant function in $B^d(0,R)$. 



Because $\rho$ and $\rho_{\rm post}$ integrate to $1$ and they are continuous, there exists $\theta^\dagger$ such that $\rho (\theta^\dagger)/\rho_{\rm post}(\theta^\dagger) = 1$. Choose $R$ sufficiently large such that $\theta^\dagger \in B^d(0,R)$. Then, we obtain
\begin{equation}
    \label{eq:A1}
f'\Bigl(\frac{\rho(\theta)}{\rho_{\rm post}(\theta)}\Bigr) - 
f'\Bigl(\frac{\rho(\theta)}{c\rho_{\rm post}(\theta)}\Bigr)
=f'\Bigl(\frac{\rho(\theta^\dagger)}{\rho_{\rm post}(\theta^\dagger)}\Bigr) - 
f'\Bigl(\frac{\rho(\theta^\dagger)}{c\rho_{\rm post}(\theta^\dagger)}\Bigr) = f'(1) - f'(1/c),
\end{equation}
for any $\theta \in B^d(0,R)$. As $R$ can be arbitrarily large, the above identity holds for all $\theta \in \R^d$. Let $g(1/c):=f'(1) - f'(1/c)$. We have obtained
\begin{equation}
\label{eqn A1}
    f'\Bigl(\frac{\rho(\theta)}{\rho_{\rm post}(\theta)}\Bigr) - f'\Bigl(\frac{\rho(\theta)}{c\rho_{\rm post}(\theta)}\Bigr) = g(1/c),
\end{equation}
where $c$ is an arbitrary positive number. Note, furthermore, that $g(\cdot)$ 
is continuous since $f$ is continuously differentiable. 
Since $\rho$ and $\rho_{\rm post}$ are arbitrary, 
we can write \eqref{eqn A1} equivalently as
\begin{equation}
\label{eqn: KL unique f divergence, function equality}
    f'(y)-f'(cy) = g(c),
\end{equation}
for any $y, c \in \mathbb{R}_{+}$. Let $h: \mathbb{R} \to \mathbb{R}$ such that $h(z) = f'(\exp(z))$. Then, we can equivalently formulate \eqref{eqn: KL unique f divergence, function equality} as
\begin{equation}
\label{eq:a3}
    h(z_1) - h(z_2) = r(z_1 - z_2),
\end{equation}
for any $z_1, z_2\in \mathbb{R}$ and $r:\mathbb{R} \to \mathbb{R}$ such that $r(t) = g(\exp(-t))$. 

We can show $r$ is linear function. Setting $z_1=z_2$ in \eqref{eq:a3}
shows that $r(0)=0.$ Note also that, again by \eqref{eq:a3},
$$r(z_1-z_2)+r(z_2-z_3)=h(z_1)-h(z_3)=r(z_1-z_3).$$
Hence, since $z_1,z_2$ and $z_3$ are arbitrary, we deduce that
for any $x, y \in \mathbb{R}$, it holds that
\begin{equation}
\label{eqn A4}
    r(x) + r(y) = r(x+y).
\end{equation}
Furthermore $r$ is continuous since $f$ is assumed continuously
differentiable. With the above conditions, it is a standard result in functional equations that $r(x)$ is linear. Indeed, as a sketch of proof, by \eqref{eqn A4} we can first deduce $r(n) = n r(1)$ for $n \in \mathbb{Z}$. Then, by setting $x,y$ to be dyadic rationals, we can deduce $r(\frac{i}{2^k}) = \frac{i}{2^k} r(1)$ for any $i, k \in \mathbb{Z}$. Finally using the continuity of the function $r$, we get $r(x) = xr(1)$ for any $x \in \mathbb{R}$. For more details see \cite{friedman1962functional}.

Using the fact that $r$ is linear and the equation \eqref{eq:a3}, we know that $h$ is an affine function and thus $f'(\exp(z)) = az+b$ for some $a, b \in \mathbb{R}$. Equivalently, $f'(y) = a \log(y) + b$. Using the condition $f(1) = 0$, we get $f(y) = a y \log(y) + (b-a)(y-1)$. Plugging this $f$ into the formula for $D_f$, we get
\[D_f[\rho||\rho_{\rm post}] = a \text{KL}[\rho||\rho_{\rm post}],\]
noting that the affine term in $f(y)$ has zero contributions in the formula for $D_f$.
The proof is complete.
\end{proof}

\section{Proof for the Convergence of Fisher-Rao Gradient Flows}
\subsection{Proof of \Cref{thm:FR-convergence}}
\label{proof:FR-convergence}
\begin{proof}
The Fisher-Rao gradient flow of the KL divergence~\cref{eq:mean-field-Fisher-Rao} 
can be solved analytically using the variation of constants formula as follows.
First note that
\begin{align*}
&\frac{\partial \log \rho_t(\theta)}{\partial t} 
= \log \rho_{\rm post}(\theta) - \log \rho_t(\theta) - \E_{\rho_t}[\log \rho_{\rm post}(\theta) - \log \rho_t(\theta)],
\end{align*}
so that
\begin{align*}
&\frac{\partial e^t \log \rho_t(\theta)}{\partial t} 
= e^t\log \rho_{\rm post}(\theta) - e^t\E_{\rho_t}[\log \rho_{\rm post}(\theta) - \log \rho_t(\theta)].
\end{align*}
Thus
\begin{equation*}\log \rho_t(\theta)
= (1-e^{-t})\log \rho_{\rm post}(\theta) + e^{-t}\log\rho_0(\theta) -  \int_0^{t} e^{\tau-t}\E_{\rho_\tau}[\log \rho_{\rm post}(\theta) - \log \rho_\tau(\theta)] \dd \tau.
\end{equation*}
It follows that
\begin{align}\label{e:rhot}
&\rho_t(\theta) =\frac{1}{Z_t} \rho_0(\theta)^{ e^{-t}}\rho_{\rm post}(\theta)^{1 - e^{-t}},\quad \frac{\rho_t(\theta)}{\rho_{\rm post}(\theta)}=\frac{1}{Z_t}\left(\frac{\rho_0(\theta)}{\rho_{\rm post}(\theta)}\right)^{e^{-t}},
\end{align}
where $Z_t$ is a normalization constant. Indeed, we have the formula
\[Z_t=\int \left(\frac{\rho_0(\theta)}{\rho_{\rm post}(\theta)}\right)^{ e^{-t}}\rho_{\rm post}(\theta) \dd \theta\]
by using the condition $\int \rho_t(\theta) {\rm d}\theta = 1$.

In the following, we first obtain the following lower bound on $Z_t$: 
\begin{align}\label{e:Ztbound}
    Z_t\geq e^{-Ke^{-t}(1+B)},
\end{align}
where the constants $K,B$ are from \cref{e:asup1} and \cref{e:asup2}. 
In fact, using our assumptions \cref{e:asup1} and \cref{e:asup2}, we have
\begin{align*}
    Z_t&=\int \left(\frac{\rho_0(\theta)}{\rho_{\rm post}(\theta)}\right)^{ e^{-t}}\rho_{\rm post}(\theta) \dd \theta
    \geq \int \left(e^{-K(1+|\theta|^2)}\right)^{ e^{-t}}\rho_{\rm post}(\theta) \dd \theta\\
    &=\int e^{-Ke^{-t}(1+|\theta|^2)}\rho_{\rm post}(\theta) \dd \theta
    \geq  e^{\int-Ke^{-t}(1+|\theta|^2)\rho_{\rm post}(\theta)\dd \theta} \geq e^{-Ke^{-t}(1+B)},
\end{align*}
where in the second to last inequalities, we used Jensen's inequality and the fact that $e^x$ is convex. 

By plugging \cref{e:asup2} and \cref{e:Ztbound} into \cref{e:rhot}, we get
\begin{align}\label{e:density_ratio}
    \frac{\rho_t(\theta)}{\rho_{\rm post}(\theta)}
    =\frac{1}{Z_t}\left(\frac{\rho_0(\theta)}{\rho_{\rm post}(\theta)}\right)^{e^{-t}}
    \leq e^{Ke^{-t}(1+B)} e^{Ke^{-t}(1+|\theta|^2)}
    =e^{Ke^{-t}(2+B+|\theta|^2)}.
\end{align}
Using \cref{e:density_ratio}, we get the following upper bound on the KL divergence:
\begin{align}\begin{split}\label{e:KLub1}
    & {\rm KL}\Bigl[\rho_{t} \Big\Vert  \rho_{\rm post}\Bigr]\\
    &\quad=\int \rho_t(\theta) \log \frac{\rho_t(\theta)}{\rho_{\rm post}(\theta)}{\rm d} \theta
    \leq 
    \int \rho_t(\theta) \log(e^{Ke^{-t}(2+B+|\theta|^2)}){\rm d} \theta\\
    &\quad=\int \rho_t(\theta) Ke^{-t}(2+B+|\theta|^2){\rm d} \theta
    =Ke^{-t}\left((2+B)+\int  |\theta|^2 \rho_t(\theta)  {\rm d} \theta\right).
\end{split}\end{align}
For the last integral in \cref{e:KLub1}, using \cref{e:rhot} and the
H\"older inequality, we can rewrite it as
\begin{align}\begin{split}\label{e:KLub2}
    &\phantom{{}={}}\frac{1}{Z_t}\int  |\theta|^2 \rho_0(\theta)^{e^{-t}}\rho_{\rm post}(\theta)^{1-e^{-t}}  {\rm d} \theta
    =\frac{1}{Z_t}\int   (|\theta|^2\rho_0(\theta))^{e^{-t}}(|\theta|^2\rho_{\rm post}(\theta))^{1-e^{-t}}  {\rm d}\theta\\
    &\leq \frac{1}{Z_t}\left(\int   |\theta|^2\rho_0(\theta)\dd \theta\right)^{e^{-t}}\left(\int (|\theta|^2\rho_{\rm post}(\theta))  {\rm d} \theta\right)^{1-e^{-t}}\leq Be^{K e^{-t}(1+B)},
\end{split}\end{align}
where for the last inequality we used \cref{e:Ztbound} and our assumption \cref{e:asup2}.

Combining \cref{e:KLub1} and \cref{e:KLub2} together, for $t\geq \log((1+B)K)$, we have
\begin{align}
    {\rm KL}\Bigl[\rho_{t} \Big\Vert  \rho_{\rm post}\Bigr]\leq Ke^{-t}(2+B + Be^{K e^{-t}(1+B)})
    \leq (2+B + eB)Ke^{-t}.
\end{align}
This completes the proof of \Cref{thm:FR-convergence}.
\end{proof}
\section{Intuitions Regarding the Definitions of Various Metric Tensors}
\vspace{1em}
\subsection{The Fisher-Rao Riemannian Metric}
\label{appendix: intuition Fisher-Rao Riemannian Metric}
Writing tangent vectors on a multiplicative scale, by setting $\sigma=\rho \psi_\sigma$, we
see that this metric may be written as $g_{\rho}^{\mathrm{FR}}(\sigma_1,\sigma_2) = \int \psi_{\sigma_1} \psi_{\sigma_2} \rho \mathrm{d}\theta,$
and hence that in the $\psi_\sigma$ variable 
the metric is described via the $L_\rho^2$ inner-product. That is, the Fisher-Rao Riemannian metric measures the multiplicative factor via the $L_\rho^2$ energy. In \Cref{appendix: intuition Wasserstein Riemannian Metric}, we will see that another important metric, the Wasserstein Riemannian metric, may also be understood as a $L_{\rho}^2$ measurement, but of the velocity field instead.
\subsection{The Wasserstein Riemannian Metric}
\label{appendix: intuition Wasserstein Riemannian Metric}
   The Wasserstein Riemannian metric has a transport interpretation. To understand
   this fix $\sigma \in T_\rho \PP$ and consider the family of \emph{velocity
    fields} $v$ related to $\sigma$ via the constraint $\sigma = -\nabla_\theta \cdot (\rho v)$. Then define $v_{\sigma} = \argmin_v \int \rho |v|^2$ in which
   the minimization is over all $v$ satisfying the constraint. A formal Lagrange multiplier argument can be used to deduce that $v_{\sigma} = \nabla_\theta \psi_{\sigma}$ for some $\psi_{\sigma}$. This motivates the relationship
   appearing in \cref{eq:liouville} as well as the form of the
Wasserstein Riemannian metric appearing in \cref{eqn-Wasserstein-Riemannian-metric}
which may then be viewed as measuring the kinetic energy $\int \rho |v_{\sigma}|^2 \mathrm{d}\theta$. We emphasize that, for ease of understanding, our discussion on the Riemannian structure of the Wasserstein metric is purely formal; for rigorous treatments, the reader can consult \cite{ambrosio2005gradient}.
\subsection{The Stein Riemannian Metric}
\label{appendix: intuition Stein Riemannian Metric}
As in the Wasserstein setting, the Stein Riemannian metric~\cite{liu2017stein} also has a transport 
interpretation. The Stein Riemannian metric identifies, for each $\sigma \in T_\rho \PP$, the set of velocity fields $v$ satisfying the constraint $\sigma = -\nabla_\theta \cdot (\rho v)$. Then $v_{\sigma} = \argmin_v  \|v\|_{\mathcal{H}_\kappa}^2$, with minimization over all $v$ satisfying the constraint, and where $\mathcal{H}_\kappa$ is a Reproducing Kernel Hilbert Space (RKHS) with kernel $\kappa$. A formal Lagrangian multiplier argument shows that 
$$v_{\sigma} = \int \kappa(\theta,\theta',\rho) \rho(\theta')\nabla_{\theta'} \psi_{\sigma}(\theta') \mathrm{d}\theta'$$ for some $\psi_{\sigma}.$ The Stein metric measures this transport change via the RKHS norm $\|v_{\sigma}\|_{\mathcal{H}_\kappa}^2$, leading to
the interpretation that the Stein Riemannian metric can be written in the form
\[g_{\rho}^{\mathrm{S}}(\sigma_1,\sigma_2) = \langle v_{\sigma_1}, v_{\sigma_2}\rangle_{\mathcal{H}_\kappa}. \]
\section{Proofs for Affine Invariance of Gradient Flows}
\label{appendix:affine-invariant}

\subsection{Preliminaries}
In this section, we present a lemma concerning the change-of-variable formula under the gradient and divergence operators, which is useful for our proofs in later subsections.
\begin{lemma}
\label{lemma:change_variable}
Consider any invertible affine mapping from $\theta \in \R^{N_{\theta}}$ to  $\tilde\theta \in \R^{N_{\theta}}$ defined by $\tilde\theta =\varphi(\theta)= A \theta + b$. 
For any differentiable scalar field $f: \R^{N_{\theta}} \rightarrow \R$ and 
vector field $g: \R^{N_{\theta}} \rightarrow \R^{N_{\theta}}$, we have
$$\nabla_{\theta} f(\theta) = A^{T} \nabla_{\tilde{\theta}} \tilde{f}(\tilde\theta) \quad \text{and} \quad \nabla_{\theta} \cdot g(\theta) =  \nabla_{\tilde{\theta}} \cdot ( A \tilde{g} (\tilde\theta) ), $$
where $\tilde{f}(\tilde{\theta}) := f(A^{-1}(\tilde{\theta} - b))$ and $\tilde{g}(\tilde{\theta}) := g(A^{-1}(\tilde{\theta} - b))$.
\end{lemma}
\begin{proof} Note that $\tilde{f}(\tilde{\theta}) = f(\theta)$ and $\tilde{g}(\tilde{\theta}) = g(\theta)$. By direct calculations, we get
\begin{align*}
&[\nabla_{\theta} f(\theta)]_i = \frac{\partial f(\theta)}{\partial \theta_i} 
= 
\sum_j\frac{\partial f(\theta)}{\partial \tilde{\theta}_j}\frac{\partial \tilde{\theta}_j }{\partial \theta_i} 
= 
\sum_j\frac{\partial \tilde{f}(\tilde{\theta})}{\partial \tilde{\theta}_j}A_{ji} = [A^T\nabla_{\tilde{\theta}} \tilde{f}(\tilde{\theta})]_i,
\\
    &\nabla_{\theta} \cdot g(\theta) = \sum_i \frac{\partial g_i(\theta)}{\partial\theta_i} = \sum_i \sum_j \frac{\partial g_i(\theta)}{\partial\tilde{\theta}_j}\frac{\partial \tilde{\theta}_j}{\partial \theta_i} =  \sum_i \sum_j \frac{\partial \tilde{g}_i(\tilde{\theta})}{\partial\tilde{\theta}_j}A_{ji} = \nabla_{\tilde{\theta}} \cdot ( A \tilde{g} (\tilde\theta) ).
\end{align*}
This completes the proof.
\end{proof}

\begin{newremark}
\label{remark: change of variable grad and div}
    Since $\tilde{f}(\tilde{\theta}) = f(\theta)$ and $\tilde{g}(\tilde{\theta}) = g(\theta)$, we can also summarize the result in \cref{lemma:change_variable} as $\nabla_{\theta} f = A^T\nabla_{\tilde{\theta}}f$ and $\nabla_{\theta}\cdot g = \nabla_{\tilde{\theta}}\cdot (Ag)$.
\end{newremark}

\subsection{Proof of \Cref{thm:Wasserstein-affine-invariant}}
\label{proof:Wasserstein-affine-invariant}
\begin{proof}
We write down the form of the affine invariant Wasserstein gradient flow as follows:
\begin{align}
\label{appendix eqn: non-transformed Wasserstein GF}
\frac{\partial \rho_t(\theta)}{\partial t} = 
\nabla_{\theta} \cdot \Bigl(\rho_t \Prec(\theta, \rho_t) (\nabla_{\theta}\log \rho_t - \nabla_{\theta} \log \rho_{\rm post})
\Bigr).
\end{align}
Consider $\tilde{\theta} = \varphi(\theta)=A\theta + b$ and $\tilde{\rho}_t = \varphi \# \rho_t$ for an invertible affine transformation $\varphi$. Then, it suffices to show that under the assumption $\Prec(\tilde\theta, \tilde\rho) = A \Prec(\theta, \rho) A^T$, one has
\begin{align}
\label{appendix eqn: transformed Wasserstein GF}
\frac{\partial \tilde{\rho}_t(\tilde{\theta})}{\partial t} 
=\nabla_{\tilde{\theta}} \cdot \Bigl(\tilde{\rho}_t \Prec(\tilde{\theta}, \tilde{\rho}_t) (\nabla_{\tilde\theta}\log \tilde\rho_t - \nabla_{\tilde\theta} \log \tilde\rho_{\rm post})
\Bigr).
\end{align}

First, for the left hand side, by definition, we have 
\begin{equation}
\label{appendix eqn: transform of rho}
    \tilde{\rho}_t(\tilde{\theta}) = \rho_t(\varphi^{-1}(\tilde{\theta}))|\nabla_{\tilde{\theta}} \varphi^{-1}(\tilde{\theta})| = \rho_t(\theta)|A^{-1}|.
\end{equation}
For the right hand side, the chain rule leads to that 
\begin{equation}
\label{appendix eqn: change of variable grad of first variation}
    \nabla_{\tilde\theta}\log \tilde\rho_t = A^{-T}  \nabla_{\theta}\log \rho_t   \qquad
    \nabla_{\tilde\theta} \log \tilde\rho_{\rm post}  = A^{-T}  \nabla_{\theta}\log \rho_{\rm post}. 
\end{equation}
Therefore, we can write the right hand side of \cref{appendix eqn: transformed Wasserstein GF} as
\begin{equation}
\label{appendix: change of variables for rhs of Wasserstein gf}
\begin{aligned}
&\nabla_{\tilde{\theta}} \cdot \Bigl(\tilde{\rho}_t \Prec(\tilde{\theta}, \tilde{\rho}_t) (\nabla_{\tilde\theta}\log \tilde\rho_t - \nabla_{\tilde\theta} \log \tilde\rho_{\rm post})
\Bigr) \\
= &\nabla_{\tilde{\theta}} \cdot \Bigl(\tilde{\rho}_t \Prec(\tilde{\theta}, \tilde{\rho}_t) A^{-T}(\nabla_{\theta}\log \rho_t - \nabla_{\theta} \log \rho_{\rm post})
\Bigr)\\
= &\nabla_{\theta} \cdot \Bigl(\tilde{\rho}_t A^{-1}\Prec(\tilde{\theta}, \tilde{\rho}_t) A^{-T}(\nabla_{\theta}\log \rho_t - \nabla_{\theta} \log \rho_{\rm post})
\Bigr)\\
= &\nabla_{\theta} \cdot \Bigl(\rho_t A^{-1}\Prec(\tilde{\theta}, \tilde{\rho}_t) A^{-T}(\nabla_{\theta}\log \rho_t - \nabla_{\theta} \log \rho_{\rm post}) 
\Bigr) \cdot |A^{-1}|,
\end{aligned}
\end{equation}
where in the second equality, we used \cref{appendix eqn: change of variable grad of first variation}, and in the third equality, we used \cref{lemma:change_variable}.
Based on \cref{appendix eqn: transform of rho}, \cref{appendix: change of variables for rhs of Wasserstein gf} and \cref{appendix eqn: non-transformed Wasserstein GF}, a sufficient condition for \cref{appendix eqn: transformed Wasserstein GF} to hold is $A^{-1}\Prec(\tilde{\theta}, \tilde{\rho}) A^{-T} = P(\theta, \rho)$, or equivalently, $\Prec(\tilde\theta, \tilde\rho) = A \Prec(\theta, \rho) A^T$. This completes the proof.
\end{proof}

\subsection{Proof of \Cref{lem:AI-Wasserstein-MD}}
\label{proof:AI-Wasserstein-MD}


\begin{proof}
Consider the invertible affine transformation $\tilde\theta = \varphi(\theta) = A \theta + b$ and correspondingly $\tilde{\rho} = \varphi \# \rho$ and $\tilde{\rho}_{\rm post} = \varphi \# \rho_{\rm post}$. Using \cref{eq:aMFD-AI-sigma}, we get
\begin{align*}
    D(\tilde{\theta},\tilde\rho) = \frac{1}{2}h(\tilde{\theta},\tilde\rho)h(\tilde{\theta},\tilde\rho)^T = \frac{1}{2}Ah(\theta,\rho)h(\theta,\rho)^TA^T = AD(\theta, \rho)A^T.
\end{align*}
Similarly, it holds that $d(\tilde{\theta},\tilde\rho) = A d(\theta,\rho)$.
Based on these relations, we can calculate as follows:
\begin{equation}
\begin{aligned}    
&Af(\theta, \rho, \rho_{\rm post})\\
&\quad = A\Prec(\theta, \rho)\nabla_{\theta}\log \rho_{\rm post}(\theta) + A(D(\theta,\rho)-P(\theta,\rho))\nabla_{\theta}\log \rho(\theta) - Ad(\theta,\rho)\\
&\quad= 
A\Prec(\theta, \rho)A^T\nabla_{\tilde{\theta}} \log \tilde{\rho}_{\rm post}(\tilde{\theta})  + \bigl(A D(\theta,\rho) - A\Prec(\theta, \rho) \bigr)A^T\nabla_{ \tilde{\theta}} \log \tilde{\rho}(\tilde{\theta}) - Ad(\theta,\rho)  
\\
&\quad= 
\Prec(\tilde{\theta}, \tilde{\rho})\nabla_{\tilde{\theta}} \log \tilde{\rho}_{\rm post}(\tilde{\theta})  + \bigl( D(\tilde{\theta},\tilde\rho) - \Prec(\tilde{\theta}, \tilde{\rho}) \bigr)\nabla_{ \tilde{\theta}} \log \tilde{\rho}(\tilde{\theta}) - d(\tilde\theta,\tilde\rho)  
\\
&\quad=f(\tilde\theta, \tilde\rho, \tilde\rho_{\rm post}).
\end{aligned}
\end{equation}
The first equality is by definition. In the second equality, we used $A^T\nabla_{\tilde{\theta}} \log \tilde{\rho}(\tilde{\theta}) = \nabla_{{\theta}} \log {\rho}({\theta})$. In the third equality, we used the condition in \cref{thm:Wasserstein-affine-invariant}. With this result, the mean-field equation is affine invariant (\cref{def: affine invariant mean-field equation}). The proof is complete.
\end{proof}

\subsection{Proof of \Cref{proposition:Stein-affine-invariant}}
\label{proof:Stein-affine-invariant}

\begin{proof}
The proof is similar to that in \Cref{proof:Wasserstein-affine-invariant}. The affine invariant Stein gradient flow has the form
\begin{align}
\label{appendix eqn: nontransformed stein GF}
\frac{\partial \rho_t(\theta)}{\partial t} = \nabla_{\theta} \cdot \Bigl[\rho_t(\theta)\int \kappa(\theta,\theta',\rho_t)\rho_t(\theta')\Prec( \theta, \theta',\rho_t)\nabla_{\theta'}\bigl(\log\rho_t(\theta') - \log \rho_{\rm post}(\theta') \bigr)\dd \theta' \Bigr]. 
\end{align}
Consider $\tilde{\theta} = \varphi(\theta)=A\theta +b$ and $\tilde{\rho}_t = \varphi \# \rho_t$ for an invertible affine transformation $\varphi$. Then, it suffices to show that under the assumed condition, one has
\begin{align}
\label{appendix eqn: transformed stein GF}
\frac{\partial \tilde{\rho}_t(\tilde{\theta})}{\partial t} 
&= \nabla_{\tilde\theta} \cdot \Bigl[\tilde\rho_t(\tilde\theta)\int \kappa(\tilde{\theta},\tilde{\theta}',\tilde\rho_t)\tilde{\rho}_t(\tilde\theta')\Prec( \tilde{\theta}, \tilde{\theta}',\tilde{\rho}_t) \nabla_{\tilde\theta'}\bigl(\log\tilde\rho_t(\tilde\theta') - \log \tilde\rho_{\rm post}(\tilde\theta') \bigr)\dd \tilde\theta' \Bigr].
\end{align}
For the right hand side of \cref{appendix eqn: transformed stein GF}, we have
\begin{equation}
\label{appendix eqn: rhs of transformed stein flow}
    \begin{aligned}
        &\nabla_{\tilde\theta} \cdot \Bigl[\tilde\rho_t(\tilde\theta)\int \kappa(\tilde{\theta},\tilde{\theta}',\tilde\rho_t)\tilde{\rho}_t(\tilde\theta')\Prec( \tilde{\theta}, \tilde{\theta}',\tilde{\rho}_t) \nabla_{\tilde{\theta}'}\bigl(\log\tilde\rho_t(\tilde\theta') - \log \tilde\rho_{\rm post}(\tilde\theta') \bigr)\dd \tilde\theta' \Bigr]\\
        =&\nabla_{\tilde\theta} \cdot \Bigl[\tilde\rho_t(\tilde\theta)\int \kappa(\tilde{\theta},\tilde{\theta}',\tilde\rho_t)\tilde{\rho}_t(\tilde\theta')\Prec( \tilde{\theta}, \tilde{\theta}',\tilde{\rho}_t) A^{-T}\nabla_{\theta'}\bigl(\log\rho_t(\theta') - \log\rho_{\rm post}(\theta') \bigr)\dd \tilde\theta' \Bigr]\\
        =&\nabla_{\tilde\theta} \cdot \Bigl[\tilde\rho_t(\tilde\theta)\int \kappa(\tilde{\theta},\varphi(\theta'),\tilde\rho_t){\rho}_t(\theta')\Prec( \tilde{\theta}, \varphi(\theta'),\tilde{\rho}_t) A^{-T}\nabla_{\theta'}\bigl(\log\rho_t(\theta') - \log\rho_{\rm post}(\theta') \bigr)\dd \theta' \Bigr]\\
        =&\nabla_{\theta} \cdot \Bigl[\tilde\rho_t(\tilde\theta)\int \kappa(\tilde{\theta},\varphi(\theta'),\tilde\rho_t){\rho}_t(\theta')A^{-1}\Prec( \tilde{\theta}, \varphi(\theta'),\tilde{\rho}_t) A^{-T}\nabla_{\theta'}\bigl(\log\rho_t(\theta') - \log\rho_{\rm post}(\theta') \bigr)\dd \theta' \Bigr]\\
        =&(\nabla_{\theta} \cdot \mathsf{f}) \cdot |A^{-1}|\\
        &\mathsf{f}=\rho_t(\theta)\int \kappa(\tilde{\theta},\varphi(\theta'),\tilde\rho_t){\rho}_t(\theta')A^{-1}\Prec( \tilde{\theta}, \varphi(\theta'),\tilde{\rho}_t) A^{-T} \nabla_{\theta'}\bigl(\log\rho_t(\theta') - \log\rho_{\rm post}(\theta') \bigr)\dd \theta',
    \end{aligned}
\end{equation}
where in the first equality, we used \cref{appendix eqn: change of variable grad of first variation}; in the second equality, we changed of coordinates in the integral $\tilde{\theta}' = \varphi(\theta')$; in the third equality, we used~\cref{lemma:change_variable} about $g$, and in the last equality, we used \cref{appendix eqn: transform of rho}.

By \cref{appendix eqn: transform of rho}, \cref{appendix eqn: rhs of transformed stein flow} and \cref{appendix eqn: nontransformed stein GF}, a sufficient condition for \cref{appendix eqn: transformed stein GF} to hold is
\[\kappa(\tilde{\theta},\varphi(\theta'),\tilde\rho)A^{-1}\Prec( \tilde{\theta}, \varphi(\theta'),\tilde{\rho}) A^{-T} = \kappa(\theta,\theta',\rho) \Prec(\theta, \theta', \rho), \]
or equivalently,
$$ \kappa(\tilde\theta, \tilde\theta',\tilde\rho)\Prec(\tilde \theta, \tilde \theta',  \tilde\rho) 
= \kappa(\theta, \theta', \rho) A \Prec(\theta, \theta', \rho) A^T,$$
where $\tilde{\theta}' = \varphi(\theta')$.
This completes the proof.
\end{proof}

\subsection{Proof of \Cref{lem:AI-Stein-MD}}
\label{proof:AI-Stein-MD}


\begin{proof}
Consider the invertible affine transformation $\tilde\theta = \varphi(\theta) = A \theta + b$ and correspondingly $\tilde{\rho} = \varphi \# \rho, \tilde{\rho}_{\rm post} = \varphi \# \rho_{\rm post}$.
By direct calculations, we get
\begin{equation}
\begin{aligned}    
& Af(\theta, \rho, \rho_{\rm post})\\
&\quad =\int \kappa(\theta,\theta',\rho)A\Prec( \theta, \theta',\rho)\nabla_{\theta'} \bigl(\log\rho(\theta') - \log \rho_{\rm post}(\theta') \bigr)\rho(\theta')\dd \theta'\\
&\quad=\int \kappa(\theta,\theta',\rho) A\Prec( \theta, \theta', \rho ) A^T
\bigl( \nabla_{\tilde\theta} \log \tilde{\rho}_{\rm post}(\tilde{\theta}') -  \nabla_{\tilde\theta} \log \tilde{\rho}(\tilde{\theta}')
\bigr) \rho(\theta') \dd \theta'
\\
&\quad =\int \kappa(\tilde\theta,\tilde\theta', \tilde\rho) \Prec( \tilde\theta, \tilde\theta', \tilde\rho )
\bigl( \nabla_{\tilde\theta} \log \tilde{\rho}_{\rm post}(\tilde{\theta}') -  \nabla_{\tilde\theta} \log \tilde{\rho}(\tilde{\theta}')
\bigr) \tilde\rho(\tilde\theta') \dd \tilde\theta'
\\
&\quad =f(\tilde\theta, \tilde\rho, \tilde\rho_{\rm post}).
\end{aligned}
\end{equation}
The first equality is by definition. In the second equality, we used $A^T\nabla_{\tilde{\theta}} \log \tilde{\rho}(\tilde{\theta}) = \nabla_{{\theta}} \log {\rho}({\theta})$. In the third equality, we used the the relation $\rho(\theta') = \tilde{\rho}(\tilde{\theta}')|A|$ due to \cref{appendix eqn: transform of rho}, and ${\rm d}\tilde{\theta}' = |A|{\rm d}\theta'$; we also used the condition in \cref{proposition:Stein-affine-invariant}. With this result, the mean-field equation is affine invariant (\cref{def: affine invariant mean-field equation}). The proof is complete.

\end{proof}


\section{Proofs for Gaussian Approximate Gradient Flows}
\subsection{Preliminaries}
We start with the following Stein's identities concerning
the Gaussian density function $\rho_a.$

\begin{lemma}
\label{lemma: stein equality}
Assume $\theta \sim \N(m, C)$ with density $\rho_a(\theta)=\rho_a(\theta;m,C)$, 
we have
\begin{align}\label{e:derrho}
    \nabla_m \rho_a(\theta)=-\nabla_\theta \rho_a(\theta)\quad \text{and}\quad \nabla_C \rho_a(\theta)=\frac{1}{2}\Hess \rho_a(\theta).
\end{align}
Furthermore, for any scalar field $f:\mathbb{R}^{N_{\theta}} \to \mathbb{R}$ 
and vector field $g:\mathbb{R}^{N_{\theta}} \to \mathbb{R}^{N_{\theta}}$, we have
\begin{equation}
\begin{aligned} 
&\E_{\rho_a}[\nabla_{\theta} g(\theta)] = \nabla_{m} \E_{\rho_a}[g(\theta)] = \Cov[g(\theta), \theta] C^{-1}, \\
&\E_{\rho_a}[\Hess f(\theta)] = \Cov[\nabla_{\theta}f(\theta), \theta] C^{-1}= -C^{-1}\E_{\rho_a}\Bigl[\bigl(C - (\theta - m)(\theta - m)^T\bigr)f\Bigr] C^{-1}.
\end{aligned}
\end{equation}
\end{lemma}

\begin{proof}
For Gaussian density function $\rho_a$, we have
\begin{align*}
    \nabla_m \rho_a(\theta) 
    &= \nabla_m\frac{1}{\sqrt{|2\pi C|}} \exp\Bigl\{-\frac{1}{2}(\theta - m)C^{-1}(\theta - m)\Bigr\}  \\
    &=  C^{-1}(\theta - m)\rho_a(\theta)
  =-\nabla_\theta \rho_a(\theta),
  \\
  \nabla_C \rho_a(\theta) 
  &=\rho_a(\theta) \Bigl(-\frac{1}{2}\frac{\partial \log|C|}{\partial C} - \frac{1}{2}\frac{\partial (\theta - m)^TC^{-1}(\theta - m)}{\partial C} \Bigr) \\
  &=-\frac{1}{2}\rho_a(\theta) \Bigl(C^{-1} - C^{-1} (\theta - m)(\theta - m)^T  C^{-1}\Bigr)
  =\frac{1}{2}\Hess \rho_a(\theta). 
\end{align*}
For any scalar field $f(\theta)$ and vector field $g(\theta)$, we have
\begin{equation}
\begin{aligned} 
\E_{\rho_a}[\nabla_{\theta} g(\theta)] 
&= \int \nabla_{\theta} g(\theta) \rho_a(\theta) \dd \theta
= -\int  g(\theta) \nabla_{\theta}\rho_a(\theta)^T \dd \theta
= \nabla_{m} \E_{\rho_a}[g(\theta)] \\
&= \int g(\theta) (\theta - m)^TC^{-1}\rho_a(\theta)  \dd \theta =    \Cov[g(\theta), \theta] C^{-1}, 
\\
\E_{\rho_a}[\Hess f(\theta)] 
&= \int \Hess f(\theta) \rho_a(\theta) \dd \theta
= 
-\int \nabla_{\theta} f(\theta) \nabla_{\theta} \rho_a(\theta)^T \dd \theta\\
&=
\int \nabla_{\theta} f(\theta) (\theta - m)^TC^{-1} \rho_a(\theta) \dd \theta
= 
\Cov[\nabla_{\theta}f(\theta), \theta] C^{-1} \\
&= 
\int  f(\theta) \Hess \rho_a(\theta) \dd \theta
=
-C^{-1}\E_{\rho_a}\Bigl[\bigl(C - (\theta - m)(\theta - m)^T\bigr)f\Bigr] C^{-1}.
\end{aligned}
\end{equation}
\end{proof}

The following lemma is proved in~\cite[Theorem 1]{lambert2022recursive}:
\begin{lemma}\label{l:gd}
Consider the KL divergence
\begin{align*}
{\rm KL}\Bigl[\rho_a\Big\Vert  \rho_{\rm post}\Bigr] 
= -\frac{1}{2}\log\bigl|C\bigr| - \int \rho_a(\theta) \log \rho_{\rm post} (\theta) \dd \theta + {\rm const}. 
\end{align*}
For fixed $\rho_{\rm post}$ the minimizer of this divergence over the space $\PPG$, so that, for $\rhoa_\star=(m_{\star}, C_{\star})$ and $\rho_{\rhoa_\star}(\theta) = \N(m_{\star}, C_{\star})$, it follows that
\begin{align*}
\E_{\rho_{\rhoa_\star}}\bigl[\nabla_\theta  \log \rho_{\rm post}(\theta)   \bigr] = 0 \quad \textrm{and}
\quad C_{\star}^{-1} = -\E_{\rho_{\rhoa_\star}}\bigl[\Hess\log \rho_{\rm post}(\theta)\bigr].
\end{align*}
\end{lemma}

\subsection{Consistency of Two Gaussian Approximation Approaches in~\Cref{ssec:GGF}}
\label{proof-ssec:GGF}
Here we prove \Cref{prop:1} and \Cref{prop:2},
which identifies specific gradient flows, within the paper, that satisfy the
assumptions required for application of \Cref{prop:1}.

\begin{proof}[Proof of \Cref{prop:1}]
Any element in the tangent space of the Gaussian density space $\PPG$, $T_{\rho_\rhoa}\PPG$ is given in the form $\nabla_a \rho_a \cdot \sigma$, where $\sigma \in \R^{N_a}$.

The density evolution equation of \cref{eq:G-GF} is
\begin{equation}
\begin{aligned}
\label{eq:mC-Riemannian}
    &\frac{\partial \rho_{\rhoa_t}}{\partial t} = -\nabla_\rhoa \rho_{\rhoa_t}\cdot\fM(\rhoa_t)^{-1}\left.\frac{\partial \mathcal{E}(\rho_\rhoa; \rho_{\rm post})}{\partial\rhoa}\right|_{\rhoa=\rhoa_t}.
\end{aligned}
\end{equation}
where $\rho_{\rhoa_t} = \N(m_t, C_t)$. We will prove that its evolution equations of $m_t$ and $C_t$ are \cref{eq:mC-Momentum}.

Using assumption (\cref{eq:mcp-cond}), we know that for any $f(\theta) \in \textrm{span}\{\theta_i, \theta_i\theta_j,1\leq i, j \leq N_{\theta}\}$, there exists a $\sigma \in \R^{N_{\rhoa}}$, such that 
$$f(\theta) = M(\rho_a)\nabla_a\rho_a \cdot \sigma.$$
Then we have 
\begin{equation}
\label{eq:M-inner}
\begin{split}
    \Bigl\langle \M(\rho_\rhoa)^{-1}\left.\frac{\delta \mathcal{E}(\rho)}{\delta \rho}\right|_{\rho=\rho_\rhoa}, f(\theta) \Bigr\rangle &= 
    \Bigl\langle \M(\rho_\rhoa)^{-1}\left.\frac{\delta \mathcal{E}(\rho)}{\delta \rho}\right|_{\rho=\rho_\rhoa},  M(\rho_a) \nabla_a\rho_a \cdot \sigma \Bigr\rangle \\
    &= 
    \Bigl\langle \left.\frac{\delta \mathcal{E}(\rho)}{\delta \rho}\right|_{\rho=\rho_\rhoa},  \nabla_a\rho_a \cdot \sigma \Bigr\rangle
\end{split}
\end{equation}
and 
\begin{equation}
\label{eq:fM-inner}
\begin{split}
    \Bigl\langle \nabla_\rhoa \rho_{\rhoa}\cdot\fM(\rhoa)^{-1}\left.\frac{\partial \mathcal{E}(\rho_\rhoa;\rho_{\rm post})}{\partial\rhoa}\right., f(\theta) \Bigr\rangle &= 
    \Bigl\langle\nabla_\rhoa \rho_{\rhoa}\cdot\fM(\rhoa)^{-1}\left.\frac{\partial \mathcal{E}(\rho_\rhoa;\rho_{\rm post})}{\partial\rhoa}\right.,  M(\rho_a) \nabla_a\rho_a \cdot \sigma \Bigr\rangle \\
    &= 
    \Bigl\langle M(\rho_a) \nabla_\rhoa \rho_{\rhoa}\cdot\fM(\rhoa)^{-1}\left.\frac{\partial \mathcal{E}(\rho_\rhoa;\rho_{\rm post})}{\partial\rhoa}\right.,   \nabla_a\rho_a \cdot \sigma \Bigr\rangle \\
    &= 
    \Bigl\langle \left.\frac{\partial \mathcal{E}(\rho_\rhoa;\rho_{\rm post})}{\partial\rhoa}\right.,    \sigma \Bigr\rangle \\
    &= 
    \Bigl\langle \left.\frac{\delta \mathcal{E}(\rho)}{\delta \rho}\right|_{\rho=\rho_\rhoa},  \nabla_a\rho_a \cdot \sigma \Bigr\rangle
\end{split}
\end{equation}
Here we used the definition of $\fM(a)$ in \cref{eq:mra2} for the the third equality. Combining \cref{eq:M-inner} and \cref{eq:fM-inner}, we have that 
the mean and covariance evolution equations of \cref{eq:mC-Riemannian} are \cref{eq:mC-Momentum2}.
\end{proof}

\begin{proof}[Proof of \Cref{prop:2}]

Using the calculation in the proof of \Cref{lemma: stein equality}, the tangent space of the Gaussian density manifold at $\rho_\rhoa$ with $\rhoa = (m ,C)$ is
\begin{equation}
    \begin{aligned}
    T_{\rho_{\rhoa}} \PPG 
    &= {\rm span}\bigl\{\rho_{\rhoa}[C^{-1}(\theta-m)]_i,\, \rho_{\rhoa}[C^{-1}(\theta-m)(\theta-m)^TC^{-1} - C^{-1}]_{ij}\bigr\}  \\
    &= {\rm span}\bigl\{\rho_{\rhoa}(\theta_i - \E_{\rho_{\rhoa}}[\theta_i]),\, \rho_{\rhoa}(\theta_i \theta_j - \E_{\rho_{\rhoa}}[\theta_i\theta_j]) \bigr\},
    \end{aligned}
\end{equation}
and $1\leq i,j \leq N_{\theta}$.

For the Fisher-Rao metric, we have 
\begin{equation}
\begin{aligned}
\M^{\rm FR}(\rho_{\rhoa})^{-1}{\rm span}\bigl\{\theta_i, \theta_i\theta_j\bigr\} 
= {\rm span}\bigl\{\rho_{\rhoa}(\theta_i - \E_{\rho_{\rhoa}}[\theta_i]), 
\rho_{\rhoa}(\theta_i\theta_j - \E_{\rho_{\rhoa}}[\theta_i\theta_j])\bigr\} 
= T_{\rho_{\rhoa}} \PPG.
\end{aligned}
\end{equation}

For the affine invariant Wasserstein metric with preconditioner $\Prec$ independent of $\theta$, we have 
\begin{equation}
\begin{aligned}
    &\M^{\rm AIW}(\rho_{\rhoa})^{-1}{\rm span}\bigl\{\theta_i, \theta_i\theta_j\bigr\} \\
    &\quad =  {\rm span}\bigl\{\nabla_{\theta}\cdot(\rho_\rhoa(\theta) \Prec(\rho_\rhoa) e_i),\,  
                   \nabla_{\theta}\cdot(\rho_\rhoa(\theta) \Prec(\rho_\rhoa) e_i \theta_j) \bigr\}\\
    &\quad =  
    {\rm span}\bigl\{ \nabla_{\theta}\cdot(\rho_\rhoa(\theta) (b' + A'\theta)) \quad \forall b'\in\R^{N_\theta}\, A'\in\R^{N_\theta\times N_\theta} \bigr\}    \\
    &\quad =  {\rm span}\bigl\{ \rho_\rhoa(\theta)[ {\rm tr}(A') - (\theta - m)^T C^{-1} (b' + A'\theta)] \quad \forall b'\in\R^{N_\theta}\, A'\in\R^{N_\theta\times N_\theta} \bigr\} \\
    &\quad = T_{\rho_{\rhoa}} \PPG.
\end{aligned}
\end{equation}
Here $e_i$ is the $i$-th unit vector.

For the affine invariant Stein metric with preconditioner $\Prec$ independent of $\theta$ and with a bilinear kernel $\kappa(\theta, \theta',\rho) = (\theta - m)^TA(\rho)(\theta' - m) + b(\rho)$ ($b\neq0$, $A$ nonsingular), we have 
\begin{equation}
\begin{aligned}
    &\M^{\rm AIS}(\rho_{\rhoa})^{-1}{\rm span}\bigl\{\theta_i, \theta_i\theta_j\bigr\} \\
    &\quad =  {\rm span}\Bigl\{\nabla_{\theta}\cdot\Bigl(\rho_{\rhoa}(\theta)\Prec(\rho_{\rhoa})\int \kappa\rho_{\rhoa}(\theta')e_i\dd \theta' \Bigr) ,  
                  \nabla_{\theta}\cdot\Bigl(\rho_{\rhoa}(\theta)\Prec(\rho_{\rhoa})\int \kappa\rho_{\rhoa}(\theta')e_i\theta_j' \dd \theta' \Bigr) \Bigr\}\\
    &\quad =  {\rm span}\Bigl\{\nabla_{\theta}\cdot\Bigl(\rho_{\rhoa}(\theta)\Prec(\rho_{\rhoa})b e_i \Bigr) ,  
                  \nabla_{\theta}\cdot\Bigl(\rho_{\rhoa}(\theta)\Prec(\rho_{\rhoa})[(\theta-m)^TA e_j c_{jj} e_i + b e_i m_j] \Bigr) \Bigr\}         \\
    &\quad =  
    {\rm span}\Bigl\{ \nabla_{\theta}\cdot(\rho_\rhoa(\theta) (b' + A'\theta)) \quad \forall b'\in\R^{N_\theta}\, A'\in\R^{N_\theta\times N_\theta} \Bigr\}  \\
    &\quad = T_{\rho_{\rhoa}} \PPG.
\end{aligned}
\end{equation}

\end{proof}

\subsection{Analytical Solutions for the Gaussian Approximate Fisher-Rao Gradient Flow}
The analytical solution for the Gaussian approximate Fisher-Rao gradient flow~\eqref{eq:Gaussian Fisher-Rao} to sample Gassian distribution is in the following lemma:
\begin{lemma}
\label{proposition:analytical-solution-ngd}
Assume the posterior distribution is Gaussian
$\rho_{\rm post}(\theta)\sim \N(m_{\star}, C_{\star})$. Then the Gaussian approximate Fisher-Rao gradient flow~\eqref{eq:Gaussian Fisher-Rao} has the analytical solution
\begin{subequations}
\begin{align} 
&m_t = m_{\star} + e^{-t}\Bigl((1-e^{-t})C_{\star}^{-1} + e^{-t}C_0^{-1}\Bigr)^{-1} C_0^{-1} \bigl(m_0 - m_{\star}\bigr),  \label{e:mtCt}   \\
&C_t^{-1} = C_{\star}^{-1} + e^{-t}\bigl(C_0^{-1}  -  C_{\star}^{-1}\bigr).
\end{align}
\end{subequations}
\end{lemma}
The proof is in~\Cref{proof:analytical-solution-ngd}. We remark that both mean and covariance converge exponentially fast to $m_{\star}$ and $C_{\star}$ with convergence rate $\bigO(e^{-t})$. This rate is independent of $m_{\star}$ and $C_{\star}$. The uniform convergence rate $\bigO(e^{-t})$ of the Gaussian approximate affine invariant Wasserstein gradient flow~\eqref{eq:Gaussian-AI-Wasserstein} is obtained in~\cite[Lemma 3.2]{garbuno2020interacting}\cite{carrillo2021wasserstein}.

\subsection{Proof of \Cref{proposition:analytical-solution-ngd}}
\label{proof:analytical-solution-ngd}
\begin{proof}
Under the Gaussian posterior assumption and the fact that $$\frac{\dd C_t^{-1}}{\dd t} = -C_t^{-1}\frac{\dd C_t}{\dd t}C_t^{-1},$$ the Gaussian approximate Fisher-Rao gradient flow \cref{eq:Gaussian Fisher-Rao} becomes 
\begin{equation}
\begin{aligned}    
&\frac{\dd  m_t}{\dd t}= C_t C_{\star}^{-1}\bigl(m_{\star} - m_t\bigr),  \\
&\frac{\dd  C_t^{-1}}{\dd t}
=  -C_t^{-1}  + C_{\star}^{-1}.
\end{aligned}
\end{equation}
The covariance update equation has an analytical solution
\begin{equation}
\begin{aligned}    
&C_t^{-1} = (1 - e^{-t})C_{\star}^{-1} + e^{- t}C_0^{-1} .
\end{aligned}
\end{equation}
We reformulate \cref{e:mtCt} as
$$
m_{\star} - m_t =  e^{-t} C_t C_0^{-1} \bigl( m_{\star} - m_0 \bigr)  .
$$
Computing its time derivative leads to
\begin{equation}
\begin{aligned}    
\frac{\dd  m_t }{\dd t} &= -e^{-t}C_tC_0^{-1}(m_0 - m_{\star}) +  e^{-t}\frac{\dd C_t}{\dd t}C_0^{-1}(m_0 - m_{\star})\\
&= -e^{-t}C_tC_0^{-1}(m_0 - m_{\star}) +  e^{-t}(C_t - C_t C_{\star}^{-1}C_t)C_0^{-1}(m_0 - m_{\star})\\
&= C_t C_{\star}^{-1} e^{-t}C_tC_0^{-1}(m_{\star}-m_0 )\\
&= C_t C_{\star}^{-1}\bigl(m_{\star} - m_t\bigr). 
\end{aligned}
\end{equation}
\end{proof}

\subsection{Proof of Convergence for Gaussian Posterior (\Cref{p:rate})}
\label{proof:p:rate}
\begin{proof}
Under the Gaussian posterior assumption, for $\theta_t\sim \N(m_t,C_t)$ with density $\rho_{\rhoa_t}$, we have
\begin{align}
\E_{\rho_{\rhoa_t}} \bigl[ \nabla_{\theta}  \log \rho_{\rm post}(\theta_t)  \bigr]= -C_{\star}^{-1}(m_t-m_{\star}),\quad \E_{\rho_{\rhoa_t}}\bigl[\Hess \log \rho_{\rm post}(\theta_t)\bigr]=-C_{\star}^{-1}.
\end{align}
Here $m_{\star}$ and $C_{\star}$ are the posterior mean and covariance given in
\cref{eq:Gaussian-KL0}.
We have explicit expressions for the Gaussian approximate gradient flows \cref{eq:g-gf,eq:Gaussian Fisher-Rao,eq:Gaussian-Wasserstein}:
\begin{align}
\begin{split}\label{e:gd}
    \text{Gaussian approximate} &\text{ gradient flow:}
    \\
    &\partial_t{m}_t = -C_{\star}^{-1}(m_t-m_{\star}),
\quad 
\partial_t{C}_t = \frac{1}{2}C_t^{-1} - \frac{1}{2}C_{\star}^{-1}.\\
\text{Gaussian approximate} &\text{  Fisher-Rao gradient flow:}\\
    &\partial_t{m}_t = -C_t C_{\star}^{-1}(m_t-m_{\star}),
\quad 
\partial_t{C}_t = C_t-C_t C_{\star}^{-1}C_t.\\
\text{Gaussian approximate} &\text{ Wasserstein gradient flow:}
\\
&\partial_t{m}_t = -C_{\star}^{-1}(m_t-m_{\star}),
\quad
\partial_t{C}_t = 2\I -C_tC^{-1}_{\star} - C^{-1}_{\star} C_t.
\end{split}
\end{align}

For the dynamics of $m_t$ for both Gaussian approximate gradient flow and Gaussian approximate Wasserstein gradient flow, we have
\begin{align}
    m_t-m_{\star}=e^{-t C_{\star}^{-1}}(m_0-m_{\star}).
\end{align}
By taking the $2-$norm on both sides, using $\|\cdot\|_2$ to denote
both the vector and induced matrix norms, and recalling that
the largest eigenvalue of $C_{\star}$ is $\la$, we obtain
\begin{align}
    \|m_t-m_{\star}\|_2\leq \|e^{-t C_{\star}^{-1}}\|_2\|m_0-m_{\star})\|_2\leq e^{-t/\la}\|m_0-m_{\star}\|_2=\bigO(e^{-t/\la}).
\end{align}
The bound can be achieved, when $m_0 - m_{\star}$ has nonzero component in the $C_{\star}$ eigenvector direction corresponding to $\lambda_{\star,\max}$.

For the Gaussian approximate Fisher-Rao gradient flow, thanks to the explicit formula \cref{e:mtCt}, we find that
\begin{align}
   \|m_t-m_{\star} \|_2\leq e^{-t} \max \{ \|C_{\star}\|_2, \|C_0\|_2\} \|C_0^{-1}\|_2\|m_0-m_{\star}\|_2=\bigO(e^{-t}).
\end{align}
The bound can be achieved when $m_0 - m_{\star}$ is nonzero.

Next, we analyze the dynamics of the covariance matrix $C_t$. Our initialization $C_0=\lambda_0\I$, commutes with $C_{\star}, C_{\star}^{-1}$. It follows that  $C_t$ commutes with $C_{\star}, C_{\star}^{-1}$ for any $t\geq 0$ and all gradient flows in \cref{e:gd}, since $0$ is the unique solution of the evolution ordinary differential equation of $C_tC_{{\star}} - C_{{\star}}C_t$. So we can diagonalize $C_t, C_t^{-1}, C_{\star}, C_{\star}^{-1}$ simultaneously, and write down the dynamics of the eigenvalues of $C_t$. For any eigenvalue $\lambda_t$ of $C_t$, it satisfies the following differential equations,
\begin{align}\begin{split}\label{e:gd2}
    \text{Gaussian approximate gradient flow:}
    \quad 
&\partial_t \lambda_t=\frac{1}{2 \lambda_t}-\frac{1}{2 \lambda_{\star}},\\
\text{Gaussian  approximate Fisher-Rao gradient flow:}
\quad 
    &
\partial_t{\lambda}_t = \lambda_t-\lambda^2_t \lambda_{\star}^{-1},\\
\text{Gaussian  approximate Wasserstein gradient flow:}
\quad 
&\partial_t{\lambda}_t = 2 -\frac{2\lambda_t}{\lambda_{\star}},
\end{split}\end{align}
where $\lambda_{\star}$ is the corresponding eigenvalue of $C_{\star}$. 
From \cref{e:gd2}, we know that $\lambda_t$ is bounded between $\lambda_0$ and $\lambda_{\star}$. Moreover,  the ordinary differential equations in \cref{e:gd2} can be solved explicitly. 

For the Gaussian approximate gradient flow
\begin{align}
    \lambda_t-\lambda_{\star}=(\lambda_0-\lambda_{\star})e^{-\frac{t}{2\lambda_{\star}^2}-\frac{\lambda_t-\lambda_0}{\lambda_{\star}}},\quad |\lambda_t-\lambda_{\star}|=\bigO(e^{-t/2\lambda_{\star}^2}).
\end{align}
Since the largest eigenvalue of $C_{\star}$ is $\la$, we conclude that 
$$\|C_t-C_{\star}\|_2=\bigO(e^{-t/2\la^2}).$$

For the Gaussian approximate Fisher-Rao gradient flow, 
\begin{align}
\lambda_t=\frac{\lambda_{\star}}{1+\left(\frac{\lambda_{\star}}{\lambda_0}-1\right)e^{-t}},\quad |\lambda_t-\lambda_{\star}|=\bigO(e^{-t}). 
\end{align}
It follows that 
$$\|C_t-C_{\star}\|_2=\bigO(e^{-t}).$$

For the Gaussian approximate Wasserstein gradient flow
\begin{align}
      \lambda_t=\lambda_{\star} + e^{-2t/\lambda_{\star}}(\lambda_0 - \lambda_{\star}),\quad |\lambda_t-\lambda_0|=\bigO(e^{-2t/\lambda_{\star}}).
\end{align}
Since the large eigenvalue of $C_{\star}$ is $\la$, we conclude that 
$$\|C_t-C_{\star}\|_2=\bigO(e^{-2t/\la}).$$
\end{proof}

\subsection{Proof of \Cref{p:logconcave-local}}
\label{proof:proof-counter-example-Gaussian-Fisher-Rao}

Let  $\rho_{\rhoa_\star}$ be $\N(m_{\star}, C_{\star})$, the
unique minimizer of \cref{eq:Gaussian-KL0}, noting that this is also the unique 
stationary point of Gaussian approximate Fisher-Rao gradient flow~\cref{eq:Gaussian Fisher-Rao}; see ~\Cref{l:gd} for definition of $\rhoa_\star$. It satisfies
\begin{equation}
\label{eq:stationary}
   \E_{\rho_{\rhoa_\star}}\bigl[\nabla_\theta  \log \rho_{\rm post}(\theta)   \bigr] = 0
\quad 
\textrm{and}
\quad
\E_{\rho_{\rhoa_\star}}\bigl[-\Hess\log \rho_{\rm post}(\theta) \bigr] = C_{\star}^{-1}. 
\end{equation}

For $N_{\theta} = 1$, we can calculate the linearized Jacobian matrix of the ODE system~\cref{eq:Gaussian Fisher-Rao} around {$(m_{\star}, C_{\star})$:
\begin{equation}
\begin{split}
\label{eq:Jacobian}
&\frac{\partial \textrm{RHS}}{\partial(m, C)}\Big\vert_{(m, C) = (m_{\star}, C_{\star})} 
\\
=&
 \begin{bmatrix}
-1 & - \frac{1}{2}\E_{\rho_{\rhoa_\star}}\bigl[-\Hess\log \rho_{\rm post}(\theta) (\theta - m_{\star})\bigr]  \\
-\E_{\rho_{\rhoa_\star}}\bigl[-\Hess\log \rho_{\rm post}(\theta) (\theta - m_{\star})\bigr] C_{\star} &  -\frac{1}{2} - \frac{1}{2}\E_{\rho_{\rhoa_\star}}[-\Hess\log \rho_{\rm post}(\theta) (\theta - m_{\star})^2 ]
\end{bmatrix}.
\end{split}
\end{equation}
}

{We further define}
\begin{equation}
\begin{aligned}
\label{eq:A1A2}
A_1 &:= \E_{\rho_{\rhoa_\star}}\bigl[-\Hess\log \rho_{\rm post}(\theta) (\theta - m_{\star})\bigr],
\\
A_2 &:= \E_{\rho_{\rhoa_\star}}[-\Hess\log \rho_{\rm post}(\theta) (\theta - m_{\star})^2 ] \geq 0.
\end{aligned}
\end{equation}
Using the Cauchy-Schwarz inequality and \cref{eq:stationary}, we have 
\begin{align}
\label{eqn: D.30}
A_2 C_{\star}^{-1} = A_2 \E_{\rho_{\rhoa_\star}}\bigl[-\Hess\log \rho_{\rm post}(\theta) \bigr] \geq A_1^2. 
\end{align}
By direct calculations, the two eigenvalues of \cref{eq:Jacobian} satisfy
\begin{subequations}
\label{eq:lambda}
\begin{alignat}{3}
    \lambda_1  &= \frac{-(\frac{3}{2} + \frac{1}{2}A_2) - \sqrt{(\frac{1}{2} - \frac{1}{2}A_2)^2 + 2A_1^2 C_{\star}}}{2} \leq -1
\\
\label{eq:lambda2}
\lambda_2  &= \frac{-(\frac{3}{2} + \frac{1}{2}A_2) + \sqrt{(\frac{1}{2} - \frac{1}{2}A_2)^2 + 2A_1^2 C_{\star}}}{2} 
\\
&= -\frac{1 + A_2 - A_1^2 C_{\star}}{(\frac{3}{2} + \frac{1}{2}A_2) + \sqrt{(\frac{1}{2} - \frac{1}{2}A_2)^2 + 2A_1^2 C_{\star}}} \leq -\frac{1}{3+ A_2}, \nonumber
\end{alignat}
\end{subequations}
where in the last inequality, we have used \eqref{eqn: D.30}. In the following, we prove bounds on $\lambda_2$.
\vspace{1em}

\paragraph{\bf Step 1 (Upper bound)}
Since the upper bounds \eqref{eq:lambda} of the two eigenvalues depend on $A_2$, we will  first prove that
\begin{align}
\label{eq:A_2}
    A_2 \leq \bigl(4 + \frac{4}{\sqrt{\pi}}\bigr)\bigl(\log(\frac{\beta}{\alpha}) + 1\bigr).
\end{align}
Without loss of generality, we assume $m_{\star} = 0$; otherwise we can always achieve this through a change of variable. Considering now only the right half of the integration (i.e. integration from $0$ to $+\infty$)
defining $A_2$, we have
\begin{equation}
\label{eq:A_2_half}
\begin{aligned}
& \int_0^{+\infty} -\Hess\log \rho_{\rm post}(\theta)  \theta^2 \frac{1}{\sqrt{2\pi C_{\star}}} e^{-\frac{1}{2} \frac{\theta^2}{C_{\star}}}\dd \theta 
\\
= & \int_{0}^{A} -\Hess\log \rho_{\rm post}(\theta)  \theta^2 \frac{1}{\sqrt{2\pi C_{\star}}} e^{-\frac{1}{2} \frac{\theta^2}{C_{\star}}}\dd \theta + \int_{A}^{+\infty} -\Hess\log \rho_{\rm post}(\theta)\theta^2 \frac{1}{\sqrt{2\pi C_{\star}}} e^{-\frac{1}{2} \frac{\theta^2}{C_{\star}}}\dd \theta
\\
\leq & \frac{A^2}{C_{\star}} + \int_{A}^{+\infty} -\Hess\log \rho_{\rm post}(\theta)\theta^2 \frac{1}{\sqrt{2\pi C_{\star}}} e^{-\frac{1}{2} \frac{\theta^2}{C_{\star}}}\dd \theta
\quad (\textrm{Using } \cref{eq:stationary} \textrm{ and } \theta \leq A)
\\
\leq & \frac{A^2}{C_{\star}} + \beta \int_{A}^{+\infty} \frac{\theta ^3}{A} \frac{1}{\sqrt{2\pi C_{\star}}} e^{-\frac{1}{2} \frac{\theta^2}{C_{\star}}}\dd \theta
\quad (\textrm{Using } \theta \geq A \textrm{ and } \beta\textrm{-smoothness of } \log \rho_{\rm post})
\\
= & \frac{A^2}{C_\star} + \frac{2\beta C_\star^2}{A\sqrt{2\pi C_\star}} (\frac{A^2}{2C_\star} + 1) e^{-\frac{A^2}{2C_\star}} \quad (\textrm{Direct calculation})\\
\leq & \frac{A^2}{C_{\star}} + \frac{\beta }{\sqrt{\pi} \alpha} (\frac{A}{\sqrt{2C_{\star}}} + \frac{\sqrt{2C_{\star}}}{A} )  e^{-\frac{A^2}{2C_{\star}}}  \qquad (\textrm{Using } C_{\star} \leq \frac{1}{\alpha})\\
 \leq & (2 + \frac{2}{\sqrt{\pi}})(\log(\frac{\beta}{\alpha}) + 1),
\end{aligned}
\end{equation}
where in the last inequality, we have chosen $A$ such that $\frac{A^2}{2C_{\star}} = \max\{\log(\frac{\beta}{\alpha}), 1\}$ since the previous derivations work for any positive $A$.
{We can get a similar bound for the left half of the integration defining $A_2$~(i.e. integration from $-\infty$ to $0$). Combining these two bounds leads to \cref{eq:A_2}.}
{Bringing \cref{eq:A_2} into \cref{eq:lambda} leads to}
\begin{align}
\label{eq:counter:lambda2}
\lambda_2  \leq -\frac{1}{(7+\frac{4}{\sqrt{\pi}})\bigl(1 + \log(\frac{\beta}{\alpha})\bigr)}.
\end{align}
{Therefore, we finish the proof of \eqref{eq:lambda_max} in \Cref{p:logconcave-local}.}

\vspace{1em}
\paragraph{\bf Step 2 (Lower bound)}
Next, we will construct an example to show the bound is sharp. The basic idea is to construct a sequence of triplets $\rho_{{\rm post},n}$,  $\beta_n$, and $\alpha_n$, where $\lim_{n \to \infty} \frac{\beta_n}{\alpha_n} = \infty$, and the corresponding $-\lambda_{2,n} = \bigO\bigl(1/ \log(\frac{\beta_n}{\alpha_n})\bigr)$. In the following proof, we ignore the subscript $n$ for simplicity.

We consider the following sequence of posterior density functions $\rho_{\rm post}$, such that
\begin{equation}
\label{eq:logconcave-slow}
\begin{aligned}
    -\Hess \log \rho_{\rm post}(\theta) &= \int H(x) \frac{e^{-\frac{(\theta - x)^2}{2\sigma^2}}}{\sigma\sqrt{2\pi}} {\rm d}x
\qquad H(x) = 
\begin{cases}
  \beta   & \gamma - \sigma < x < \gamma + \sigma \\
  \alpha  & \textrm{ Otherwise}
\end{cases},
  \\
\nabla_{\theta}\log \rho_{\rm post}(\theta) &= \int_{-\infty}^{\theta}-\Hess \log \rho_{\rm post}(\theta') d\theta' + c,
\end{aligned}
\end{equation}
where $-\Hess \log \rho_{\rm post}(\theta)$ is a smoothed bump function containing four parameters $\gamma, \sigma > 0$ and $0 < \alpha < \beta$. Clearly, we have \[\alpha\I \preceq -\Hess \log \rho_{\rm post}(\theta) \preceq \beta \I.\]
Moreover, we will have another parameter $c$ to determine $\nabla_{\theta}\log \rho_{\rm post}(\theta)$.
It is worth mentioning that for such $\alpha$-strongly logconcave posterior, the Gaussian variational inference \eqref{eq:Gaussian-KL0} has a unique minimizer $(m^\star, C^\star)$, which is determined by the stationary point condition in~\cref{eq:stationary}; see also \cite{lambert2022variational}.

The \textbf{intuition} behind the construction of the bump function is as follows. 
Our objective is to ensure that the dominant eigenvalue, denoted as $\lambda_2$ in equation \eqref{eq:Jacobian}, as large as possible (thus leading to a lower bound).
Since $\lambda_2$ satisfies
\begin{align*}
\lambda_2 = -\frac{1 + A_2 - A_1^2 C_{\star}}{\frac{3}{2}+\frac{1}{2}A_2 + \sqrt{(\frac{1}{2}A_2 - \frac{1}{2})^2 + 2A_1^2 C_{\star}}}  &\geq -\frac{1 + A_2 - A_1^2 C_{\star}}{1 + A_2 } =
-\frac{\frac{1}{A_2} + (1 - \frac{A_1^2 C_{\star}}{A_2})}{1 + \frac{1}{A_2} },
\end{align*}
we require $A_2$ to be as large as possible, while ensuring that the expression $(1 - \frac{A_1^2 C_{\star}}{A_2})$ as small as possible. Recall the definitions of $A_1$ and $A_2$ in \cref{eq:A1A2}; for the latter term, we have 
\begin{align}
A_2 C_{\star}^{-1} &= A_2 \E_{\rho_{\rhoa_\star}}\bigl[-\Hess\log \rho_{\rm post}(\theta) \bigr] \\
& = \E_{\rho_{\rhoa_\star}}[-\Hess\log \rho_{\rm post}(\theta) (\theta - m_{\star})^2 ] \E_{\rho_{\rhoa_\star}} \bigl[-\Hess\log \rho_{\rm post}(\theta) \bigr] \\
&\geq (\E_{\rho_{\rhoa_\star}}[-\Hess\log \rho_{\rm post}(\theta) (\theta - m_{\star})])^2 = A_1^2
\end{align}
due to the Cauchy-Schwarz inequality. Thus we get $(1 - \frac{A_1^2 C_{\star}}{A_2}) \geq 0$. To make this term as close to zero as possible, we consider when the Cauchy-Schwarz inequality can become equality. In fact, we need $-\Hess \log \rho_{\rm post}(\theta)$ to take the form of a delta function. As in the assumption we have $\log \rho_{\rm post}(\theta) \in C^2$ so this is not achievable. To approximate this condition, we can construct $-\Hess \log \rho_{\rm post}(\theta)$ as a bump function $H(x)$ and gradually narrow the width of the bump to approach zero. The Gaussian kernel is employed to smooth the bump function and simplify the subsequent calculations; this is why the form of $H$ in \eqref{eq:logconcave-slow} is constructed.

Now, we will provide a detailed construction. Instead of specifying $\beta$ and $c$, we can specify $m_{\star}$ and $C_{\star}$ since there is a one-to-one correspondence between them. We specify\footnote{These choices of parameters are motivated by the subsequent calculations.} that
\begin{align}
\label{eq:construct-A}
    \sigma = \gamma^{1.5} \quad m_{\star} = 0 \quad C_{\star} = -\frac{\gamma^2}{2\log\gamma}  - \gamma^3 \quad \textrm{and} \quad\alpha = \frac{1}{(-\log \gamma) C_\star}. 
\end{align}
Then, $\beta$ and $c$ are determined 
by the stationary point condition~\eqref{eq:stationary}, namely
\begin{equation}
\label{eq:stationary-cov}
\begin{aligned}
 {C_{\star}^{-1}}
 &= \E_{\rho_{\rhoa_\star}}\bigl[-\Hess\log \rho_{\rm post}(\theta) \bigr] =   \int H(x) \frac{1}{\sqrt{2\pi(\sigma^2 + C_{\star})}}e^{-\frac{(x - m_{\star})^2}{2(\sigma^2 + C_{\star})}} {\rm d}x \\
 &= \alpha + (\beta - \alpha)\int_{\gamma-\sigma}^{\gamma+\sigma}\frac{1}{\sqrt{2\pi(\sigma^2 + C_{\star})}}e^{-\frac{(x - m_{\star})^2}{2(\sigma^2 + C_{\star})}}{\rm d}x \, .
 \\
&  \hspace{-1cm}0 = \E_{\rho_{\rhoa_\star}}\bigl[\nabla_{\theta}\log \rho_{\rm post}(\theta) \bigr]  = \int_{-\infty}^{\infty} \rho_{\rhoa_\star}(\theta) \int_{-\infty}^{\theta}-\Hess \log \rho_{\rm post}(\theta') {\rm d}\theta' {\rm d}\theta + c
 \end{aligned}
\end{equation}
We will let $\gamma \to 0$ later. Note that in the above system, when $\gamma$ is close to zero, for any $C_{\star} = \frac{1}{-\alpha \log \gamma}\leq \frac{1}{\alpha}$, we will have $\beta > \alpha$, so the above lead to valid specification of parameters. We now only have one free parameter $\gamma$. We can now also view the triplet $\rho_{{\rm post}}$,  $\beta$, and $\alpha$ as functions parameterized by $\gamma$.

In the following, we will let $\gamma \rightarrow 0$ and estimate the leading order of $A_1$, $A_2$, $\alpha$ and $\beta$ in terms of $\gamma$. 
Let denote $m_{\star  \sigma} = \frac{xC_{\star} + m_{\star}\sigma^2}{\sigma^2 + C_{\star}}$ and $C_{\star  \sigma} = \frac{\sigma^2C_\star}{\sigma^2 + C_{\star}}$, we have
\begin{align*}
&    -\Hess \log \rho_{\rm post}(\theta) \N(m_{\star}, C_{\star}) =  \int H(x) \frac{e^{-\frac{(\theta - m_{\star  \sigma})^2}{2C_{\star  \sigma}}}}{\sqrt{2\pi C_{\star  \sigma}}}\frac{1}{\sqrt{2\pi(\sigma^2 + C_{\star})}}e^{-\frac{(x - m_{\star})^2}{2(\sigma^2 + C_{\star})}}{\rm d}x.
\end{align*}
Bringing this to \cref{eq:A1A2} leads to
\begin{equation}
\label{eq:A1A2-intg}
    \begin{split}
&A_1 = \E_{\rho_{a_\star}}\bigl[-\Hess\log \rho_{\rm post}(\theta) (\theta - m_{\star})\bigr]
=
\int H(x) \frac{m_{\star  \sigma} - m_{\star}}{\sqrt{2\pi(\sigma^2 + C_{\star})}}e^{-\frac{(x - m_{\star})^2}{2(\sigma^2 + C_{\star})}}{\rm d}x,
\\
&A_2 = \E_{\rho_{a_\star}}[-\Hess\log \rho_{\rm post}(\theta) (\theta - m_{\star})^2 ]
=
\int H(x) \frac{C_{\star\sigma} + (m_{\star\sigma}-m_{\star})^2}{\sqrt{2\pi(\sigma^2 + C_{\star})}}e^{-\frac{(x - m_{\star})^2}{2(\sigma^2 + C_{\star})}}{\rm d}x .
    \end{split}
\end{equation}
With the definitions in~\cref{eq:construct-A}, we have the following estimation about Gaussian integration
\begin{align}
\label{eq:construct-A-G}
\int_{\gamma-\sigma}^{\gamma + \sigma} x^ne^{-\frac{x^2}{2(\sigma^2 + C_{\star})}}{\rm d}x  
= 2\sqrt{\gamma}\gamma^{n+2} (1 +\bigTheta(\log\gamma\sqrt\gamma)).
\end{align}
Here we use $\bigTheta$ to denote the leading order term, as $\gamma \rightarrow 0$.
Bringing the definition of $\sigma$ and $C_{\star}$ and using change-of-variable with $y = \frac{x}{\gamma}$, the left hand side of \cref{eq:construct-A-G} becomes
\begin{align}
\label{eq:x-y}
\int_{\gamma-\sigma}^{\gamma + \sigma} x^ne^{-\frac{x^2}{2(\sigma^2 + C_{\star})}}{\rm d}x  
= \gamma^{n+1}\int_{1 -\sqrt\gamma}^{1+\sqrt\gamma} y^ne^{y^2\log\gamma}{\rm d}y. 
\end{align}
Bringing the following inequalities into \cref{eq:x-y} leads to \cref{eq:construct-A-G}
\begin{align*}
\gamma^{n+1}\int_{1 -\sqrt\gamma}^{1+\sqrt\gamma} y^ne^{y^2\log\gamma}{\rm d}y \geq& 2\sqrt{\gamma}\gamma^{n+1}  (1-\sqrt\gamma)^ne^{(1+\sqrt\gamma)^2\log\gamma}  \\
                                                            =& 2\sqrt{\gamma}\gamma^{n+2} (1 + \bigTheta(2\log\gamma\sqrt\gamma)),
\\
\gamma^{n+1}\int_{1 -\sqrt\gamma}^{1+\sqrt\gamma} y^ne^{y^2\log\gamma}{\rm d}y \leq& 2\sqrt{\gamma}\gamma^{n+1}  (1+\sqrt\gamma)^ne^{(1 -\sqrt\gamma)^2\log\gamma}  \\
                                                           =& 2\sqrt{\gamma}\gamma^{n+2} (1 - \bigTheta(2\log\gamma\sqrt\gamma)).
\end{align*}
Here we used the Taylor expansions of $(1 - \sqrt{\gamma})^ne^{(\gamma + 2\sqrt{\gamma})\log\gamma} = \bigTheta\bigl((1 - n\sqrt{\gamma}) (1 + (\gamma + 2\sqrt{\gamma})\log\gamma)\bigr)$ and $(1 + \sqrt{\gamma})^ne^{(\gamma - 2\sqrt{\gamma})\log\gamma} = \bigTheta\bigl((1 + n\sqrt{\gamma}) (1 + (\gamma - 2\sqrt{\gamma})\log\gamma)\bigr)$.
Then the estimation for $A_1$, $A_2$ from \cref{eq:A1A2-intg} and the covariance condition in ~\cref{eq:stationary} become
\begin{subequations}
\label{eq:Taylor-A}
    \begin{align}
A_1 &=
\frac{(\beta  -\alpha)C_{\star}}{\sigma^2 + C_{\star}}\int_{\gamma-\sigma}^{\gamma + \sigma} x\frac{1}{\sqrt{2\pi(\sigma^2 + C_{\star})}}e^{-\frac{x^2}{2(\sigma^2 + C_{\star})}}{\rm d}x
\label{eq:Taylor-A1}\\
&=
\frac{(\beta  -\alpha)C_{\star}}{\sigma^2 + C_{\star}}
\frac{2\gamma^{2.5}\sqrt{-\log \gamma}}{\sqrt{\pi}}(1 +\bigTheta(\log\gamma\sqrt\gamma))  \quad (\textrm{Using } \eqref{eq:construct-A-G}) 
\nonumber\\
&=
\frac{2(\beta  -\alpha)}{\sqrt{\pi}}\gamma^{2.5}(-\log\gamma)^{0.5  }  + \bigTheta\Bigl(\frac{2(\beta  -\alpha)}{\sqrt{\pi}}\gamma^{3}(-\log\gamma)^{1.5  } \Bigr),
\nonumber
\\
A_2 &=\alpha C_{\star} + 
(\beta - \alpha)\int_{\gamma-\sigma}^{\gamma + \sigma} \frac{C_{\star\sigma} + m_{\star\sigma}^2}{\sqrt{2\pi(\sigma^2 + C_{\star})}}e^{-\frac{x^2}{2(\sigma^2 + C_{\star})}}{\rm d}x
\label{eq:Taylor-A2}\\
&=\alpha C_{\star} + 
(\beta - \alpha)\frac{2\sqrt{-\log\gamma}\gamma^{3.5}(1+2\gamma\log\gamma)(1+\gamma+2\gamma\log\gamma)}{\sqrt{\pi}}(1 +\bigTheta(\log\gamma\sqrt\gamma)) \quad (\textrm{Using }~\eqref{eq:construct-A-G})
\nonumber\\
&=\alpha C_{\star} 
+ 
\frac{2(\beta - \alpha)}{\sqrt{\pi}}\gamma^{3.5}(-\log\gamma)^{0.5}  
+
\bigTheta\Bigl(  
\frac{2(\beta - \alpha)}{\sqrt{\pi}}\gamma^{4}(-\log\gamma)^{1.5}
\Bigr),
\nonumber\\
{C_{\star}}^{-1}  &= \alpha + (\beta - \alpha)\int_{\gamma-\sigma}^{\gamma+\sigma}\frac{1}{\sqrt{2\pi(\sigma^2 + C_{\star})}}e^{-\frac{x^2}{2(\sigma^2 + C_{\star})}}{\rm d}x 
\label{eq:Taylor-Cstar}\\
&= \alpha + (\beta - \alpha)2\gamma^{1.5}\frac{\sqrt{-\log\gamma}}{\sqrt\pi} 
(1 + \bigTheta(\log\gamma\sqrt\gamma))  \quad (\textrm{Using } \eqref{eq:construct-A-G})
\nonumber\\
&= \alpha + \frac{2(\beta - \alpha)}{\sqrt{\pi}}\gamma^{1.5}(-\log\gamma)^{0.5} 
+\bigTheta\Bigl( \frac{2(\beta - \alpha)}{\sqrt{\pi}}\gamma^{2}(-\log\gamma)^{1.5} \Bigr). \nonumber
\end{align}
\end{subequations}
From \cref{eq:construct-A}, we have 
\begin{align}
\label{eq:alpha}
&\alpha =\bigTheta(\frac{1}{\gamma^2}).
\end{align}
Combining \cref{eq:alpha}, the definition of $C_\star$ in \cref{eq:construct-A}, and \cref{eq:Taylor-Cstar} leads to the estimation about $\beta$, as follows:
\begin{align}
\label{eq:beta}
\frac{\beta}{\alpha} = \bigTheta(\frac{(-\log\gamma)^{0.5}}{\gamma^{1.5}}) \quad \textrm{and} \quad \beta - \alpha = \bigTheta(\frac{(-\log\gamma)^{0.5}}{\gamma^{3.5}}).
\end{align}
Combining \cref{eq:Taylor-A,eq:alpha,eq:beta} leads to the estimations about $A_1$ and $A_2$, as follows
\begin{align}
\label{eq:beta_alpha}
    &A_1 = \bigTheta(\frac{-\log \gamma}{\gamma}) \qquad A_2 = \bigTheta(-\log \gamma)
    \\
    &1 - \frac{A_1^2 C_{\star}}{A_2} = \frac{A_2C_{\star}^{-1} - A_1^2 }{A_2C_{\star}^{-1}} = \frac{\bigTheta\Bigl((-\log\gamma)/\gamma^2\Bigr)}{\bigTheta\Bigl((-\log\gamma)^2/\gamma^2\Bigr)}
    = \bigTheta(\frac{1}{-\log\gamma}).
\end{align}

Finally, we can bound the large eigenvalue of \cref{eq:Jacobian} by
\begin{align*}
-\lambda_2 = \frac{1 + A_2 - A_1^2 C_{\star}}{\frac{3}{2}+\frac{1}{2}A_2 + \sqrt{(\frac{1}{2}A_2 - \frac{1}{2})^2 + 2A_1^2 C_{\star}}}  \leq 
\frac{\frac{1}{A_2} + (1 - \frac{A_1^2 C_{\star}}{A_2})}{1 + \frac{1}{A_2} }  
= \bigTheta(\frac{1}{-\log \gamma}). 
\end{align*}
Here we use $A_2 \geq 0$ and $1 - \frac{A_1^2 C_{\star}}{A_2} \geq 0 $.
\Cref{eq:beta} $\log(\frac{\beta}{\alpha}) = \bigTheta(-\log\gamma)$ indicates that for the constructed logconcave density, the local convergence rate is not faster than $-\bigO(1/\log (\frac{\beta}{\alpha})).$

\subsection{Proof of \Cref{proposition:counter-example-slow}}
\label{proof:counter-example-slow}

Consider the following example, where $\theta \in \R$ and $\V(\theta) = \sum_{k=0}^{2K+1} a_{2k} \theta^{2k}$ with $a_{4K+2} > 0$. We will choose $a_{2k}$ later so that the convergence of these dynamics is $\bigTheta(t^{-\frac{1}{2K}})$. Recall the Gaussian approximate Fisher-Rao gradient flow is
\begin{equation}
\begin{aligned}
\frac{\dd  m_t}{\dd  t}  & =  C_t\E_{\rho_{\rhoa_t}}[\nabla_\theta \log \rho_{\rm post} ],\\
\frac{\dd  C_t}{\dd  t}
&= C_t + C_t \E_{\rho_{\rhoa_t}}[\Hess \log \rho_{\rm post}]C_t.
\end{aligned}
\end{equation}


We first calculate the explicit formula of the dynamics. For the mean part, we have
\begin{equation}
\label{eqn: D45}
\begin{aligned}
\E_{\rho_{\rhoa}}\bigl[\nabla_\theta  \log \rho_{\rm post}(\theta)   \bigr] =& -\sum_{k=1}^{2K+1} 2k a_{2k} \E_{\rho_{\rhoa}}[\theta^{2k-1}]\\
=& -\sum_{k=1}^{2K+1} 2k a_{2k} \sum_{i=0}^{k-1} \binom{2k-1}{2i+1}m^{2i+1}C^{k-i-1}\frac{(2k-2i-2)!}{2^{k-i-1}(k-i-1)!}\, .
\end{aligned}
\end{equation}
In the above we have used the explicit formula for the moments of Gaussian distributions. By \eqref{eqn: D45}, we know that when $m = 0$, $\E_{\rho_{\rhoa}}\bigl[\nabla_\theta  \log \rho_{\rm post}(\theta)   \bigr] = 0$. Later, we will initialize the dynamics at $m_0 = 0$; as a consequence, $m_t = 0$ so we only need to consider the convariance dynamics given $m=0$.


For the covariance part (assuming $m=0$), using Stein's identity, we get
\begin{equation}
\begin{aligned}
\E_{\rho_{\rhoa}}\bigl[\Hess \log \rho_{\rm post}(\theta) \bigr]_{m = 0}=
\frac{\partial \E_{\rho_{\rhoa}}\bigl[\nabla_\theta  \log \rho_{\rm post}(\theta)   \bigr]}{\partial m}|_{m = 0}=
 - f(C),
\end{aligned}
\end{equation}
where
\begin{equation}
\begin{aligned}
f(C) = \sum_{k=1}^{2K+1} 2k(2k-1) a_{2k}  C^{k-1}\frac{(2k-2)!}{2^{k-1}(k-1)!}.
\end{aligned}
\end{equation}

%
We choose $\{a_{2k}\}_{k=1}^{2K+1}$ such that 
$$ 2k(2k-1)  \frac{(2k-2)!}{2^{k-1}(k-1)!} a_{2k} = \binom{2K+1}{k} (-1)^{2K+1-k},$$ 
which leads to the identity
$$1 - f(C)C = -(C - 1)^{2K+1}.$$

Now, we calculate the explicit form of the dynamics, with the above choice of $\V$. 
We initialize the dynamics with $m_0 = 0$. 
For the Gaussian approximate Fisher-Rao gradient flow~\cref{eq:Gaussian Fisher-Rao}, we have
\begin{equation}
\small
\begin{aligned}    
&\frac{\dd m_t}{\dd t} =  0, \\
&\frac{\dd C_t}{\dd t}  =  C_t(1 - f(C_t)C_t) = -C_t(C_t-1)^{2K+1}.
\end{aligned}
\end{equation}
It is clear that the convergence rate to $C = 1$ is $\bigTheta(t^{-\frac{1}{2K}})$, if we initialize $C_0$ close to $1$.

In fact, we can also obtain convergence rates for other gradient flows under different metrics. For the vanilla Gaussian approximate gradient flow~\cref{eq:g-gf}, we have
\begin{equation}
\small
\begin{aligned}    
&\frac{\dd m_t}{\dd t} =  0, \\
&\frac{\dd C_t}{\dd t}  =  \frac{1}{2C_t}(1 - f(C_t)C_t) = - \frac{(C_t-1)^{2K+1}}{2C_t}.
\end{aligned}
\end{equation}
For the Gaussian approximate Wasserstein gradient flow~\cref{eq:Gaussian-Wasserstein}, we have
\begin{equation}
\small
\begin{aligned}    
&\frac{\dd m_t}{\dd t} =  0, \\
&\frac{\dd C_t}{\dd t}  =  2(1 - f(C_t)C_t)=-2(C_t-1)^{2K+1}.
\end{aligned}
\end{equation}
In all cases the convergence rate to $C = 1$ is $\bigTheta(t^{-\frac{1}{2K}})$.


\section{Details of Numerical Integration}
\label{sec:integration}
In this section, we discuss how to compute the reference values of $\E[\theta]$, $\Cov[\theta]$, and $\E[\cos(\omega^T \theta + b)]$ for the logconcave posterior and Rosenbrock posterior in~\Cref{sec-affine-invariance-experiments}.
First, for any Gaussian distribution, we have
\begin{align*}
 &\int \cos(\omega^T \theta + b) \N(\theta; m, C) \dd \theta = \exp(-\frac{1}{2}\omega^T C\omega ) \cos(\omega^T m+b).
\end{align*}
We can rewrite the logconcave function as
\begin{align*}
 &\V(\theta) =  \frac{1}{2} \frac{({\theta^{(1)}} - {\theta^{(2)}}/\sqrt{\lambda})^2}{10/\lambda} +\frac{{\theta^{(2)}}^4}{20}.
\end{align*}
The following integration formula holds:
\begin{align*}
 &\int e^{-\V(\theta)} \dd {\theta^{(1)}} = \sqrt{20\pi/\lambda} e^{-{\theta^{(2)}}^4/20},\\
 &\int e^{-\V(\theta)} {\theta^{(1)}}  \dd {\theta^{(1)}} = \sqrt{20\pi/\lambda} \frac{{\theta^{(2)}}}{\sqrt{\lambda}}  e^{-{\theta^{(2)}}^4/20},\\
 &\int e^{-\V(\theta)} {\theta^{(2)}}  \dd {\theta^{(1)}} = \sqrt{20\pi/\lambda}   \theta^{(2)}e^{-{\theta^{(2)}}^4/20},\\
 &\int e^{-\V(\theta)} {\theta^{(1)}}^2 \dd {\theta^{(1)}} = \sqrt{20\pi/\lambda}(\frac{{\theta^{(2)^2}}}{\lambda} + \frac{10}{\lambda}) e^{-{\theta^{(2)}}^4/20},\\
 &\int e^{-\V(\theta)} {\theta^{(1)}} {\theta^{(2)}}  \dd {\theta^{(1)}} = \sqrt{20\pi/\lambda} \frac{{\theta^{(2)^2}}}{\sqrt{\lambda}}  e^{-{\theta^{(2)}}^4/20},\\
 &\int e^{-\V(\theta)} {\theta^{(2)^2}}  \dd {\theta^{(1)}} = \sqrt{20\pi/\lambda}   \theta^{(2)^2}e^{-{\theta^{(2)}}^4/20},\\
 &\int e^{-\V(\theta)} \cos({\omega^{(1)}} {\theta^{(1)}} + {\omega^{(2)}} {\theta^{(2)}} + b) \dd {\theta^{(1)}}\\
&\quad\quad\quad =  \sqrt{20\pi/\lambda}e^{-\frac{5}{\lambda}{\omega^{(1)}}^2}\cos({\omega^{(1)}}{\theta^{(2)}}/\sqrt{\lambda} + {\omega^{(2)}} {\theta^{(2)}} + b)e^{-{\theta^{(2)}}^4/20}.
\end{align*}
Other 2D integrations can be addressed by first performing 1D integration
with respect to $\theta^{(1)}$ analytically, and then the second 1D integration with respect to $\theta^{(2)}$
is computed numerically with $10^7$ uniform points.

We can rewrite the Rosenbrock function as
$$
\V(\theta) = \frac{1}{2}\frac{( {\theta^{(2)}} - {\theta^{(1)}}^2)^2}{10/\lambda} + \frac{(1 - {\theta^{(1)}})^2}{20}. 
$$
We have the following integration formula:
\begin{align*}
 &\int e^{-\V(\theta)} \dd {\theta^{(2)}} = \sqrt{20\pi/\lambda} e^{-(1 - {\theta^{(1)}})^2/20},\\
 &\int e^{-\V(\theta)} {\theta^{(1)}}  \dd {\theta^{(2)}} = \sqrt{20\pi/\lambda} {\theta^{(1)}} e^{-(1 - {\theta^{(1)}})^2/20},\\
 &\int e^{-\V(\theta)} {\theta^{(2)}}  \dd {\theta^{(2)}} = \sqrt{20\pi/\lambda} {\theta^{(1)}}^2 e^{-(1 - {\theta^{(1)}})^2/20},\\
 &\int e^{-\V(\theta)} {\theta^{(1)^2}}  \dd {\theta^{(2)}} = \sqrt{20\pi/\lambda} {\theta^{(1)^2}} e^{-(1 - {\theta^{(1)}})^2/20},\\
 &\int e^{-\V(\theta)} {\theta^{(2)}}^2 \dd {\theta^{(2)}} = \sqrt{20\pi/\lambda}({\theta^{(1)}}^4 + \frac{10}{\lambda}) e^{-(1 - {\theta^{(1)}})^2/20},\\
 &\int e^{-\V(\theta)} {\theta^{(1)}\theta^{(2)}}  \dd {\theta^{(2)}} = \sqrt{20\pi/\lambda} {\theta^{(1)}}^3 e^{-(1 - {\theta^{(1)}})^2/20},\\
 &\int e^{-\V(\theta)} \cos({\omega^{(2)}} {\theta^{(2)}} + {\omega^{(1)}} {\theta^{(1)}} + b) \dd {\theta^{(2)}}\\
&\quad\quad\quad =  \sqrt{20\pi/\lambda}e^{-\frac{5}{\lambda}{\omega^{(2)}}^2}\cos({\omega^{(2)}} {\theta^{(1)}}^2 + {\omega^{(1)}} {\theta^{(1)}} + b)e^{-(1 - {\theta^{(1)}})^2/20}.
\end{align*}
Moreover, we have
$$
\int\int e^{-\V(\theta)} \dd {\theta^{(1)}}\dd {\theta^{(2)}} = \frac{20\pi}{\sqrt{\lambda}}\qquad
\E[\theta] = \begin{bmatrix}
  1\\
  11
\end{bmatrix}
\qquad 
\Cov[\theta] = \begin{bmatrix}
  10& 20\\
  20& \frac{10}{\lambda} + 240
\end{bmatrix}.
$$
Other 2D integrations can be addressed by first performing 1D integration
with respect to $\theta^{(2)}$ analytically, and then the second 1D integration with respect to $\theta^{(1)}$
is computed numerically with $10^7$ uniform points.
\end{document}